\documentclass{article}

\PassOptionsToPackage{numbers, compress}{natbib}

\usepackage{amsmath}
\usepackage{mathtools}
\usepackage{amsfonts}
\usepackage{amssymb}
\usepackage{bbm}
\usepackage{stmaryrd}
\usepackage[makeroom]{cancel}
\usepackage{booktabs}
\usepackage[english]{babel} 
\usepackage{graphicx}
\usepackage[export]{adjustbox}
\usepackage{subcaption}
\usepackage{soul}
\usepackage{makecell}
\usepackage{colortbl}
\usepackage[dvipsnames,svgnames,x11names,hyperref]{xcolor}
\usepackage{comment}
\usepackage{longtable}
\usepackage{float}
\usepackage{multirow}
\usepackage{changepage}
\usepackage{pifont}
\usepackage{units}
\usepackage{nicematrix}
\usepackage{tabularx}
\usepackage{pifont}
\newcommand{\cmark}{\ding{51}}%
\newcommand{\xmark}{\ding{55}}%
\usepackage{rotating}

\usepackage{algpseudocode,algorithm,algorithmicx}

\algrenewcommand\algorithmicindent{1em}%

\usepackage{microtype}
\usepackage{graphicx}
\usepackage{booktabs} 

\usepackage{hyperref}

\usepackage[final]{neurips_2022}

\usepackage[utf8]{inputenc} 
\usepackage[T1]{fontenc}    
\usepackage{hyperref}       
\usepackage{url}            
\usepackage{booktabs}       
\usepackage{amsfonts}       
\usepackage{nicefrac}       
\usepackage{microtype}      
\usepackage{xcolor}         

\usepackage{enumitem}
\setlist[itemize]{itemsep=0pt, parsep=0.5em, topsep=0pt}

\newcommand{\red}[1]{{\color{red}#1}}

\definecolor{subaction1}{HTML}{EA6B66}
\definecolor{subaction2}{HTML}{3399FF}
\definecolor{subaction3}{HTML}{008A00}

\usepackage{hyperref}
\hypersetup{
    colorlinks=true,
    citecolor=NavyBlue,
    linkcolor=NavyBlue,
    urlcolor=NavyBlue,
}
\usepackage[open,openlevel=1]{bookmark}

\DeclareMathOperator*{\argmax}{arg\,max}
\DeclareMathOperator{\rank}{rank}
\DeclareMathOperator*{\proj}{proj}
\DeclareMathOperator{\colspace}{colspace}
\DeclareMathOperator{\ncols}{ncols}
\DeclareMathOperator{\vect}{vec}

\newcommand{\Scal}{\mathcal{S}}
\newcommand{\Acal}{\mathcal{A}}
\newcommand{\Zcal}{\mathcal{Z}}
\newcommand{\Fcal}{\mathcal{F}}
\newcommand{\Mcal}{\mathcal{M}}
\newcommand{\avec}{\boldsymbol{a}}
\newcommand{\svec}{\boldsymbol{s}}
\newcommand\set[1]{\{#1\}}
\newcommand{\rotvert}{\rotatebox[origin=c]{90}{$\vert$}}

\newcommand\coolunder[2]{\mathrlap{\smash{\underbrace{\phantom{%
    \begin{matrix} #2 \end{matrix}}}_{\mbox{$#1$}}}}#2}

\newcommand\coolrightbrace[2]{%
\left.\vphantom{\begin{matrix} #1 \end{matrix}}\right\}#2}

\usepackage[nameinlink,capitalise,noabbrev]{cleveref}
\crefname{equation}{Eqn.}{Eqns.}

\usepackage{amsthm,thmtools}
\newtheorem{theorem}{Theorem}

\newtheorem{proposition}[theorem]{Proposition}
\newtheorem{corollary}[theorem]{Corollary}
\theoremstyle{definition}
\newtheorem{definition}{Definition}
\declaretheorem[name=Example,qed={\lower-0.3ex\hbox{$\triangleleft$}}]{example}
\theoremstyle{remark}
\newtheorem*{remark}{Remark}

\usepackage{tikz}
\newcommand*\circled[1]{\tikz[baseline=(char.base)]{
            \node[shape=circle,draw,inner sep=1pt] (char) {#1};}}
\newcommand*\smallcircled[1]{{\small \circled{#1}}}

\title{Leveraging Factored Action Spaces for Efficient Offline Reinforcement Learning in Healthcare}

\author{%
  Shengpu Tang$^1$ ~ Maggie Makar$^1$ ~ Michael W. Sjoding$^2$ ~ Finale Doshi-Velez$^3$ ~ Jenna Wiens$^1$ \\
  ~ \\
  \small $^1$Division of Computer Science \& Engineering, University of Michigan, Ann Arbor, MI, USA \\
  \small $^2$Division of Pulmonary and Critical Care, Michigan Medicine, Ann Arbor, MI, USA \\
  \small $^3$SEAS, Harvard University, Cambridge, MA, USA \\[0.8ex]
  Correspondence to: \href{mailto:tangsp@umich.edu,wiensj@umich.edu}{\color{black}\{tangsp,wiensj\}@umich.edu}\\[0.8ex]
  Reviewed on OpenReview: \href{https://openreview.net/forum?id=Jd70afzIvJ4}{\color{black} https://openreview.net/forum?id=Jd70afzIvJ4}
}

\begin{document}

\maketitle

\begin{abstract}
  Many reinforcement learning (RL) applications have combinatorial action spaces, where each action is a composition of sub-actions. A standard RL approach ignores this inherent factorization structure, resulting in a potential failure to make meaningful inferences about rarely observed sub-action combinations; this is particularly problematic for offline settings, where data may be limited. In this work, we propose a form of linear Q-function decomposition induced by factored action spaces. We study the theoretical properties of our approach, identifying scenarios where it is guaranteed to lead to zero bias when used to approximate the Q-function. Outside the regimes with theoretical guarantees, we show that our approach can still be useful because it leads to better sample efficiency without necessarily sacrificing policy optimality, allowing us to achieve a better bias-variance trade-off. Across several offline RL problems using simulators and real-world datasets motivated by healthcare, we demonstrate that incorporating factored action spaces into value-based RL can result in better-performing policies. Our approach can help an agent make more accurate inferences within underexplored regions of the state-action space when applying RL to observational datasets. 
\end{abstract}

\section{Introduction}

In many real-world decision-making problems, the action space exhibits an inherent combinatorial structure. For example, in healthcare, an action may correspond to a combination of drugs and treatments. When applying reinforcement learning (RL) to these tasks, past work \citep{ernst2006clinical,komorowski2018artificial,prasad2017reinforcement,parbhoo2020transfer} typically considers each combination a distinct action, resulting in an exponentially large action space (\cref{fig:arch}a). This is inefficient as it fails to leverage any potential independence among dimensions of the action space. 

This type of factorization structure in action space could be incorporated when designing the architecture of function approximators for RL (\cref{fig:arch}b). Similar ideas have been used in the past, primarily to improve online exploration \citep{tavakoli2018action,tavakoli2021learning}, or to handle multiple agents \citep{boutilier1996planning,sunehag2018value,rashid2018qmix,son2019qtran,zhou2019factorized} or multiple rewards \citep{juozapaitis2019explainable}. However, the applicability of this approach has not been systematically studied, especially in offline settings and when the MDP presents no additional structure (e.g., when the state space cannot be explicitly factorized). 

In this work, we develop an approach for offline RL with factored action spaces by learning linearly decomposable Q-functions. First, we study the theoretical properties of this approach, investigating the sufficient and necessary conditions for it to lead to an unbiased estimate of the Q-function (i.e., zero approximation error). Even when the linear decomposition is biased, we note that our approach leads to a reduction of variance, which in turn leads to an improvement in sample efficiency. Lastly, we show that when sub-actions exhibit certain structures (e.g., when two sub-actions ``reinforce'' their independent effects), the linear approximation, though biased, can still lead to the optimal policy. We test our approach in offline RL domains using a simulator \citep{oberst2019gumbel} and a real clinical dataset \citep{komorowski2018artificial}, where domain knowledge about the relationship among actions suggests our proposed factorization approach is applicable. Empirically, our approach outperforms a non-factored baseline when the sample size is limited, even when the theoretical assumptions (around the validity of a linear decomposition) are not perfectly satisfied. Qualitatively, in the real-data experiment, our approach learns policies that better capture the effect of less frequently observed treatment combinations. 

\begin{figure}[t]
    \centering
    \includegraphics[scale=0.75,trim=0 0 0 0]{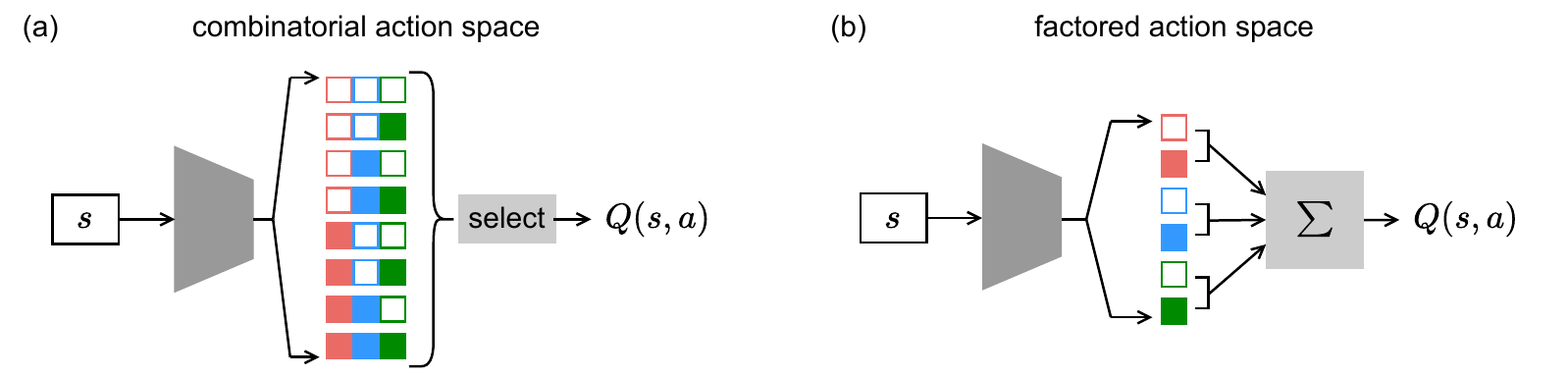}
    \caption{Illustration of Q-network architectures, which take the state $s$ as input and output $Q(s,a)$ for a selected action. In this example, the action space $\mathcal{A}$ consists of $D=3$ binary sub-action spaces $\{\textcolor{subaction1}{\square}, \textcolor{subaction1}{\blacksquare}\}$, $\{\textcolor{subaction2}{\square}, \textcolor{subaction2}{\blacksquare}\}$ and $\{\textcolor{subaction3}{\square}, \textcolor{subaction3}{\blacksquare}\}$. (a) Learning with the combinatorial action space requires $2^3 = 8$ output heads (exponential in $D$), one for each combination of sub-actions. (b) Incorporating the linear Q decomposition for the factored action space requires $2 \times 3 = 6$ output heads (linear in $D$). } 
    \label{fig:arch}
\end{figure}

\vspace{-0.25ex}
Our work provides both theoretical insights and empirical evidence for RL practitioners to consider this simple linear decomposition for value-based RL approaches. Our contribution complements many popular offline RL methods focused on distribution shift (e.g., BCQ \citep{fujimoto2019BCQ}) and goes beyond pessimism-only methods by leveraging domain knowledge. Compatible with any algorithm that has a Q-function component, we expect our approach will lead to gains for offline RL problems with combinatorial action spaces where data are limited and when domain knowledge can be used to check the validity of theoretical assumptions. 
\vspace{-0.25ex}

\vspace{-1ex}
\section{Problem Setup}
\vspace{-2ex}
We consider Markov decision processes (MDPs) defined by a tuple $\mathcal{M} = (\mathcal{S}, \mathcal{A}, p, r, \mu_0, \gamma)$, where $\mathcal{S}$ and $\mathcal{A}$ are the state and action spaces, $p(s'|s,a)$ and $r(s,a)$ are the transition and instantaneous reward functions, $\mu_0(s)$ is the initial state distribution, and $\gamma \in [0,1]$ is the discount factor. A probabilistic policy $\pi(a|s)$ specifies a mapping from each state to a probability distribution over actions. For a deterministic policy, $\pi(s)$ refers to the action with $\pi(a|s)=1$. The state-value function is defined as \( V^{\pi}(s) = \mathbb{E}_{\pi}\mathbb{E}_{\mathcal{M}} \left[ \sum_{t=1}^{\infty}\gamma^{t-1} r_t \ | \  s_1 = s \right] \). The action-value function, $Q^{\pi}(s,a)$, is defined by further restricting the action taken from the starting state. The goal of RL is to find a policy \( \pi^* = \argmax_{\pi} \mathbb{E}_{s \sim \mu_0}[V^{\pi}(s)] \) (or an approximation) that has the maximum expected performance.

\vspace{-1.25ex}
\subsection{Factored Action Spaces}
\vspace{-1.75ex}
While the standard MDP definition abstracts away the underlying structure within the action space $\mathcal{A}$, in this paper, we explicitly express a factored action space as a Cartesian product of $D$ sub-action spaces, $\mathcal{A} = \bigotimes_{d=1}^{D} \Acal_d = \mathcal{A}_1 \times \cdots \times \mathcal{A}_D$. We use $\avec \in \mathcal{A}$ to denote each action, which can be written as a vector of sub-actions $\avec = [a_1, \dots, a_D]$, with each $a_d \in \mathcal{A}_d$. In general, a sub-action space can be discrete or continuous, and the cardinalities of discrete sub-action spaces are not required to be the same. For clarity of analysis and illustration, we consider discrete sub-action spaces in this paper. 

\vspace{-1.25ex}
\subsection{Linear Decomposition of Q Function}
\vspace{-1.75ex}

The traditional factored MDP literature almost exclusively considers state space factorization \cite{koller1999computing}. In contrast, here we capitalize on action space factorization to parameterize value functions. Specifically, our approach considers a linear decomposition of the $Q$ function, as illustrated in \cref{fig:arch}b: 
\begin{equation} \textstyle
    Q^{\pi}(s,\avec) = \sum_{d=1}^{D} q_d(s,a_d). \label{eqn:Q-decomposition}
\end{equation}
Each component $q_d(s,a_d)$ in the summation is allowed to condition on the full state space $s$ and only one sub-action $a_d$. While similar forms of decomposition have been used in past work, there are key differences in how the summation components are parameterized. In the multi-agent RL literature, each component $q_d(s_d,a_d)$ can only condition on the corresponding state space of the $d$-th agent \citep[e.g.,][]{sunehag2018value,rashid2018qmix}. The decomposition in \cref{eqn:Q-decomposition} also differs from a related form of decomposition considered by \citet{juozapaitis2019explainable} where each component $q_d(s,\avec)$ can condition on the full action $\avec$. To the best of our knowledge, we are the first to consider this specific form of Q-function decomposition backed by both theoretical rigor and empirical evidence; in addition, we are the first to apply this idea to offline RL. We discuss other related work in \cref{sec:related}.

\vspace{-1ex}
\section{Theoretical Analyses} \label{sec:theory}
\vspace{-1.75ex}
In this section, we study the theoretical properties of the linear Q-function decomposition induced by factored action spaces. We first present sufficient and necessary conditions for our approach to yield unbiased estimates, and then analyze settings in which our approach can reduce variance without sacrificing policy performance when the conditions are violated. Finally, we discuss how domain knowledge may be used to check the validity of these conditions, providing examples in healthcare. 

\subsection{Sufficient Conditions for Zero Bias} \label{sec:sufficient}
\vspace{-1.5ex}
If we consider the total return of $D$ MDPs running in parallel, where each MDP is defined by their respective state space $\Scal_d$ and action space $\Acal_d$, then the desired linear decomposition holds for the MDP defined by the joint state space $\bigotimes_{d=1}^{D}\Scal_d$ and joint action space $\bigotimes_{d=1}^{D}\Acal_d$ (formally discussed in \cref{appx:sufficient-trivial}). However, this relies on an explicit, known state space factorization, limiting its applicability. In contrast, we now present a generalization that forgoes the explicit factorization of the state space by making use of state abstractions. Intuitively, the MDP should have some implicit factorization, such that it is homomorphic to $D$ parallel MDPs. It is, however, not a requirement that this factorization is known, as long as it exists. 

\begin{theorem}\label{thm:sufficient-abstract}
Given an MDP defined by $\Scal, \Acal, p, r$ and a policy $\pi: \Scal\rightarrow\Delta(\Acal)$, where $\Acal = \bigotimes_{d=1}^{D} \Acal_d$ is a factored action space with $D$ sub-action spaces, if there exists $D$ unique corresponding state abstractions $\boldsymbol{\phi} = [\phi_1, \cdots, \phi_D]$ where $\phi_d: \Scal \rightarrow \Zcal_d$, $z_d = \phi_d(s)$, $z'_d = \phi_d(s')$, such that for all $s,a,s'$ the following holds: 
\begin{equation} \textstyle
  \sum_{\tilde{s} \in \boldsymbol{\phi}^{-1}(\boldsymbol{\phi}(s'))} p(\tilde{s}|s,\avec) = \prod_{d=1}^{D} p_d(z'_d|z_d,a_d) \label{eqn:abstract-factored-transition}
\end{equation}
\begin{tabularx}{\textwidth}{@{}XX@{}}
  \begin{equation} \textstyle
  r(s,\avec) = \sum_{d=1}^{D} r_d(z_d,a_d) \label{eqn:abstract-factored-reward}
  \end{equation} &
  \begin{equation} \textstyle
  \pi(\avec|s) = \prod_{d=1}^{D} \pi_d(a_d|z_d) \label{eqn:abstract-factored-policy}
  \end{equation}
\end{tabularx}
for some $p_d: \Zcal_d \times \Acal_d \to \Delta(\Zcal_d)$, $r_d: \Zcal_d \times \Acal_d \to \mathbb{R}$, and $\pi_d: \Zcal_d \to \Delta(\Acal_d)$, then the Q-function of policy $\pi$ can be expressed as \(Q^{\pi}(s,\avec) = \sum_{d=1}^{D} q_{d}(s,a_d). \)
\end{theorem}
In \cref{appx:sufficient-abstract}, we present an induction-based proof of \cref{thm:sufficient-abstract}. Since every assumption is used in key steps of the proof, we conjecture that the sufficient conditions cannot be relaxed in general. Consequently, if the sufficient conditions are satisfied, then using \cref{eqn:Q-decomposition} to parameterize the Q-function leads to zero approximation error and results in an unbiased estimator. Note that this does not require knowledge of $\boldsymbol{\phi}$. To highlight the significance of \cref{thm:sufficient-abstract}, we present the following example, in which the state space cannot be explicitly factored, yet the linear decomposition exists (additional examples probing the sufficient conditions can be found in \cref{appx:examples}).

\begin{figure}[h]
    \setlength{\belowcaptionskip}{-10pt}
    \centering
    \centerline{
    \hspace{-1em}\adjustbox{valign=t}{\begin{minipage}{0.5\textwidth}
    \begin{tabular}{c}
    \multicolumn{1}{l}{(a)} \\
    \begin{tabular}{lll}
    \includegraphics[scale=0.675, valign=c, trim=0 0 0 10]{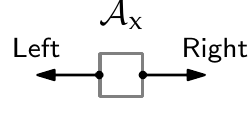} \qquad &
    \includegraphics[scale=0.675, valign=c, trim=0 0 0 10]{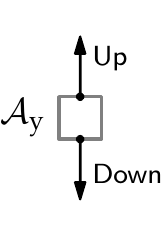} \qquad &
    \includegraphics[scale=0.675, valign=c, trim=0 0 0 10]{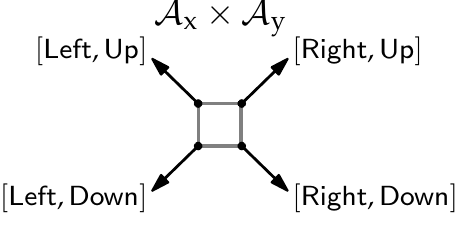}
    \end{tabular} \\[2em]
    \multicolumn{1}{l}{(c)} \\
\scalebox{0.75}{
    \begin{tabular}{cccc}
    \toprule
    $\avec = [a_{\mathrm{x}}, a_{\mathrm{y}}]$
    & \multicolumn{1}{c@{\hspace*{\tabcolsep}\makebox[0pt]{$+$}}}{$q_{\mathrm{x}}(s_{0,0},a_{\mathrm{x}})$} 
    & \multicolumn{1}{c@{\hspace*{\tabcolsep}\makebox[0pt]{$=$}}}{$q_{\mathrm{y}}(s_{0,0},a_{\mathrm{y}})$} 
    & $Q^{\pi}(s_{0,0},\avec)$ \\
    \midrule
    $\swarrow = [\leftarrow,\downarrow]$ & 0.9 & 0.9 & 1.8 \\
    $\nwarrow = [\leftarrow,\uparrow]$ & 0.9 & 1 & 1.9 \\
    $\searrow = [\rightarrow,\downarrow]$ & 1 & 0.9 & 1.9 \\
    $\nearrow = [\rightarrow,\uparrow]$ & 1 & 1 & 2 \\
    \bottomrule
    \end{tabular}%
}
    \end{tabular}
    \end{minipage}}
    \hspace*{\fill}
    \adjustbox{valign=t}{\begin{minipage}{0.46\textwidth}
    \begin{tabular}{c}
    \multicolumn{1}{l}{(b)} \\
    \includegraphics[scale=0.75, trim=0 30 0 20]{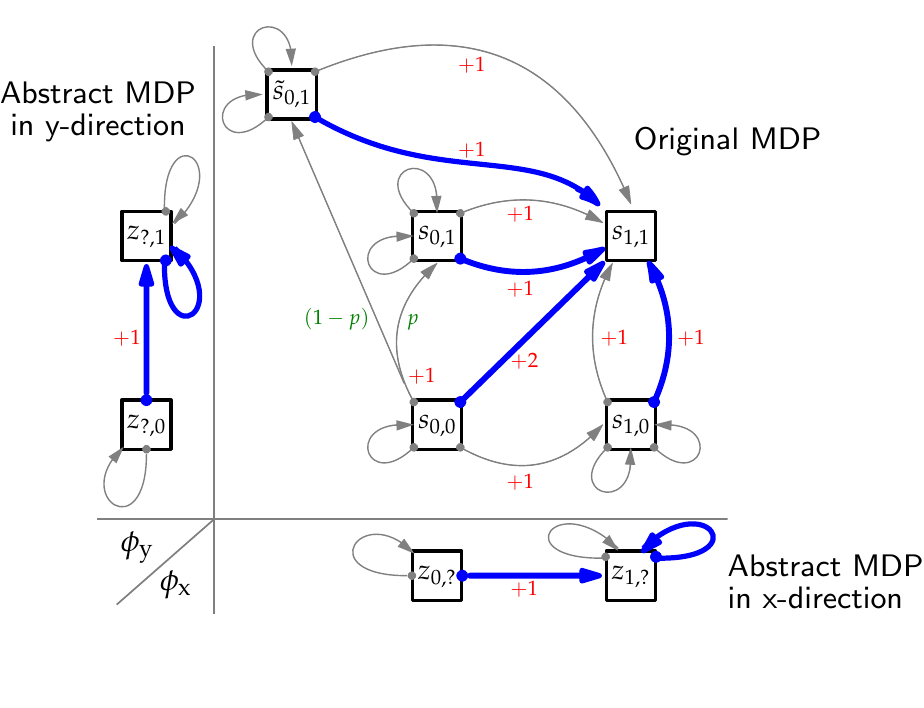}
    \end{tabular}
    \end{minipage}}
    }
    \caption{(a) The composition of sub-action spaces $\Acal_{\mathrm{x}}$ and $\Acal_{\mathrm{y}}$ results in $\Acal = \Acal_{\mathrm{x}} \times \Acal_{\mathrm{y}}$ depicted by outgoing arrows exiting the corners of each state (denoted by $\square$). \textbf{\textit{The corner from which the action exits encodes the direction.}} (b) An MDP with 5 states and 4 actions of the factored action space $\Acal$. For example, action $\nearrow = [\rightarrow, \uparrow]$ from $s_{0,0}$ moves the agent both right ($\rightarrow$) and up ($\uparrow$), to $s_{1,1}$. Under abstractions $\boldsymbol{\phi} = [\phi_{\mathrm{x}}, \phi_{\mathrm{y}}]$, this MDP can be mapped to two abstract MDPs (with action spaces $\Acal_{\mathrm{x}}$ and $\Acal_{\mathrm{y}}$, respectively). The abstract state spaces are $\Zcal_{\mathrm{x}} = \{z_{0,?}, z_{1,?}\}$ and $\Zcal_{\mathrm{y}} = \{z_{?,0}, z_{?,1}\}$, respectively, where $?$ indicates the coordinate ignored by the abstraction. $s_{1,1}$ is an absorbing state whose outgoing transition arrows are not shown. Taking action $\nwarrow = [\leftarrow, \uparrow]$ from $s_{0,0}$ leads to $s_{0,1}$ with probability $p$ and to $\tilde{s}_{0,1}$ with probability $(1-p)$ (denoted in {\color{ForestGreen} green}). Actions taken by a determinisitic policy $\pi$ are denoted by {\color{blue} \textbf{bold blue}} arrows. $\pi$ takes the same action $\searrow = [\rightarrow, \downarrow]$ from $s_{0,1}$ and $\tilde{s}_{0,1}$. Nonzero rewards are denoted in {\color{red} red}. (c) Linear decomposition of $Q^{\pi}$ for $s_{0,0}$ with respect to the factored action space ($\gamma=0.9$). Similar decompositions for other states also exist (omitted for space). }
    \label{fig:chain2d_abstract}
\end{figure}

\begin{example}[Two-dimensional chains with abstractions] \label{eg:chain2d_abstract}
The factored action space shown in \cref{fig:chain2d_abstract}a, $\Acal = \mathcal{A}_{\mathrm{x}} \times \mathcal{A}_{\mathrm{y}}$, is the composition of two binary sub-action spaces: $\mathcal{A}_{\mathrm{x}} = \{\leftarrow, \rightarrow\}$ leading the agent to move left or right, and $\mathcal{A}_{\mathrm{y}} = \{\downarrow, \uparrow\}$ leading the agent to move down or up. Thus, $\Acal$ consists of four actions, where each action $\avec = [a_{\mathrm{x}},a_{\mathrm{y}}]$ leads the agent to move \textit{\textbf{diagonally}}. 

Consider the MDP in \cref{fig:chain2d_abstract}b with action space $\Acal$. The state space $\Scal = \{s_{0,0}, s_{0,1}, \tilde{s}_{0,1}, s_{1,0}, s_{1,1}\}$ contains 5 different states; subscripts indicate the abstract state vector under $\boldsymbol{\phi} = [\phi_{\mathrm{x}}, \phi_{\mathrm{y}}]$ (e.g., $s_{0,1}$ and $\tilde{s}_{0,1}$ are two different raw states but are identical under the abstraction, $\boldsymbol{\phi}(s_{0,1}) = \boldsymbol{\phi}(\tilde{s}_{0,1}) = [z_{0,?}, z_{?,1}]$). There does not exist an explicit state space factorization such that $\Scal = \Scal_{\mathrm{x}} \times \Scal_{\mathrm{y}}$. One can check that \cref{eqn:abstract-factored-transition,eqn:abstract-factored-reward} are satisfied by comparing the raw transitions and rewards against the abstracted version (e.g., action $\nearrow$ from $s_{0,0}$ moves both $\rightarrow$ (under $\phi_{\mathrm{x}}$) and $\uparrow$ (under $\phi_{\mathrm{y}}$) to $s_{1,1}$ and receives the sum of the two rewards, $1+1 = 2$). For \cref{eqn:abstract-factored-policy} to hold, the policy must take the same action from $s_{0,1}$ and $\tilde{s}_{0,1}$. In \cref{fig:chain2d_abstract}c, we show the linear decomposition of the Q-function for one such policy where \cref{thm:sufficient-abstract} applies, under which the evolution of the MDP can be seen as two chain MDPs running in parallel (also in \cref{fig:chain2d_abstract}b). 
\end{example}

\vspace{-1ex}
\subsection{Necessary Conditions for Zero Bias} \label{sec:necessary}
\vspace{-1.5ex}
In \cref{appx:necessary}, we derive a necessary condition for the linear parameterization to be unbiased. Unfortunately, the condition is not verifiable unless the exact MDP parameters are known; this highlights the non-trivial nature of the problem. One may naturally question whether the sufficient conditions (which are arguably more verifiable in practice) must hold (i.e., are necessary) for the linear parameterization to be unbiased. Perhaps surprisingly, \textit{\textbf{none}} of the conditions are necessary. We state the following propositions and provide justifications through a set of counterexamples below and in \cref{appx:examples}. 

\begin{proposition}
There exists an MDP $\Mcal$ and a policy $\pi$ for which $Q_{\Mcal}^{\pi}$ decomposes as \cref{eqn:Q-decomposition} but the transition function $p$ of $\Mcal$ does not satisfy \cref{eqn:abstract-factored-transition}. 
\end{proposition}

\begin{proposition}
There exists an MDP $\Mcal$ and a policy $\pi$ for which $Q_{\Mcal}^{\pi}$ decomposes as \cref{eqn:Q-decomposition} but the reward function $r$ of $\Mcal$ does not satisfy \cref{eqn:abstract-factored-reward}. 
\end{proposition}

\begin{proposition}
There exists an MDP $\Mcal$ and a policy $\pi$ for which $Q_{\Mcal}^{\pi}$ decomposes as \cref{eqn:Q-decomposition} but the policy $\pi$ does not satisfy \cref{eqn:abstract-factored-policy}. 
\end{proposition}

\begin{example}[Modified two-dimensional chains] In \cref{fig:chain2d_necesary}, all conditions in \cref{thm:sufficient-abstract} are violated, yet for each state, there exists a linear decomposition of Q-values (see \cref{appx:examples}). 
\end{example}

\vspace{-5pt}
\begin{figure}[h]
    \centering
    \begin{minipage}[c]{0.32\textwidth}
    \includegraphics[scale=0.8, trim=0 0 0 10]{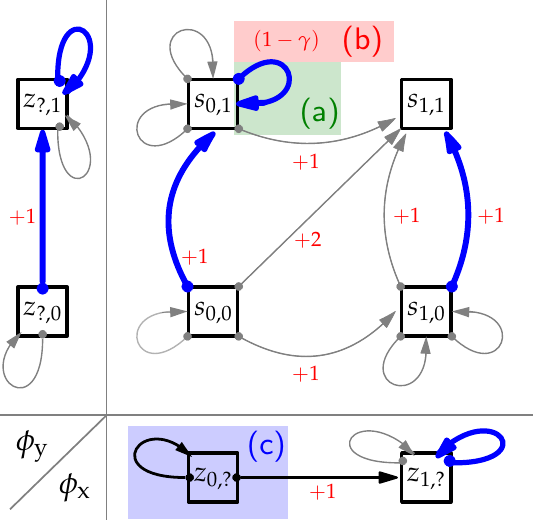}
    \end{minipage}
    \hspace*{\fill}
    \begin{minipage}[c]{0.65\textwidth}
    \caption{This MDP is similar to \cref{eg:chain2d_abstract} (except it does not have state $\tilde{s}_{0,1}$) and we consider the same abstractions $\boldsymbol{\phi} = [\phi_{\mathrm{x}}, \phi_{\mathrm{y}}]$. The Q-function and decomposition are exactly the same as in the previous example. However, none of the conditions in \cref{thm:sufficient-abstract} are satisfied. (a) The transition function does not satisfy \cref{eqn:abstract-factored-transition} because action $\nearrow = [\rightarrow, \uparrow]$ from $s_{0,1}$ does not move right ($\rightarrow$ under $\phi_{\mathrm{x}}$) to $s_{1,1}$ and instead moves back to state $s_{0,1}$. (b) The reward function does not satisfy \cref{eqn:abstract-factored-reward} as the reward of $(1-\gamma)$ for action $\nearrow = [\rightarrow, \uparrow]$ from $s_{0,1}$ is not the sum of $+1$ ($\rightarrow$ from $z_{0,?}$ under $\phi_{\mathrm{x}}$) and $0$ ($\uparrow$ from $z_{?,1}$ under $\phi_{\mathrm{y}}$). (c) The policy does not satisfy \cref{eqn:abstract-factored-policy} as it takes different sub-actions from $z_{0,?}$ under $\phi_{\mathrm{x}}$ ($\nwarrow$ from $s_{0,0}$ specifies $\leftarrow$, whereas $\nearrow$ from $s_{0,1}$ specifies $\rightarrow$). }
    \label{fig:chain2d_necesary}
    \end{minipage}
\end{figure}
\vspace{-5pt}

Therefore, while \cref{thm:sufficient-abstract} imposes a rather stringent set of assumptions on the MDP structure (transitions, rewards) and the policy, violations of these conditions do not preclude the linear parameterization of the Q-function from being an unbiased estimator.

\subsection{How Does Bias Affect Policy Learning?}
When the sufficient conditions do not hold perfectly, using the linear parameterization in \cref{eqn:Q-decomposition} to fit the Q-function may incur nonzero approximation error (bias). This can affect the performance of the learned policy; in \cref{appx:perf_bounds}, we derive error bounds based on the extent of bias relative to the sufficient conditions in \cref{thm:sufficient-abstract}. Despite this bias, our approach always leads to a reduction in the variance of the estimator. This gives us an opportunity to achieve a better bias-variance trade-off, especially given limited historical data in the offline setting. In addition, as we will demonstrate, biased Q-values do not always result in suboptimal policy performance, and we identify the characteristics of problems where this occurs under our proposed linear decomposition. 

\subsubsection{Bias-Variance Trade-off} \label{sec:bias-variance}
While the amount of bias incurred depends on the problem structure, the benefit of variance reduction is immediate. Intuitively, to learn the Q-function of a tabular MDP with state space $\Scal$ and action space $\Acal = \bigotimes_{d=1}^{D} \Acal_d$, the linear parameterization reduces the number of free parameters from $|\Scal||\Acal| = |\Scal|(\prod_{d=1}^{D}|\Acal_d|)$ to $|\Scal|(\sum_{d=1}^{D}|\Acal_d| - D + 1 )$ (see \cref{appx:subspace_proof}). This reduces the hypothesis class from exponential in $D$ to linear in $D$. To analyze variance reduction, we compare the bounds on Rademacher complexity \citep{mohri2018foundations,duan2021risk,makar2022causally} of the Q-function approximator using the factored action space with that of the full combinatorial action space (formally discussed in \cref{appx:rademacher}). 

\begin{proposition} \label{thm:variance}
Using the linear Q-function decomposition for the factored action space in \cref{eqn:Q-decomposition} has a smaller lower bound on the empirical Rademacher complexity compared to learning the Q-function in the combinatorial action space.
\end{proposition}
\vspace{-0.5em}
\cref{thm:variance} shows that our linear Q-function parameterization leads to a smaller function space, which implies a lower-variance estimator. Hence, our factored-action approach can make more efficient use of limited samples, leading to an interesting bias-variance trade-off that is especially beneficial for offline settings with limited data.

\subsubsection{Bias $\not\Rightarrow$ Suboptimal Performance} \label{sec:bias-suboptimal}
Even in the presence of bias, an inaccurate Q-function may still correctly identify the value-maximizing action (\cref{thm:argmax-preserve}). While this statement is generally true, in this section, we identify \textbf{\textit{when}} this occurs \textbf{\textit{specifically given}} our linear decomposition based on factored action spaces. To focus the analysis on the most interesting aspects unique to our approach, we consider a bandit setting; extensions to the sequential RL setting are possible by applying induction similar to the proof for the main theorems (\cref{appx:sufficient-trivial,appx:sufficient-abstract}). 

\begin{proposition} \label{thm:argmax-preserve}
There exists an MDP with the optimal $Q^*$ and its approximation $\hat{Q}$ parameterized in the form of \cref{eqn:Q-decomposition}, such that $\hat{Q} \neq Q^*$ and yet $\argmax_{\avec}\hat{Q}(\avec) = \argmax_{\avec}Q^*(\avec)$. 
\end{proposition}
\vspace{-0.5em}

\begin{figure}[b]
    \centering
\begin{minipage}[c]{0.55\textwidth}
    \begin{tabular}{ll}
    (a) & (b) \\
    \rule{0pt}{-2ex} 
    \makecell[t]{\includegraphics[scale=0.8, valign=c]{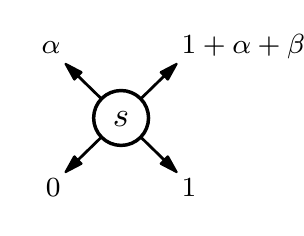}} &
\scalebox{0.75}{\parbox{0.75\linewidth}{%
\begin{equation*}
\setlength\arraycolsep{1pt}
    \begin{matrix}
    r_{\scriptscriptstyle\textsf{Left}} & & + r_{\scriptscriptstyle\textsf{Down}} & & & = & 0 \\
    r_{\scriptscriptstyle\textsf{Left}} & & & + r_{\scriptscriptstyle\textsf{Up}} & & = & \alpha \\
    & \, r_{\scriptscriptstyle\textsf{Right}} & + r_{\scriptscriptstyle\textsf{Down}} & & & = & 1 \\
    & \, r_{\scriptscriptstyle\textsf{Right}} & & + r_{\scriptscriptstyle\textsf{Up}} & + \cancel{r_{\scriptscriptstyle\textsf{Interact}}} & = & 1 + \alpha + \beta
    \end{matrix}
\end{equation*}
}} \\
    (c) \\
    \multicolumn{2}{c}{
\begin{tabular}{cc}
    True Value Function & Linear Approximation \\
\scalebox{0.7}{\parbox{0.5\linewidth}{%
\begin{equation*}
\setlength\arraycolsep{1pt}
\def\arraystretch{1.2}
\begin{bmatrix} Q^*({\scriptstyle\swarrow}) \\ Q^*({\scriptstyle\nwarrow}) \\ Q^*({\scriptstyle\searrow}) \\ Q^*({\scriptstyle\nearrow}) \end{bmatrix} = \begin{bmatrix} 0 \\ \alpha \\ 1 \\ 1+\alpha+\beta \end{bmatrix}
\end{equation*}
}}
& 
\scalebox{0.7}{\parbox{0.5\linewidth}{%
\begin{equation*}
\setlength\arraycolsep{1pt}
\def\arraystretch{1.2}
\begin{bmatrix} \hat{Q}({\scriptstyle\swarrow}) \\ \hat{Q}({\scriptstyle\nwarrow}) \\ \hat{Q}({\scriptstyle\searrow}) \\ \hat{Q}({\scriptstyle\nearrow}) \end{bmatrix} = \begin{bmatrix} & - \frac{1}{4} \beta \\ \alpha & + \frac{1}{4} \beta \\ 1 & +\frac{1}{4} \beta \\ 1 + \alpha & +\frac{3}{4} \beta \end{bmatrix}
\end{equation*}
}}
    \end{tabular}
    }
    \end{tabular}
\end{minipage}
\hspace*{\fill}
\begin{minipage}[c]{0.42\textwidth}
    \caption{(a) A two-dimensional bandit problem with action space $\Acal$. Rewards are denoted for each arm. (b) Learning using the linear Q decomposition approach corresponds to a system of linear equations that relates the reward of each arm. The parameter $r_{\scriptscriptstyle\textsf{Interact}}$ is dropped in our linear approximation, leading to omitted-variable bias. (c) Solving the system results in an approximate value function $\hat{Q}$, which does not equal to the true value function $Q^*$ unless $\beta = 0$.}
    \label{fig:2d_bandit}
\end{minipage}
\end{figure}
\textit{Justification.} Consider a 1-step bandit problem with a single state and the same action space as before, $\Acal = \Acal_{\mathrm{x}} \times \Acal_{\mathrm{y}}$. Taking an action $\avec = [a_{\mathrm{x}},a_{\mathrm{y}}]$ leads the agent to move diagonally and terminate immediately. Since there are no transitions, the Q-values of any policy are simply the immediate reward from each action, $Q(\avec) = r(\avec)$. We assume the reward function is defined as in \cref{fig:2d_bandit}a (\cref{appx:bandit-standardize} describes a procedure to standardize an arbitrary reward function). Applying our approach amounts to solving for the parameters $r_{\scriptscriptstyle\textsf{Left}}, r_{\scriptscriptstyle\textsf{Right}}, r_{\scriptscriptstyle\textsf{Down}}, r_{\scriptscriptstyle\textsf{Up}}$ of the linear system in \cref{fig:2d_bandit}b, while dropping the interaction term $r_{\scriptscriptstyle\textsf{Interact}}$, resulting in a form of omitted-variable bias \citep{wooldridge2015introductory}. Solving the system gives the approximate value function where the interaction term $\beta$ appears in the approximation $\hat{Q}$ for all arms (\cref{fig:2d_bandit}c, details in \cref{appx:OVB}). 

Note that $\hat{Q} = Q^*$ only when $\beta=0$, i.e., there is no interaction between the two sub-actions. We first consider the family of problems with $\alpha=1$ and $\beta\in[-4,4]$. In \cref{fig:2d_bandit_heatmaps}a, we measure the value approximation error $\mathrm{RMSE}(Q^*, \hat{Q})$, as well as the suboptimality $V^{\pi^*} - V^{\hat{\pi}} = \max_{\avec} Q^*(\avec) - Q^*(\argmax_{\avec} \hat{Q}(\avec))$ of the greedy policy defined by $\hat{Q}$ as compared to $\pi^*$. As expected, when $\beta=0$, $\hat{Q}$ is unbiased and has zero approximation error. When $\beta \neq 0$, $\hat{Q}$ is biased and RMSE $>0$; however, for $\beta \geq -1$, $\hat{Q}$ corresponds to a policy that correctly identifies the optimal action. 

We further investigate this phenomenon considering both $\alpha, \beta \in [-4,4]$ (to show all regions with interesting trends), measuring RMSE and suboptimality in the same way as above. As shown in \cref{fig:2d_bandit_heatmaps}b, the approximation error is zero only when $\beta=0$, regardless of $\alpha$. However, in \cref{fig:2d_bandit_heatmaps}c, for a wide range of $\alpha$ and $\beta$ settings, suboptimality is zero; this suggests that in those regions, even in the presence of bias (non-zero approximation error), our approach leads to an approximate value function that correctly identifies the optimal action. The irregular contour outlines multiple regions where this happens; one key region is when the two sub-actions affect the reward in the same direction (i.e., $\alpha \geq 0$) and their interaction effects also affect the reward in the same direction (i.e., $\beta \geq 0$). 
\vspace{-1ex}

\begin{figure}[h]
\setlength{\abovecaptionskip}{-5pt}
\setlength{\belowcaptionskip}{-5pt}
    \centering
    \begin{tabular}{lll}
    (a) & (b) & (c) \\
    \includegraphics[width=0.28\linewidth, trim=-20 0 0 40]{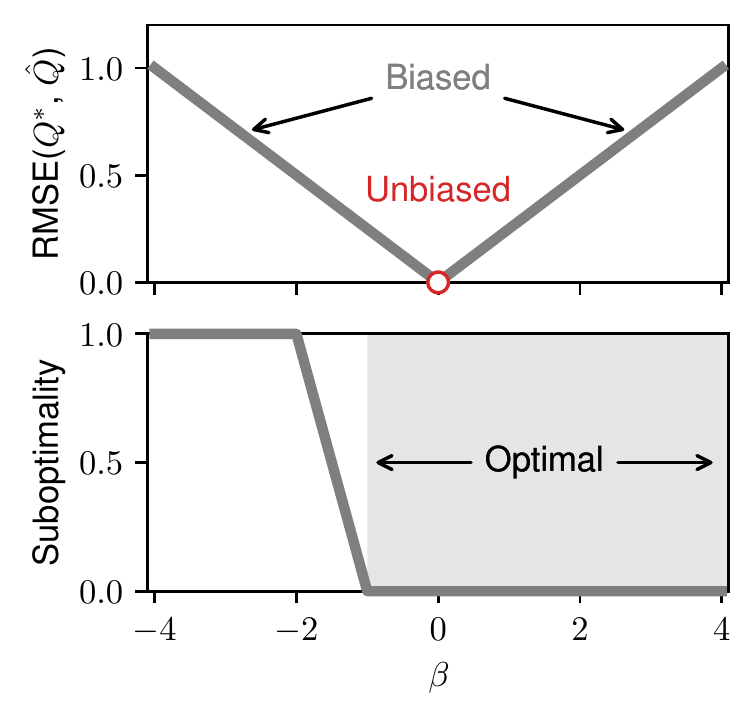} &
    \includegraphics[width=0.28\linewidth, trim=-10 0 10 30]{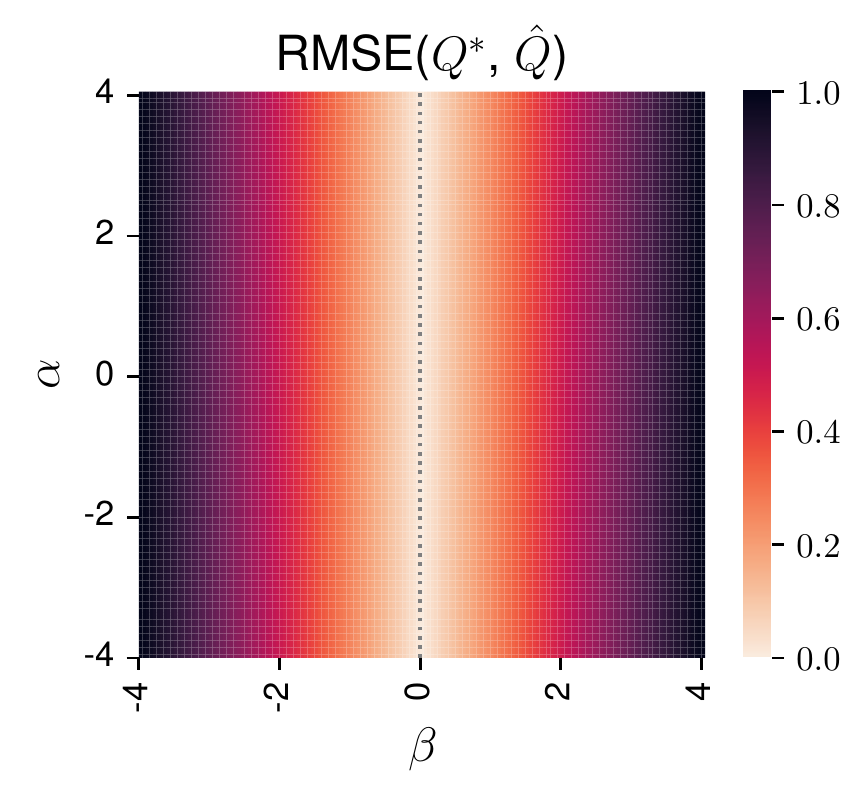} &
    \includegraphics[width=0.28\linewidth, trim=-10 0 10 30]{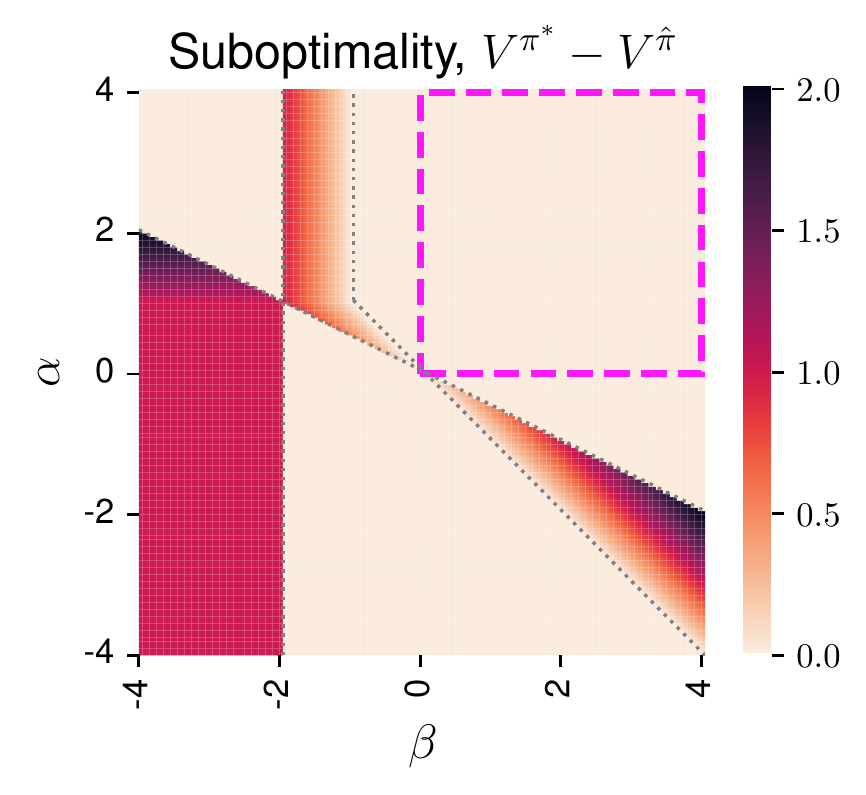}
    \end{tabular}
    \caption{(a) The approximation error and policy suboptimality of our approach for the bandit problem in \cref{fig:2d_bandit}a, for different settings of $\beta$ when $\alpha=1$. The Q-value approximation is unbiased only when $\beta=0$, but the corresponding approximate policy is optimal for a wider range of $\beta\geq -1$. (b-c) The approximation error and policy suboptimality of our approach for the bandit problem in \cref{fig:2d_bandit}a, for different settings of $\alpha$ and $\beta$. The Q-value approximation is unbiased only when $\beta=0$, but the corresponding approximate policy is optimal for a wide range of $\alpha$ and $\beta$ values. The highlighted region of zero suboptimality corresponds to $\alpha \geq 0$ and $\beta \geq 0$. }
    \label{fig:2d_bandit_heatmaps}
\end{figure}

\subsection{Practical Considerations: Are these Assumptions too Strong?} \label{sec:practical}

Based on our theoretical analysis, strong assumptions (\cref{sec:sufficient}) on the problem structure (though not necessary, \cref{sec:necessary}) are the only known way to guarantee the unbiasedness of our proposed linear approximation. It is thus crucial to understand the applicability (and inapplicability) of our approach in real-world scenarios. Exploring to what extent these assumptions hold in practice is especially important for safety-critical domains such as healthcare where incorrect actions (treatments) can have devastating consequences. Fortunately, RL tasks for healthcare are often equipped with significant domain knowledge, which serves as a better guide to inform the algorithm design than heuristics-driven reasoning alone \citep{sharma2017learning, tavakoli2018action,rashid2018qmix}. 
\vspace{-0.5ex}

Oftentimes, when clinicians treat conditions using multiple medications at the same time (giving rise to the factored action space), it is because each medication has a different ``mechanism of action,''  resulting in negligible or limited interactions. For example, several classes of medications are used in the management of chronic heart failure, and each has a unique and incremental benefits on patient outcomes \citep{komajda2018incremental}. Problems such as this satisfy the sufficient conditions in \cref{sec:sufficient} in spite of a non-factorized state space. Moreover, any small interactions would have a bounded effect on RL policy performance (according to \cref{appx:perf_bounds}). 
\vspace{-0.5ex}

Similarly, in the management of sepsis (which we consider in \cref{sec:mimic-sepsis}), fluids and vasopressors affect blood pressure to correct hypotension via different mechanisms \citep{gotts2016sepsis}. Fluid infusion increases ``preload'' by increasing the blood return to the heart to make sure the heart has enough blood to pump out \citep{guerin2015effects}. In contrast, common vasopressors (e.g., norepinephrine) increase ``inotropy'' by stimulating the heart muscle and increase peripheral vascular resistance to maintain perfusion to organs \citep{hamzaoui2010early,monnet2011norepinephrine}. Therefore, while the two treatments may appear to operate on the same part of the state space (e.g., they both increase blood pressure), in general they are not expected to interfere with each other. Recently, there has also been evidence suggesting that their combination can better correct hypotension \citep{hamzaoui2021combining}, which places this problem approximately in the regime discussed in \cref{sec:bias-suboptimal}. 
\vspace{-0.5ex}

In offline settings with limited historical data, the benefits of a reduction in variance can outweigh any potential small bias incurred in the scenarios above and lead to overall performance improvement (\cref{sec:bias-variance}). However, our approach is not suitable if the interaction is counter to the effect of the sub-actions (e.g., two drugs that raise blood pressure individually, but when combined lead to a decrease). In such scenarios, the resulting bias will likely lead to suboptimal performance (\cref{sec:bias-suboptimal}). Nevertheless, many drug-drug interactions are known and predictable \citep{saari2008effect,smithburger2012drug,webmd,epicddi}. In such cases, one can either explicitly encode the interaction terms or resort back to a combinatorial action space (\cref{appx:interactions}). While we focus on healthcare, there are other domains in which significant domain knowledge regarding the interactions among sub-actions is available, e.g., cooperative multi-agent games in finance where there is a higher payoff if agents cooperate (positive interaction effects) or intelligent tutoring systems that teach basic arithmetic operations as well as fractions (which are distinct but related skills). For these problems, this knowledge can and should be leveraged.

\section{Experimental Evaluations} \label{sec:experiments}

We apply our approach to two offline RL problems from healthcare: a simulated and a real-data problem, both having an action space that is composed of several sub-action spaces. These problems correspond to settings discussed in \cref{sec:practical} where we expect our proposed approach to perform well. In the following experiments, we compare our proposed approach (\cref{fig:arch}b), which makes assumptions regarding the effect of sub-actions in combination with other sub-actions, against a common baseline that considers a combinatorial action space (\cref{fig:arch}a). 

\vspace{-0.5em}
\subsection{Simulated Domain: Sepsis Simulator} \label{sec:sepsis-sim}
\vspace{-0.5em}

\textbf{Rationale.} First, we apply our approach to a simulated domain modeled after the physiology of patients with sepsis \citep{oberst2019gumbel}. Although the policies are learned ``offline,'' a simulated setting allows us to evaluate the learned policies in an ``online'' fashion without requiring offline policy evaluation (OPE). 

\textbf{Setup.} Following prior work \citep{tang2021model}, a state is represented by a feature vector $\mathbf{x}(s) \in \{0,1\}^{21}$ that uses a one-hot encoding for each underlying variable (diabetes status, heart rate, blood pressure, oxygen concentration, glucose; all of which are discrete). The action space is composed of 3 binary treatments: antibiotics, vasopressors, and mechanical ventilation, such that $\mathcal{A} = \mathcal{A}_{\textrm{abx}} \times \mathcal{A}_{\textrm{vaso}} \times \mathcal{A}_{\textrm{mv}}$, with $\mathcal{A}_{\textrm{abx}} = \mathcal{A}_{\textrm{vaso}} = \mathcal{A}_{\textrm{mv}} = \{0,1\}$ and $|\mathcal{A}| = 2^3 = 8$. Each treatment affects certain vital signs and may raise or lower their values with pre-specified probabilities (precise definition in \citep{tang2021model}). A patient is discharged alive when all vitals are normal and all treatments have been withdrawn; death occurs if 3 or more vitals are abnormal. Rewards are sparse and only assigned at the end of each episode ($+1$ for survival and $-1$ for death), after which the system transitions into the respective absorbing state. Episodes are truncated at a maximum length of 20 following \citep{oberst2019gumbel} (where no terminal reward is assigned). Here, the MDP partly satisfies the sufficient conditions outlined in \cref{sec:theory}. For example, oxygen saturation (which can be seen as a state abstraction) is only affected by mechanical ventilation, whereas heart rate is only affected by antibiotics. However, blood pressure is affected by both antibiotics and vasopressors, meaning the effects of these two sub-actions are \textit{not} independent. 

\textbf{Offline learning.} First, we generated datasets with different sample sizes following different behavior policies. We ran fitted Q-iteration for up to 50 iterations using a neural network function approximator, selecting the early-stopping iteration based on ground-truth policy performance. Each setting of sample size and behavior policy was repeated 10 times with different random seeds. Additional details are described in \cref{appx:sepsisSim-impl}. 

\textbf{Results.} \cref{fig:sepsisSim-results} compares median performance of the proposed approach vs. the baseline over the 10 runs (error bars are interquartile ranges). We considered behavior policies that take the optimal action with probability $\rho$ and select randomly among non-optimal actions with probability $1-\rho$. 

\textit{How does sample size affect performance?}
We first look at a uniformly random behavior policy ($\rho=1/|\Acal|=0.125$, \cref{fig:sepsisSim-results} center). As expected, larger sample sizes (i.e., more training episodes) lead to better policy performance for both the baseline and proposed approaches. For smaller sample sizes ($<5000$), the proposed approach consistently outperforms the baseline. As sample size increases further, the performance gap shrinks and eventually the baseline overtakes our proposed approach. This is because variance decreases with increasing sample size but the bias incurred by the factored approximation does not change. Once there are enough samples, reductions in variance are no longer advantageous and the incurred bias dominates the performance. Overall, this shows that our approach is promising especially for datasets with limited sample size. 

\textit{How does behavior policy affect performance?}
As we anneal the behavior policy closer to the optimal policy ($\rho > 0.125$, \cref{fig:sepsisSim-results} left two), we reduce the randomness in the behavior policy and limit the amount of exploration possible at the same sample size. The same overall trend largely holds. On the other hand, when the probability of taking the optimal action is less than random ($\rho < 0.125$, \cref{fig:sepsisSim-results} right two), the proposed approach achieves better performance than the baseline with an even larger gap for limited sample sizes ($\le 10^3$). Without observing the optimal actions ($\rho=0$), the baseline performs relatively poorly, even for large sample sizes. In comparison, our approach accounts for relationships among actions to some extent and is thus able to better generalize to the unobserved and underexplored optimal actions, thereby outperforming the baseline. 

\textbf{Takeaways.} In a challenging situation where our theoretical assumptions do not perfectly hold, our proposed approach matches or outperforms the baseline, especially for smaller sample sizes. 

\begin{figure}[t]
    \setlength{\belowcaptionskip}{-10pt}
    \centering
    \includegraphics[width=\textwidth, trim=10 10 0 15]{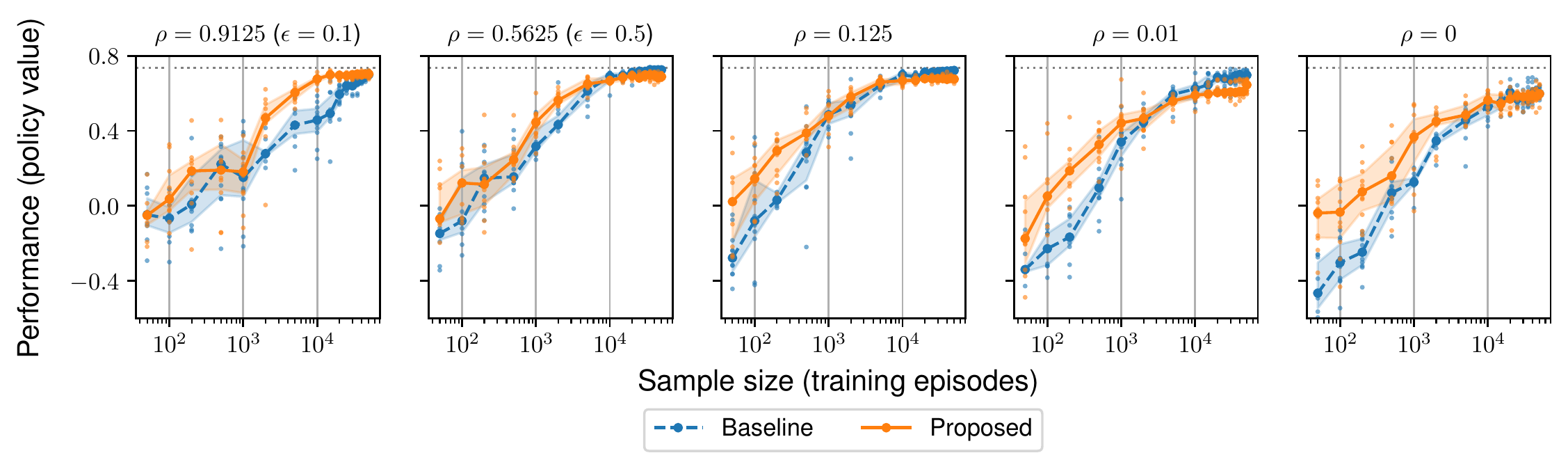}
    \caption{Performance on the sepsis simulator across sample sizes and behavior policies. Plots display the performance over 10 runs, with the trend lines showing medians and error bars showing interquartile ranges. $\rho$ is the probability of taking the optimal action under the behavior policies used to generate offline datasets from the simulator. The left two plots show two $\epsilon$-greedy policies ($\rho > 0.125$; conversion: $\rho = (1-\epsilon) + \epsilon/|\Acal|$); the middle plot shows a uniformly random policy ($\rho = 0.125$); the right two plots show two policies that undersample the optimal action, $\rho < 0.125$; from left to right, $\rho$ decreases. Across different data distributions, our proposed approach outperforms the baseline at small sample sizes, and closely matches baseline performance at large sample sizes. Dashed lines denote the value of the optimal policy, which equals to $0.736$. }
    \label{fig:sepsisSim-results}
\end{figure}

\vspace{-0.5em}
\subsection{Real Healthcare Data: Sepsis Treatment in MIMIC-III} \label{sec:mimic-sepsis}
\vspace{-0.5em}

\textbf{Rationale.} We apply our method to a real-world example of learning optimal sepsis treatment policies for patients in the intensive care unit. Acknowledging the challenging nature of OPE for quantitative comparisons \citep{gottesman2018evaluating,gottesman2019guidelines}, here we qualitatively inspect the learned policies using clinical domain knowledge. 

\textbf{Setup.} Originally introduced by \citep{komorowski2018artificial}, we use the improved formulations of this task as per \citep{tang2020clinician} and \citep{killian2020empirical}. After applying the specified inclusion and exclusion criteria to the de-identified MIMIC-III database \citep{johnson2016mimic}, we obtained a cohort of 19,287 patients and performed a 70/15/15 split for training, validation and testing. For each patient, their data include 10 time-invariant demographic and contextual features and a 33-dimensional time series collected at 4h intervals, consisting of measurements from up to 24h before until up to 48h after sepsis onset. We used a recurrent neural network (RNN) with long short-term memory (LSTM) cells to create an approximate information state \citep{subramanian2019approximate} to summarize the history into a $d_{\mathcal{S}}$-dimensional embedding vector. A terminal reward of 100 is assigned for 48h survival and 0 otherwise. Intermediate rewards are all 0. $\gamma$ for learning is $0.99$ and for evaluation is $1$. Actions pertain to treatment decisions in each 4h interval, representing total volume of intravenous (IV) fluids and amount of vasopressors administered, resulting in a $5 \times 5$ factored action space. 

\textbf{Offline learning.} After learning the state representations, we apply variants of discrete-action batch-constrained Q-learning (BCQ) \citep{fujimoto2019BCQd,fujimoto2019BCQ}, where the baseline uses the combinatorial action space and the proposed approach incorporates the linear decomposition induced by the factored action space. The Q-networks were trained for a maximum of $10,000$ iterations, with checkpoints saved every $100$ iterations. We performed model selection \citep{tang2021model} over the saved checkpoints (candidate policies) by evaluating policy performance using the validation set with OPE. Specifically, we estimated policy value using weighted importance sampling (WIS) and measured effective sample size (ESS), where the behavior policy was estimated using $k$ nearest neighbors in the embedding space. Following previous work \citep{liu2022avoiding}, the final policies were selected by maximizing validation WIS with ESS of $\geq 200$ (we consider other thresholds in \cref{appx:mimic-results}), for which we report results on the test set. 

\textbf{Results.} We visualize the validation performance over all candidate policies. \cref{tab:mimic_quantitative}-left shows that the performance Pareto frontier (in terms of WIS and ESS) of the proposed approach generally dominates the baseline. 

\textit{Quantitative comparisons.} Evaluating the final selected policies on the test set (\cref{tab:mimic_quantitative}-right) shows that the proposed factored BCQ achieves a higher policy value (estimated using WIS) than baseline BCQ at the same level of ESS. In addition, both policies have a similar level of agreement with the clinician policy, comparable to the average agreement among clinicians. 

\textit{Qualitative comparisons. } In \cref{fig:mimic_policy_heatmap}a, we compare the distributions of recommended actions by the clinician behavior policy, baseline BCQ and factored BCQ, as evaluated on the test set. While overall the policies look rather similar, in that the most frequently recommended action corresponds to low doses of IV fluids $<$500mL with no vasopressors, there are notable differences for key parts of the action space. In particular, baseline BCQ almost never recommends higher doses of IV fluids $>$500 mL, either alone or in combination with vasopressors, whereas both clinician and factored BCQ recommend IV fluids $>$500 mL more frequently. These actions are typically used for critically ill patients, for whom the Surviving Sepsis Campaign guidelines recommends up to $>$2L of fluids \citep{evans2021surviving}. We hypothesize that this difference is due to a higher level of heterogeneity in the patient states for which actions with high IV fluid doses were observed, compared to the remaining actions with lower doses of IV fluids. To further understand this phenomenon, we measure the per-action state heterogeneity in the test set by computing, for each action, the standard deviation (averaged over the embedding dimensions) of all RNN state embeddings from which that action is taken according to the behavior policy. As shown in \cref{fig:mimic_policy_heatmap}b, actions with higher IV fluids generally have larger standard deviations, supporting our hypothesis. The larger heterogeneity combined with lower sample sizes makes it difficult for baseline BCQ to correctly infer the effects of these actions, as it does not leverage the relationship among actions. In contrast, our approach leverages the factored action space and can thus make better inferences about these actions. 

\begin{figure}[t]
    \setlength{\belowcaptionskip}{-8pt}
    \centering
    \begin{tabular}{cc}
    \parbox{0.3\textwidth}{\includegraphics[scale=0.55,trim=10 10 0 20]{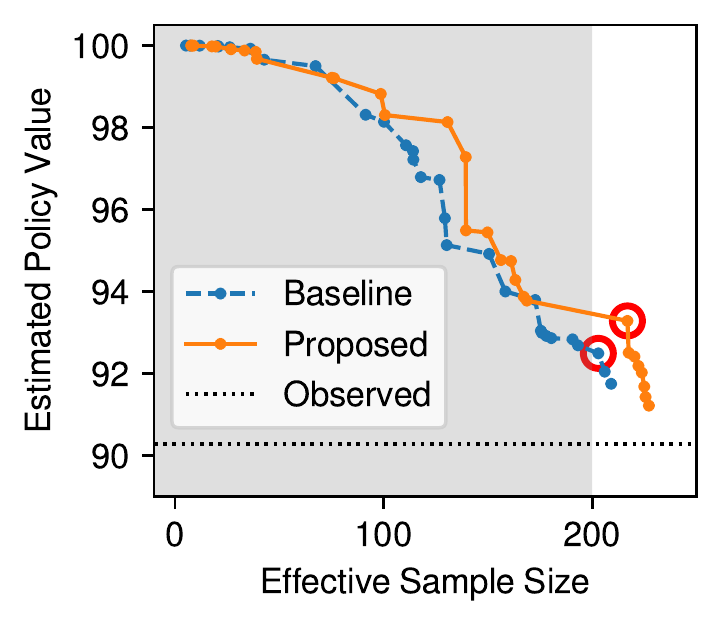}} & 
\scalebox{0.8}{
    \begin{tabular}{c|cc|c}
    \toprule
    \textbf{Policy} & \textbf{Baseline BCQ} & \textbf{Factored BCQ} & \textbf{Clinician} \\
    \midrule
    Test WIS & 90.44 $\pm$ 2.44 & 91.62 $\pm$ 2.12 & 90.29 $\pm$ 0.51 \\
    Test ESS & 178.32 $\pm$ 11.42 & 178.32 $\pm$ 11.96 & 2894 \\
    \midrule
    \makecell{\% agreement \\ with clinician} & 62.42\% & 62.37\% & 57.16\% \\
    \bottomrule
    \end{tabular}
}
    \end{tabular}
    \caption{Left - Pareto frontiers of validation performance for the candidate policies (all points plotted in \cref{fig:mimic_validation_full}). The shaded region does not meet the ESS cutoff of $\geq 200$. The red circles indicate the selected models (based on best validation WIS) for baseline and proposed (both have a BCQ threshold of $\tau=0.5$). Right - Performance on test set, $\pm$ standard errors from 100 bootstraps. }
    \label{tab:mimic_quantitative}
\end{figure}

\begin{figure*}[t]
    \centering
    \centerline{
    \begin{tabular}{ccccc}
    \multicolumn{4}{l}{(a)} & \multicolumn{1}{l}{(b)}\\[-2.8ex]
    \includegraphics[scale=0.38,trim=-5 5 60 12, clip]{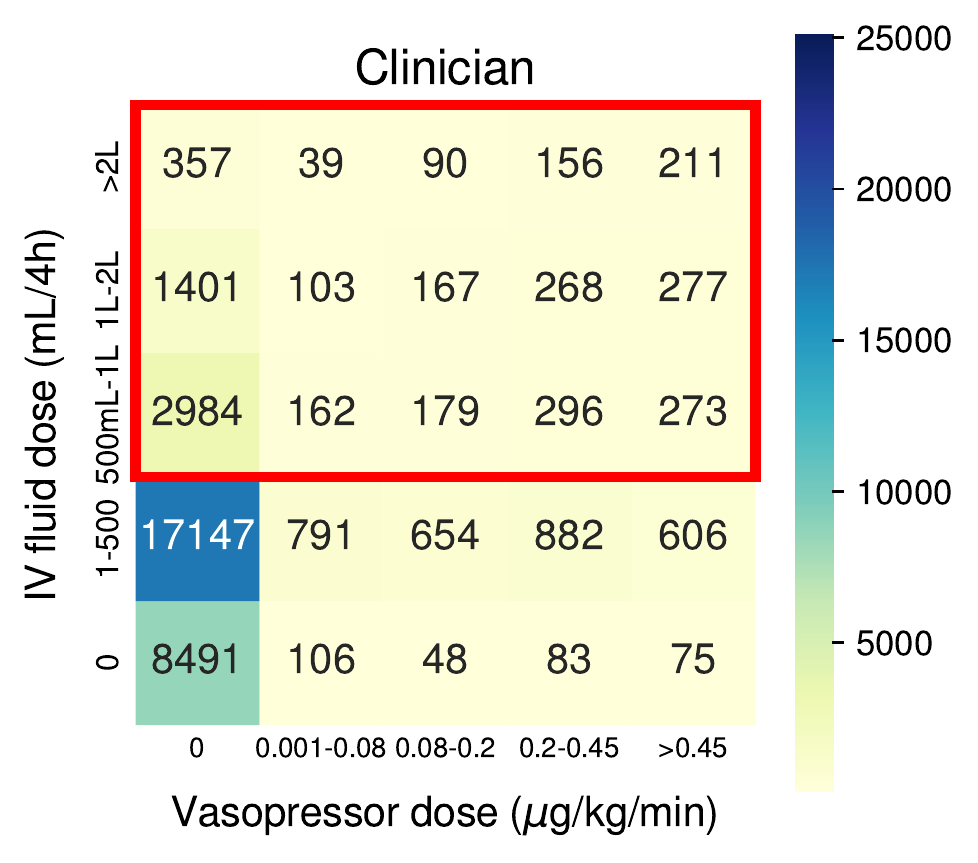} &
    \includegraphics[scale=0.38,trim=5 5 60 12, clip]{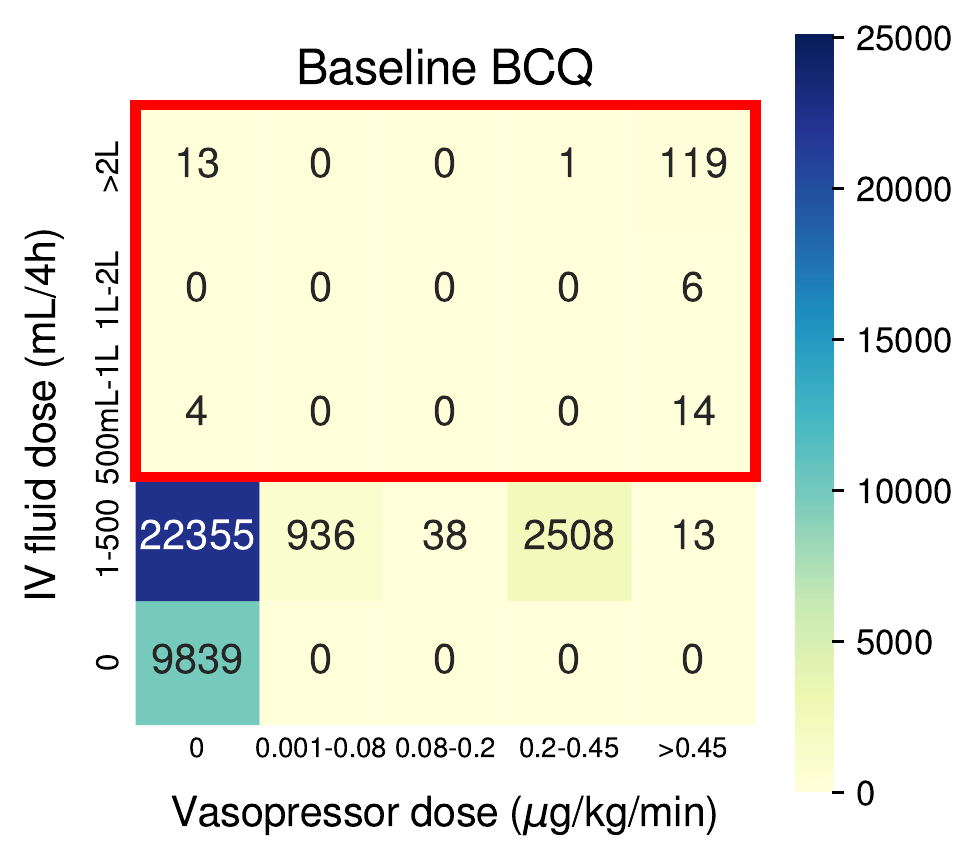} &
    \includegraphics[scale=0.38,trim=5 5 60 12, clip]{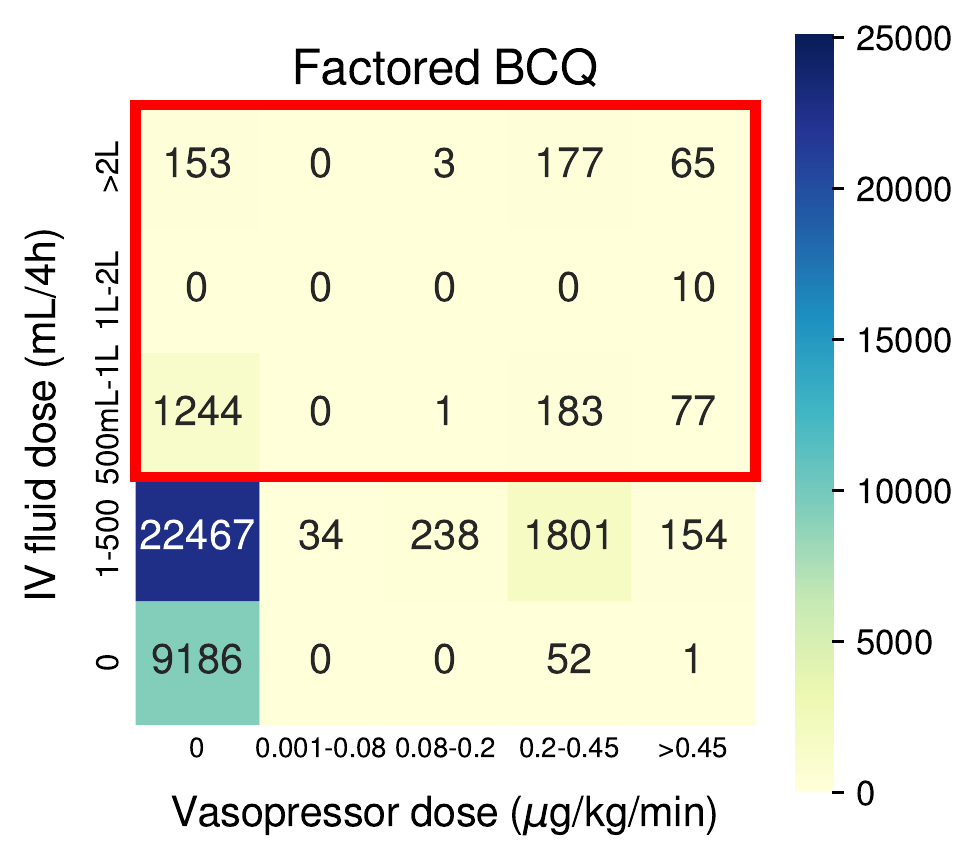} &
    \includegraphics[scale=0.38,trim=225 5 0 0, clip]{fig/mimic_test_policy_BCQf.pdf} &
    \includegraphics[scale=0.38,trim=-5 5 0 0]{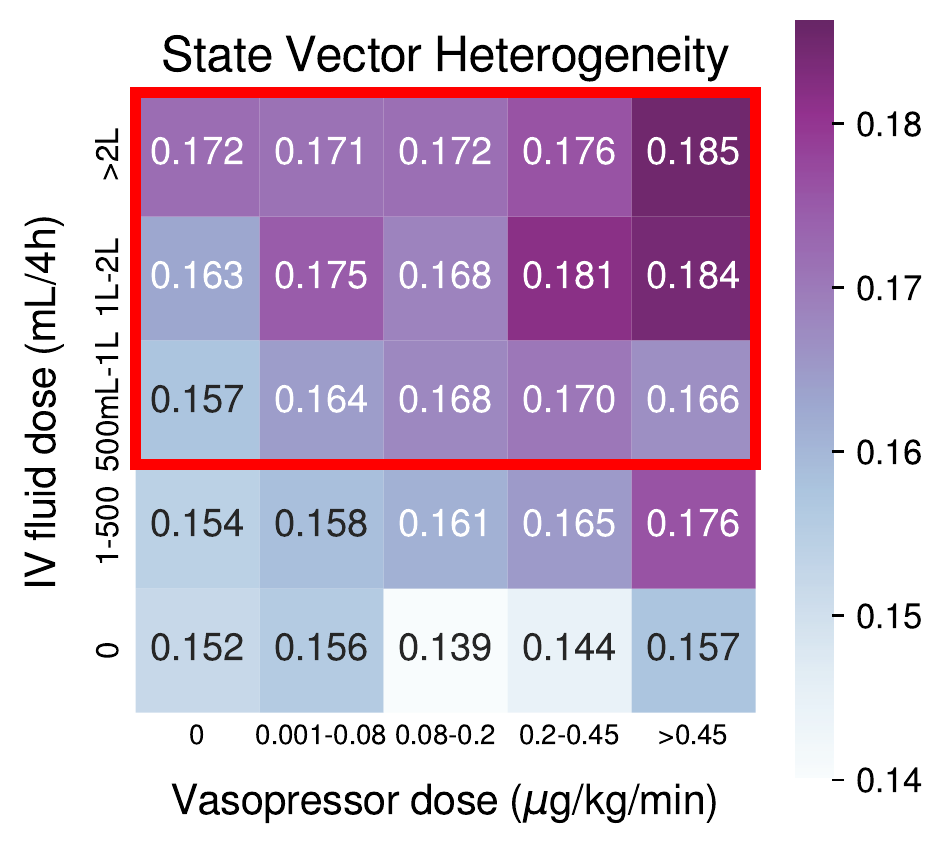}
    \end{tabular}
    }
    \vspace{-0.5em}
    \caption{(a) Qualitative comparison of policies. (b) Per-action state heterogeneity, measured as the standard deviation of all state embeddings from which a particular action is observed in the dataset, averaged over state embedding dimensions. Actions with higher IV fluid doses exhibited greater heterogeneity in the observed states from which those actions were taken (by the clinician policy). }
    \label{fig:mimic_policy_heatmap}
    \vspace{-1.5em}
\end{figure*}

\textbf{Takeaways.} Applied to real clinical data, our proposed approach outperforms the baseline quantitatively and recommends treatments that align better with clinical knowledge. While promising, these results are based in part on OPE, which has many issues \citep{gottesman2018evaluating,gottesman2019guidelines}. We stress that further investigation and close collaboration with clinicians are essential before such RL algorithms are used in practice.

\vspace{-0.5ex}
\section{Related Work} \label{sec:related}
\vspace{-1.5ex}

For many years, the factored RL literature focused exclusively on state space factorization \citep{koller1999computing,guestrin2003efficient,strehl2007efficient,delgado2011efficient}. More recently, interest in action space factorization has grown, as RL is applied in increasingly more complex planning and control tasks. In particular, researchers have previously considered the model-based setting with known MDP factorizations in which both state and action spaces are factorized \citep{raghavan2012planning,raghavan2013symbolic,osband2014near,lu2021causal}. For model-free approaches, others have studied methods for factored actions with a policy component (i.e., policy-based or actor-critic) \citep{sallans2004reinforcement,sharma2017learning, van2020q, pierrot2021factored, spooner2021factored}. In contrast, our work considers value-based methods as those have been the most successful in offline RL \citep{levine2020offlineRL}. 

Among prior work with a value-based component (e.g., Q-network), the majority pertains to multi-agent \citep{sunehag2018value,rashid2018qmix,son2019qtran,matignon2012independent,tampuu2017multiagent} or multi-objective \citep{spooner2021factored} problems that impose known, explicit assumptions on the state space or the reward function. Notably, \citet{son2019qtran} established theoretical conditions for factored optimal actions (called ``Individual-Global-Max'') for multi-agent RL and motivated subsequent works \citep{wang2021qplex,wang2021towards}; their result differs from our contribution on the unbiasedness of factored Q-functions (instead of actions) for single-agent RL. In the online setting for single-agent deep RL, \citet{sharma2017learning} and \citet{tavakoli2018action} incorporated factored action spaces into Q-network architecture designs, but did not provide a formal justification for the linear decomposition. Others have empirically compared various ``mixing'' functions to combine the values of sub-actions \citep{sharma2017learning,rashid2018qmix}. In contrast, while our work only considers the linear decomposition function, we examine its theoretical properties and provide justifications for using this approach in practical problems, especially in offline settings. Our linear Q-decomposition is related to that of \citet{swaminathan2017off} who also applied a linearity assumption for off-policy evaluation, but for combinatorial contextual bandits rather than RL. Our insights on the bias-variance trade-off is also related to a concurrent work by \citet{saito2022off} who proposed efficient off-policy evaluation for bandit problems with large (but not necessarily factored) action spaces. In appendix \cref{tab:compare}, we further outline the differences of our work compared to the existing literature. 

Finally, the sufficient conditions we establish are related to, but different from, those identified by \citet{van2017hybrid} and \citet{juozapaitis2019explainable} who considered reward decompositions in the absence of factored actions. Related, \citet{metz2018discrete} proposed an approach that sequentially predicts values for discrete dimensions of a transformed continuous action space, but assume an \textit{a priori} ordering of action dimensions, which we do not; \citet{pierrot2021factored} studied a different form of action space factorization where sub-actions are sequentially selected in an autoregressive manner. Complementary to our work, \citet{tavakoli2021learning} proposed to organize the sub-actions and interactions as a hypergraph and linearly combining the values; our theoretical results on the linear decomposition nonetheless apply to their setting where the sub-action interactions are explicitly identified and encoded.

\vspace{-0.5ex}
\section{Conclusion}
\vspace{-1.5ex}

To better leverage factored action spaces in RL, we developed an approach to learning policies that incorporates a simple linear decomposition of the Q-function. We theoretically analyze the sufficient and necessary conditions for this parameterization to yield unbiased estimates, study its effect on variance reduction, and identify scenarios when any resulting bias does not lead to suboptimal performance. We also note how domain knowledge may be used to inform the applicability of our approach in practice, for problems where any possible bias is negligible or does not affect optimality. Through empirical experiments on two offline RL problems involving a simulator and real clinical data, we demonstrate the advantage of our approach especially in settings with limited sample sizes. We provide further discussions on limitations, ethical considerations and societal impacts in \cref{sec:discussion}. Though motivated by healthcare, our approach could apply more broadly to scale RL to other applications (e.g., education) involving combinatorial action spaces where domain knowledge may be used to verify the theoretical conditions. Future work should consider the theoretical implications of linear Q decompositions when combined with other offline RL-specific algorithms \citep{levine2020offlineRL}. Given the challenging nature of identifying the best treatments from offline data, our proposed approach may also be combined with other RL techniques that do not aim to identify the single best action (e.g., learning dead-ends \citep{fatemi2021medical} or set-valued policies \citep{tang2020clinician}).

\section*{Acknowledgments}
This work was supported by the National Science Foundation (NSF; award IIS-1553146 to JW; award IIS-2007076 to FDV; award IIS-2153083 to MM) and the National Library of Medicine of the National Institutes of Health (NLM; grant R01LM013325 to JW and MWS). The views and conclusions in this document are those of the authors and should not be interpreted as necessarily representing the official policies, either expressed or implied, of the National Science Foundation, nor of the National Institutes of Health. This work was supported, in part, by computational resources and services provided by \href{https://arc.umich.edu/}{Advanced Research Computing}, a division of Information and Technology Services (ITS) at the University of Michigan, Ann Arbor. The authors would like to thank Adith Swaminathan, Tabish Rashid, and members of the \href{https://wiens-group.engin.umich.edu/}{MLD3 group} for helpful discussions regarding this work, as well as the reviewers for constructive feedback.

\section*{Data and Code Availability}
The code for all experiments is available at \url{https://github.com/MLD3/OfflineRL_FactoredActions}. The sepsis simulator is based on prior work with public implementation at \url{https://github.com/clinicalml/gumbel-max-scm}. The MIMIC-III database used in the real-data experiments of this paper is publicly available through the PhysioNet website: \url{https://physionet.org/content/mimiciii/1.4/}. The cohort definition, extraction and preprocessing code are based on prior work with publicly available implementation at \url{https://github.com/microsoft/mimic_sepsis}.

{\small
\bibliography{ref}
\bibliographystyle{unsrtnat}
}


\section*{Checklist}


\begin{enumerate}

\item For all authors...
\begin{enumerate}
  \item Do the main claims made in the abstract and introduction accurately reflect the paper's contributions and scope?
    \answerYes{}
  \item Did you describe the limitations of your work?
    \answerYes{Section 3.4, Appendix A}
  \item Did you discuss any potential negative societal impacts of your work?
    \answerYes{Appx A}
  \item Have you read the ethics review guidelines and ensured that your paper conforms to them?
    \answerYes{}
\end{enumerate}

\item If you are including theoretical results...
\begin{enumerate}
  \item Did you state the full set of assumptions of all theoretical results?
    \answerYes{Sec 3, Appx B}
        \item Did you include complete proofs of all theoretical results?
    \answerYes{Appendix B}
\end{enumerate}

\item If you ran experiments...
\begin{enumerate}
  \item Did you include the code, data, and instructions needed to reproduce the main experimental results (either in the supplemental material or as a URL)?
    \answerYes{see Data \& Code Availability}
  \item Did you specify all the training details (e.g., data splits, hyperparameters, how they were chosen)?
    \answerYes{Sec 4.1-4.2, Appendix D.1-D.2}
        \item Did you report error bars (e.g., with respect to the random seed after running experiments multiple times)?
    \answerYes{Fig 6}
        \item Did you include the total amount of compute and the type of resources used (e.g., type of GPUs, internal cluster, or cloud provider)?
    \answerNo{our experiments only involved small neural networks}
\end{enumerate}

\item If you are using existing assets (e.g., code, data, models) or curating/releasing new assets...
\begin{enumerate}
  \item If your work uses existing assets, did you cite the creators?
    \answerYes{MIMIC-III by \citet{johnson2016mimic}, sepsisSim by \citet{oberst2019gumbel}}
  \item Did you mention the license of the assets?
    \answerNo{please refer to the website of the creator of MIMIC-III \url{https://physionet.org/content/mimiciii/1.4/}}
  \item Did you include any new assets either in the supplemental material or as a URL? {\color{gray}[N/A]}
  \item Did you discuss whether and how consent was obtained from people whose data you're using/curating?
    \answerNo{please refer to the website of the creator of MIMIC-III \url{https://physionet.org/content/mimiciii/1.4/}}
  \item Did you discuss whether the data you are using/curating contains personally identifiable information or offensive content?
    \answerYes{Sec 4.2, MIMIC-III is de-identified}
\end{enumerate}

\item If you used crowdsourcing or conducted research with human subjects...
\begin{enumerate}
  \item Did you include the full text of instructions given to participants and screenshots, if applicable?
    \answerNA{}
  \item Did you describe any potential participant risks, with links to Institutional Review Board (IRB) approvals, if applicable?
    \answerNA{}
  \item Did you include the estimated hourly wage paid to participants and the total amount spent on participant compensation?
    \answerNA{}
\end{enumerate}

\end{enumerate}


\clearpage
\appendix

\section{Additional Discussion} \label{sec:discussion}

\begin{table}[b]
\vspace{-1.5em}
{
\centering
\caption{Qualtitative comparisons of this work with the existing literature. }
\label{tab:compare}
\begin{tabular}{l|cccccc}
\toprule
 & \makecell[l]{Policy- \\based?} & \makecell[l]{Model- \\based?} & \makecell[l]{Value- \\based?} & \makecell[l]{Linear value \\ decomposition?}  & \makecell[l]{Known state \\ factorization or \\ abstraction?} & \makecell[l]{Unbiasedness \\ guarantees?} \\
\midrule
\citep{osband2014near,lu2021causal} & & \cmark & & & \cmark & \\ 
\citep{sallans2004reinforcement} & \cmark & & & & \cmark &\\ 
\citep{pierrot2021factored,spooner2021factored} & \cmark & & & & \xmark & \\ 
\citep{koller1999computing,guestrin2003efficient,strehl2007efficient,delgado2011efficient} & & & \cmark (V) & \cmark & \cmark & \\ 
\citep{sharma2017learning} & \cmark & & \cmark (Q) & \phantom{$^1$}\cmark$^1$ & \cmark & \\ 
VDN \citep{sunehag2018value} & & & \cmark (Q) & \cmark & \cmark & \\ 
QMIX \citep{rashid2018qmix} & & & \cmark (Q) & *$^2$ & \cmark & \\ 
BDQ \citep{tavakoli2018action} & & & \cmark (Q) & *$^2$ & \xmark & \\ 
\textbf{This work} & & & \cmark (Q) & \cmark & \xmark & \cmark \\
\bottomrule
\end{tabular}
}
\footnotesize{$^1$Empirically tested various “combination” functions including linear. $^2$Both QMIX and BDQ do not aggregate the sub-Q functions; instead, they aggregate the argmax sub-actions.}
\vspace{-1.5em}
\end{table}

\textbf{Computational Efficiency.} While our main analysis focuses on statistical efficiency (variance) and its trade-off with approximation error (bias), here we outline some considerations on computational efficiency. To compute the values for all output heads in \cref{fig:arch}, there is a clear saving of computational cost by our approach with a linear complexity $O(D)$ (measured in flops) in the number of sub-actions, whereas the baseline has an exponential complexity $O(\exp(D))$. We consider two common inference operations after the values of the output heads are computed: $\max_{\avec} Q(s,\avec)$ and $\argmax_{\avec} Q(s,\avec)$. For both operations, the baseline has an exponential time complexity of $O(\exp(D))$. For our proposed approach, an optimized implementation has a linear time complexity of $O(D)$: one can perform argmax/max per sub-action and then concatenate/sum the results. In our current code release, we did not implement the optimized version; instead, we made use of the sub-action featurization matrix defined in \cref{appx:subspace_proof} so that automatic differentiation can be applied directly. This implementation is computationally more expensive than our analysis above and than the baseline: the forward pass includes a dense matrix multiplication with time complexity $O(D \exp(D))$ flops, followed by an $O(\exp(D))$ argmax/max operation. In settings where computational complexity might be a bottleneck (especially at inference time), we recommend using the featurization matrix implementation for learning and the optimized version for inference. 

\textbf{Limitations.} Our theoretical analysis in \cref{sec:theory} focuses on the ``realizability'' condition of the linear function class \citep{chen2019infotheory}, where we are interested in guarantees of zero approximation error, i.e., whether the true $Q^*$ lies within the linear function class. In principle, it is possible to find $Q^*$ given a realizable function class (e.g., by enumerating all member functions). However, when Q-learning-style iterative algorithms are used in practice, its convergence relies on a stronger ``completeness'' condition, as discussed in \citep{chen2019infotheory,xie2021batch,zhan2022offline}. We did not investigate how our proposed form of parameterization (and the specific shape of bias introduced) interacts with the learning procedure, and this is an interesting direction for future work (\citet{wang2021towards} studied this for linear value factorization in the context of FQI but for multi-agent RL). 

\textbf{Ethical Considerations and Societal Impact.} In general, policies computationally derived using RL must be carefully validated before they can be used in high-stakes domains such as healthcare. Our linear parameterization implicitly makes an independence assumption with respect to the sub-actions, allowing the Q-function to generalize to sub-action combinations that are underexplored (and even unexplored) in the offline data (as shown in \cref{sec:sepsis-sim}). When the independence assumptions are valid (according to domain knowledge), this is a case of a ``free lunch'' as we can reduce variance without introducing any bias. However, inaccurate or incomplete domain knowledge may render the independence assumptions invalid and cause the agent to incorrectly generalize to dangerous actions (e.g., learned policy recommends drug combinations with adverse side effects, see \cref{sec:practical}). This misuse may be alleviated by incorporating additional offline RL safeguards to constrain the learned policy (e.g., BCQ was used in \cref{sec:mimic-sepsis} to restrict the learned policy to not take rarely observed sub-action combinations). Still, to apply RL in healthcare and other safety-critical domains, it is important to consult and closely collaborate with domain experts (e.g., clinicians for healthcare problems) to come up with meaningful tasks and informed assumptions, and perform thorough evaluations involving both the quantitative and qualitative aspects \citep{gottesman2018evaluating,gottesman2019guidelines}.

\newpage

\section{Detailed Theoretical Analyses} \label{appx:theory}

\subsection{Sufficient Condition: The Trivial Setting - $D$ Parallel MDPs} \label{appx:sufficient-trivial}
To build intuition, we first consider a related setting where $D$ MDPs are running in parallel. If every MDP evolves independently as controlled by its respective policy, then the total return from all $D$ MDPs should naturally be the sum of the individual returns from each MDP. Formally, we state the following proposition involving fully factored MDPs and factored policies. Here, we use the vector notation $\svec = [s_1, \cdots, s_D]$ to indicate the explicit state space factorization. 

\begin{definition} \label{def:factored-MDP}
Given MDPs $\Mcal_1 \cdots \Mcal_D$ where each $\Mcal_d$ is defined by $(\Scal_d, \Acal_d, p_d, r_d)$, a fully factored MDP $\Mcal =  \bigotimes_{d=1}^{D} \Mcal_d$ is defined by $\Scal, \Acal, p, r$ such that $\Scal = \bigotimes_{d=1}^{D} \Scal_d$, $\Acal = \bigotimes_{d=1}^{D} \Acal_d$, $p(\svec'|\svec,\avec) = \prod_{d=1}^{D} p_d(s'_d|s_d,a_d)$, and $r(\svec,\avec) = \sum_{d=1}^{D} r_d(s_d,a_d)$.
\end{definition}

\begin{definition} \label{def:factored-policy}
Given MDPs $\Mcal_1 \cdots \Mcal_d$ and policies $\pi_1 \cdots \pi_D$ where each $\pi_d: \Scal_d \rightarrow \Delta(\Acal_d)$, then a factored policy $\pi = \bigotimes_{d=1}^{D} \pi_j$ for the MDP $\Mcal = \bigotimes_{d=1}^{D} \Mcal_j$ is $\pi:\Scal\rightarrow\Delta(\Acal)$ such that $\pi(\avec|\svec) = \prod_{d=1}^{D} \pi_d(a_d|s_d)$.
\end{definition}

\begin{proposition} \label{thm:sufficient-factored}
The Q-function of policy $\pi = \bigotimes_{d=1}^{D} \pi_j$ for MDP $\Mcal = \bigotimes_{d=1}^{D} \Mcal_d$ can be expressed as \(Q_{\Mcal}^{\pi}(\svec,\avec) = \sum_{d=1}^{D} Q_{\Mcal_d}^{\pi_d}(s_d,a_d)\).
\end{proposition}

To match the form in \cref{eqn:Q-decomposition}, we can set $q_d(\svec,a_d) = Q_{\Mcal_d}^{\pi_d}(s_d,a_d)$. Importantly, each $Q_{\Mcal_d}^{\pi_d}$ does not depend on any $a_{d'}$ where $d' \neq d$. Note that although our definition of $q_d$ is allowed to condition on the entire state space $\svec$, each $Q_{\Mcal_d}^{\pi_d}$ only depends on $s_d$. \cref{thm:sufficient-factored} can be seen as a corollary to \cref{thm:sufficient-abstract} where the abstractions are defined using the sub-state spaces, such that $\phi_d: \Scal \to \Scal_d$.

\begin{proof}[Proof of \cref{thm:sufficient-factored}]

Without loss of generality, we consider the setting with $D=2$ such that $\mathcal{A} = \mathcal{A}_1 \times \mathcal{A}_2$; extension to $D>2$ is straightforward. The proof is based on mathematical induction on a sequence of $h$-step Q-functions of $\pi$ defined as $Q^{\pi,(h)}_{\Mcal}(\svec,\avec) = \mathbb{E} [\sum_{t=1}^{h}\gamma^{t-1}r_t | \svec_1 = \svec, \avec_1 = \avec, \avec_t \sim \pi]$. 

\begin{adjustwidth}{1em}{0pt}
\textit{Base case.} For $h=1$, the one-step Q-function is simply the reward, which by assumption \(r(\svec,\avec) = r_1(s_1,a_1) + r_2(s_2,a_2)\). Therefore, \( Q^{\pi,(1)}_{\Mcal}(s,\avec) = Q^{\pi_1, (1)}_{\Mcal_1}(s_1,a_1) + Q^{\pi_2, (1)}_{\Mcal_2}(s_2,a_2) \). 

\textit{Inductive step.} 
Suppose \(Q^{\pi,(h)}_{\Mcal}(\svec,\avec) = Q^{\pi_1,(h)}_{\Mcal_1}(s_1,a_1) + Q^{\pi_2,(h)}_{\Mcal2}(s_2,a_2)\) holds. We can express $Q^{\pi,(h+1)}_{\Mcal}$ in terms of $Q^{\pi,(h)}_{\Mcal}$ using the Bellman equation: 

\begin{equation*}
    Q^{\pi,(h+1)}_{\Mcal}(\svec,\avec) = \underbrace{r(\svec,\avec)}_{\tiny \circled{1}} + \gamma \underbrace{\sum_{\svec'} p(\svec'|\svec,\avec) V^{\pi,(h)}_{\Mcal}(\svec')}_{\tiny \circled{2}}
\end{equation*}

where \( \displaystyle V^{\pi,(h)}_{\Mcal}(\svec') = \sum_{\avec'} \pi(\avec'|\svec') Q^{\pi,(h)}_{\Mcal}(\svec',\avec') \). 

By \cref{def:factored-MDP}, \smallcircled{1} can be written as a sum \(r(\svec,\avec) = r_1(s_1,a_1) + r_2(s_2,a_2)\) where each summand depends on only either $a_1$ or $a_2$ but not both. Next we show that \smallcircled{2} also decomposes in a similar manner. For a given $\svec$ we have:
\begin{align*}
    & V^{\pi,(h)}_{\Mcal}(\svec) = \sum_{\avec} \pi(\avec|\svec) Q^{\pi,(h)}_{\Mcal}(\svec,\avec) \\
    &= \sum_{a_1, a_2} \pi_1(a_1|s_1) \pi_2(a_2|s_2) \Big(Q^{\pi_1,(h)}_{\Mcal_1}(s_1,a_1) + Q^{\pi_2,(h)}_{\Mcal_2}(s_2,a_2)\Big) \\
    &= \text{$\textstyle \Big(\cancelto{1}{\sum_{a_2} \pi_2(a_2|s_2)}\Big) \sum_{a_1} \pi_1(a_1|s_1) Q^{\pi_1,(h)}_{\Mcal_1}(s_1,a_1)$} + \text{$\textstyle \Big(\cancelto{1}{\sum_{a_1} \pi_1(a_1|s_1)}\Big) \sum_{a_2} \pi_2(a_2|s_2) Q^{\pi_2,(h)}_{\Mcal_2}(s_2,a_2)$} \\
    &= \Big(\sum_{a_1}\pi_1(a_1|s_1) Q^{\pi_1,(h)}_{\Mcal_1}(s_1,a_1)\Big) + \Big(\sum_{a_2}\pi_2(a_2|s_2) Q^{\pi_2,(h)}_{\Mcal_2}(s_2,a_2)\Big) \ ,
\end{align*}

where we use the fact that $\pi_1(a_1|s_1) Q^{\pi_1,(h)}_{\Mcal_1}(s_1,a_1)$ is independent of $\pi_2(a_2|s_2)$ (and vice versa), and that $\pi_d(\cdot|s_d)$ is a probability simplex. Letting \(V^{\pi_d,(h)}_{\Mcal_d}(s_d) = \sum_{a_d}\pi_1(a_d|s_d) Q^{\pi_d,h}_{\Mcal_d}(s_d,a_d)\), then \(V^{\pi,(h)}_{\Mcal}(s') = V^{\pi_1,(h)}_{\Mcal_1}(s'_1) + V^{\pi_2,(h)}_{\Mcal_2}(s'_2)\). 

Substituting into \smallcircled{2}, we have:
\begin{align*}
    & \sum_{\svec'} p(\svec'|\svec,\avec) V^{\pi,(h)}_{\Mcal}(\svec') \\
    &= \sum_{s'_1,s'_2} p_1(s'_1|s_1,a_1) p_2(s'_2|s_2,a_2) \left(V^{\pi_1,(h)}_{\Mcal_1}(s'_1) + V^{\pi_2,(h)}_{\Mcal_2}(s'_2)\right) \\
    &= \text{$\textstyle \Big(\cancelto{1}{\sum_{s'_2} p_2(s'_2|s_2,a_2)}\Big) \sum_{s'_1} p_1(s'_1|s_1,a_1) V^{\pi_1,(h)}_{\Mcal_1}(s'_1)$} + \text{$\textstyle \Big(\cancelto{1}{\sum_{s'_1} p_1(s'_1|s_1,a_1)}\Big) \sum_{s'_2} p_2(s'_2|s_2,a_2) V^{\pi_2,(h)}_{\Mcal_2}(s'_2)$} \\
    &= \Big(\sum_{s'_1} p_1(s'_1|s_1,a_1) V^{\pi_1,(h)}_{\Mcal_1}(s'_1)\Big) + \Big(\sum_{s'_2} p_2(s'_2|s_2,a_2) V^{\pi_2,(h)}_{\Mcal_2}(s'_2)\Big)
\end{align*}
where we make use of a similar independence property between $p_1(s'_1|s_1,a_1)V^{\pi_1,(h)}_{\Mcal_1}(s'_1)$ and $p_2(s'_2|s_2,a_2)$, and the fact that that $p_d(\cdot|s_d,a_d)$ is a probability simplex. 

Therefore, we have \( Q^{\pi,(h+1)}_{\Mcal}(\svec,\avec) = Q^{\pi_1,(h+1)}_{\Mcal_1}(s_1,a_1) + Q^{\pi_2,(h+1)}_{\Mcal_2}(s_2,a_2) \) as desired, where \( Q^{\pi_d,(h+1)}_{\Mcal_d}(s_d,a_d) = r_d(s_d,a_d) + \gamma \sum_{s'_d} p_d(s'_d|s_d,a_d) \sum_{a'_d} \pi_j(a'_d|s'_d) Q^{\pi_d,(h)}_{\Mcal_j}(s'_d,a'_d) \). 
\end{adjustwidth}
By mathematical induction, this decomposition holds for any $h$-step $Q$-function. Letting $h \rightarrow \infty$ shows that this holds for the full Q-function. 
\end{proof}

\subsection{Sufficient Condition: The Abstraction Setting} \label{appx:sufficient-abstract}
\vspace{-0.5em}
We first review some important background on state abstractions. Using the properties of state abstractions, we can prove the main sufficient condition in \cref{thm:sufficient-abstract}. This proof follows largely from the techniques used in proving \cref{thm:sufficient-factored}, with the exception of how marginalization over the state space is handled. 

\textbf{\textit{Background on State Abstractions.}} A state abstraction (also known as state aggregation) \citep{li2006abstraction}, is a mapping $\phi: \Scal \to \Zcal$ that converts each element of the primitive state space $\Scal$ to an element of the abstract state space $\Zcal$. Intuitively, if two states $s_1$ and $s_2$ are mapped to the same element under $\phi$, i.e., $\phi(s_1) = \phi(s_2)$, then they are treated as the same (abstract) state under the abstraction. Therefore, we can view an abstraction as a partitioning of the primitive state space into non-overlapping subsets. Since a state abstraction is a many-to-one mapping, we define its inverse as $\phi^{-1}(z) = \{\tilde{s}: \phi(\tilde{s}) = z\}$, a set containing all primitive states that are mapped to the abstract state $z$. 

We have the following property of summations involving state abstractions, where for any function $f: \Scal \to \mathbb{R}$,
\[
    \sum_{s \in \Scal} f(s) = \adjustlimits \sum_{z \in \Zcal} \sum_{\tilde{s} \in \phi^{-1}(z)} f(\tilde{s})
\]
To understand this property, let us consider the sum of $f(s)$ for all states in $\Scal$ which can be obtained in two different ways: i) directly iterating through the elements of $\Scal$, ii) first iterating through the partitions of $\Scal$ (induced by the abstraction), and then iterating through the elements in each partition, giving rise to the double summation. This property allows us to change the index of summation from primitive states to abstract states. For multiple abstractions $\boldsymbol{\phi} = [\phi_1, \cdots, \phi_D]$ where $\phi_{d} \neq \phi_{d'}$ if $d \neq d'$, denoting $\boldsymbol{z} = \boldsymbol{\phi}(s) = [z_1,\dots,z_D]$, we can similarly define the inverse abstraction $\boldsymbol{\phi}^{-1}(\boldsymbol{z}) = \{\tilde{s}: \boldsymbol{\phi}(\tilde{s}) = \boldsymbol{z}\}$, and the summation property similarly applies.

\begin{proof}[Proof of \cref{thm:sufficient-abstract}]

Without loss of generality, we consider the setting with $D=2$ so $\mathcal{A} = \mathcal{A}_1 \times \mathcal{A}_2$; extension to $D>2$ is straightforward. The proof is based on mathematical induction on a sequence of $h$-step Q-functions of $\pi$ denoted by $Q^{(h)}(s,\avec) = \mathbb{E} [\sum_{t=1}^{h}\gamma^{t-1}r_t | s_1 = s, \avec_1 = \avec, \avec_t \sim \pi]$. 

\begin{adjustwidth}{1em}{0pt}
\textit{Base case.} For $h=1$, the one-step Q-function is simply the reward, which by assumption \(r(s,\avec) = r_1(z_1,a_1) + r_2(z_2,a_2)\). We can trivially set \( q^{(1)}_d(z_d,a_d) = r_d(z_d,a_d) \) such that \( Q^{(1)}(s,\avec) = q^{(1)}_1(z_1,a_1) + q^{(1)}_2(z_2,a_2) \). 

\textit{Inductive step.} 
Suppose \(Q^{(h)}(s,\avec) = q^{(h)}_1(z_1,a_1) + q^{(h)}_2(z_2,a_2)\) holds. We can express $Q^{(h+1)}$ in terms of $Q^{(h)}$ using the Bellman equation: 

\begin{equation*}
    Q^{(h+1)}(s,\avec) = \underbrace{r(s,\avec)}_{\tiny \circled{1}} + \gamma \underbrace{\sum_{s'} p(s'|s,\avec) V^{(h)}(s')}_{\tiny \circled{2}}
\end{equation*}

where \( \displaystyle V^{(h)}(s') = \sum_{\avec'} \pi(\avec'|s') Q^{(h)}(s',\avec') \). 

\smallcircled{1} can be written as a sum \(r(s,\avec) = r_1(z_1,a_1) + r_2(z_2,a_2)\) where each summand depends on only either $a_1$ or $a_2$ but not both. Next we show \smallcircled{2} also decomposes in a similar manner. 

For a given $s$ we have:
\begin{align*}
    & V^{(h)}(s) = \sum_{\avec} \pi(\avec|s) Q^{(h)}(s,\avec) \\
    &= \sum_{a_1, a_2} \pi_1(a_1|z_1) \pi_2(a_2|z_2) \Big(q^{(h)}_1(z_1,a_1) + q^{(h)}_2(z_2,a_2)\Big) \\
    &= \text{$\textstyle \Big(\cancelto{1}{\sum_{a_2} \pi_2(a_2|z_2)}\Big) \sum_{a_1} \pi_1(a_1|z_1) q^{(h)}_1(z_1,a_1)$} + \text{$\textstyle \Big(\cancelto{1}{\sum_{a_1} \pi_1(a_1|z_1)}\Big) \sum_{a_2} \pi_2(a_2|z_2) q^{(h)}_2(z_2,a_2)$} \\
    &= \sum_{a_1}\pi_1(a_1|z_1) q^{(h)}_1(z_1,a_1) + \sum_{a_2}\pi_2(a_2|z_2) q^{(h)}_2(z_2,a_2) \ ,
\end{align*}
where we used the property that $\pi_1(a_1|z_1) q^{(h)}_1(z_1,a_1)$ is independent of $\pi_2(a_2|z_2)$ (and vice versa), and that $\pi_d(\cdot|z_d)$ is a probability simplex. Letting \(v^{(h)}_d(z_d) = \sum_{a_d}\pi_d(a_d|z_d) q^{(h)}_d(z_d,a_d)\), then we can write \(V^{(h)}(s') = v^{(h)}_1(z'_1) + v^{(h)}_2(z'_2)\). 

Substituting into \smallcircled{2}, we have:
\begin{align*}
    & \sum_{s'} p(s'|s,\avec) V^{(h)}(s') = \sum_{\boldsymbol{z}'} \sum_{\tilde{s} \in \boldsymbol{\phi}^{-1}(\boldsymbol{z}')} p(\tilde{s}|s,\avec) V^{(h)}(\tilde{s}) \\
    &= \sum_{\boldsymbol{z}'} \sum_{\tilde{s} \in \boldsymbol{\phi}^{-1}(\boldsymbol{z}')} p(\tilde{s}|s,\avec) V^{(h)}(\tilde{s}) \\
    &= \sum_{\boldsymbol{z}'} \Big(\sum_{\tilde{s} \in \boldsymbol{\phi}^{-1}(\boldsymbol{z}')} p(\tilde{s}|s,\avec)\Big) V^{(h)}(\tilde{s}) \\
    &= \sum_{z'_1,z'_2} p_1(z'_1|z_1,a_1) p_2(z'_2|z_2,a_2) \left(v^{(h)}_1(z'_1) + v^{(h)}_2(z'_2)\right) \\
    &= \text{$\textstyle \Big(\cancelto{1}{\sum_{z'_2} p_2(z'_2|z_2,a_2)}\Big) \sum_{z'_1} p_1(z'_1|z_1,a_1) v^{(h)}_{1}(z'_1)$} + \text{$\textstyle \Big(\cancelto{1}{\sum_{z'_1} p_1(z'_1|z_1,a_1)}\Big) \sum_{z'_2} p_2(z'_2|z_2,a_2) v^{(h)}_{2}(z'_2)$} \\
    &= \Big(\sum_{z'_1} p_1(z'_1|z_1,a_1) v^{(h)}_1(z'_1)\Big) + \Big(\sum_{z'_2} p_2(z'_2|z_2,a_2) v^{(h)}_2(z'_2)\Big)
\end{align*}
where on the first line we used the property of state abstractions to replace the index of summation, and from the second to the third line we used the fact that for all $\tilde{s} \in \boldsymbol{\phi}^{-1}(\boldsymbol{z}')$ that have the same abstract state vector $\boldsymbol{z}'$, their value $V^{(h)}(s') = v^{(h)}_1(z'_1) + v^{(h)}_2(z'_2)$ are equal; this allows us to directly sum their transition probabilities $p(\tilde{s}|s,\avec)$. Following that, we substitute in \cref{eqn:abstract-factored-transition}, and then use a similar independence property as above and that $p_d(\cdot|z_d,a_d)$ is a probability simplex. 

Therefore, we have \( Q^{(h+1)}(s,\avec) = q^{(h+1)}_1(z_1,a_1) + q^{(h+1)}_2(z_2,a_2) \) as desired where \( q^{(h+1)}_d(z_d,a_d) = r_d(z_d,a_d) + \gamma \sum_{z'_d} p_d(z'_d|z_d,a_d) \sum_{a'_d} \pi(a'_d|z'_d) q^{(h)}_d(z'_d,a'_d)\). 
\end{adjustwidth}
By mathematical induction, this decomposition holds for any $h$-step $Q$-function. Letting $h \rightarrow \infty$ shows that this holds for the full Q-function. 
\end{proof}

\subsection{Policy Learning with Bias - Performance Bounds} \label{appx:perf_bounds}

Consider a particular model-based procedure for approximating the optimal Q-function using \cref{eqn:Q-decomposition}: i) finding approximations $\widehat{\Mcal} = (\hat{p}, \hat{r})$ that are close to the true transition/reward functions $p$, $r$ such that there exists some state abstraction set $\boldsymbol{\phi}$ with $\hat{p}$, $\hat{r}$ satisfying \eqref{eqn:abstract-factored-transition} and \eqref{eqn:abstract-factored-reward} with respect to $\boldsymbol{\phi}$, ii) doing planning (e.g., dynamic programming) using the approximate MDP parameters $\hat{p}$ and $\hat{r}$. We can show the following performance bounds; note that these upper bounds are loose and information-theoretic (in that they require knowledge of the implicit factorization). 

\begin{proposition}
If the approximation errors in $\hat{p}$ and $\hat{r}$ are upper bounded by $\epsilon_{p}$ and $\epsilon_{r}$ for all $s\in\Scal, \avec\in\Acal$:
\begin{align*}
    \sum_{s'} \big| p(s'|s,\avec) - \hat{p}(s'|s,\avec) \big| &\leq \epsilon_{p}, \\
    \big| r(s,\avec) - \hat{r}(s,\avec) \big| &\leq \epsilon_{r},
\end{align*}
then the above model-based procedure leads to an approximate Q-function $\hat{Q}$ and an approximate policy $\hat{\pi}$ that satisfy:
\begin{align*}
    \| Q^{*}_{\Mcal} - Q^{*}_{\widehat{\Mcal}} \|_{\infty} &\leq \frac{\epsilon_{r}}{1-\gamma} + \frac{\gamma \epsilon_{p} R_{\max}}{2(1-\gamma)^2}, \\
    \| V^{*}_{\Mcal} - V^{\hat{\pi}}_{\Mcal} \|_{\infty} &\leq \frac{2\epsilon_{r}}{1-\gamma} + \frac{\gamma \epsilon_{p} R_{\max}}{(1-\gamma)^2}.
\end{align*}
\end{proposition}

\begin{proof}
See classical results by \citet{singh1994upper} and \citet{kearns2002near} (the simulation lemma). \renewcommand{\qedsymbol}{}
\end{proof}

\subsection{Subspace of Representable $Q$ Functions} \label{appx:subspace_proof}
To help understand how the linear parameterization of Q-function \cref{eqn:Q-decomposition} affects the representation power of the function class, we first define the following matrices for action space featurization. 

\begin{definition}
The \textit{sub-action mapping matrix for sub-action space $\Acal_d$}, $\boldsymbol{\Psi}_d$, is defined as 
\begin{align*}
    \boldsymbol{\Psi}_j = \begin{bmatrix}
    \rotvert & \boldsymbol{\psi}_d(\avec^1)^{\intercal} & \rotvert\\
    & \vdots & \\
    \rotvert & \boldsymbol{\psi}_d(\avec^{|\Acal|})^{\intercal} & \rotvert \\
    \end{bmatrix} \in \set{0,1}^{|\Acal|\times|\Acal_d|}
\end{align*}
where each row $\boldsymbol{\psi}_d(\avec^i)^{\intercal} \in \set{0,1}^{1 \times |\Acal_d|}$ is a one-hot vector with a value $1$ in column $\proj_{\Acal \to \Acal_d}(\avec^i)$. 
\end{definition}
\begin{remark}
The $i$-th row of $\boldsymbol{\Psi}_d$ corresponds to an action $\avec^i \in \Acal$, and the $j$-th column corresponds to a particular element of the sub-action space $a_d^j \in \Acal_d$. The $(i,j)$-entry of $\boldsymbol{\Psi}_d$ is 1 if and only if the projection of $\avec^i$ onto the sub-action space $\Acal_d$ is $a_d^j$. Since each row is a one-hot vector, the sum of elements in each row is exactly 1, i.e., $\boldsymbol{\psi}_d(\avec^i)^{\intercal} \boldsymbol{1} = 1$. 
\end{remark}

\begin{definition} \label{def:action-matrix}
The \textit{sub-action mapping matrix}, $\boldsymbol{\Psi}$, is defined by a horizontal concatenation of $\boldsymbol{\Psi}_d$ for $d = 1 \dots D$
\begin{align*}
    \boldsymbol{\Psi} = \begin{bmatrix}
    \\
    \boldsymbol{\Psi}_1 & \cdots & \boldsymbol{\Psi}_D \\
    \\
    \end{bmatrix} \in \set{0,1}^{|\Acal| \times (\sum_{d}|\Acal_d|)}
\end{align*}
\end{definition}
\begin{remark}
$\boldsymbol{\Psi}$ describes how to map each action $\avec^i \in \Acal$ to its corresponding sub-actions. Therefore, the sum of elements in each row is exactly $D$, the number of sub-action spaces; $\boldsymbol{\psi}(\avec^i)^{\intercal} \boldsymbol{1} = D$. 
\end{remark}

\begin{definition}
The \textit{condensed sub-action mapping matrix}, $\tilde{\boldsymbol{\Psi}}$, is
\begin{align*}
    \tilde{\boldsymbol{\Psi}} = \left[ \begin{array}{c|ccc}
        & && \\
        \boldsymbol{1} & \tilde{\boldsymbol{\Psi}}_1 & \cdots & \tilde{\boldsymbol{\Psi}}_D \\
        & && \\
    \end{array} \right]  \in \set{0,1}^{|\Acal| \times \left(1+\sum_{d}(|\Acal_d|-1)\right)}
\end{align*}
where the first column contains all 1's, and $\tilde{\boldsymbol{\Psi}}_d$ denotes $\boldsymbol{\Psi}_d$ with the first column removed. 
\end{definition}

\begin{proposition} \label{thm:projection}
$\colspace(\boldsymbol{\Psi}) = \colspace(\tilde{\boldsymbol{\Psi}})$ \; and \; $\rank(\boldsymbol{\Psi}) = \rank(\tilde{\boldsymbol{\Psi}}) = \ncols(\tilde{\boldsymbol{\Psi}})$ \; (i.e., matrix $\tilde{\boldsymbol{\Psi}}$ has full column rank). Consequently, $\boldsymbol{\Psi} \boldsymbol{\Psi}^{+} = \tilde{\boldsymbol{\Psi}} \tilde{\boldsymbol{\Psi}}^{+}$. 
\end{proposition}

\begin{corollary} \label{cor:Q-in-subspace}
Suppose the Q-function $Q$ of a policy $\pi$ at state $s$ is linearly decomposable with respect to the sub-actions, i.e., we can write $Q(s,a) = \sum_{d=1}^{D} q_d(s,a_d)$ for all $a_d\in\Acal_d$, then there exists $\boldsymbol{w}$ and $\tilde{\boldsymbol{w}}$ such that the column vector containing the Q-values for all actions at state $s$ can be expressed as
$\boldsymbol{Q}(s,\Acal) = \boldsymbol{\Psi} \boldsymbol{w} = \tilde{\boldsymbol{\Psi}} \tilde{\boldsymbol{w}}$. In other words, \cref{eqn:Q-decomposition} is equivalent to $\boldsymbol{Q}(s,\Acal) \in \colspace(\tilde{\boldsymbol{\Psi}})$. 
\end{corollary}

\begin{corollary} \label{cor:Q-notin-subspace}
Suppose $\boldsymbol{Q}(s,\Acal) \notin \colspace(\tilde{\boldsymbol{\Psi}})$. Let $\hat{\boldsymbol{w}} = \boldsymbol{\Psi}^{+} \boldsymbol{Q}(s,\Acal)$ and $\hat{\tilde{\boldsymbol{w}}} = \tilde{\boldsymbol{\Psi}}^{+} \boldsymbol{Q}(s,\Acal)$ be the least-squares solutions of the respective linear equations. Then $\boldsymbol{\Psi} \hat{\boldsymbol{w}} = \tilde{\boldsymbol{\Psi}} \hat{\tilde{\boldsymbol{w}}}$. 
\end{corollary}

\begin{remark}
\cref{cor:Q-in-subspace,cor:Q-notin-subspace} imply there are two possible implementations, regardless of whether the true Q-function can be represented by the linear parameterization. Intuitively, both versions try to project the true Q-value vector $\boldsymbol{Q}(s,\Acal)$ for a particular state $s$ onto the subspace spanned by the columns of $\boldsymbol{\Psi}$ or $\tilde{\boldsymbol{\Psi}}$. Since the two matrices have the same column space, the results of the projections are equal. This does not imply $\hat{\boldsymbol{w}}$ and $\hat{\tilde{\boldsymbol{w}}}$ are equal (they cannot be as they have different dimensions), but rather the resultant Q-value estimates are equal, $\hat{\boldsymbol{Q}}(s,\Acal) = \boldsymbol{\Psi} \hat{\boldsymbol{w}} = \tilde{\boldsymbol{\Psi}} \hat{\tilde{\boldsymbol{w}}}$. 
\end{remark}

To make the theorem statements more concrete, we inspect a simple numerical example and verify the theoretical properties. 

\begin{example}
Consider an MDP with  $\Acal = \Acal_1 \times \Acal_2$, where $\Acal_1 = \set{0,1}$ and $\Acal_2 = \set{0,1}$. Consequently, $|\Acal_1| = |\Acal_2| = 2$ and $|\Acal| = 2^2 = 4$. 

Suppose for state $s$ we can write $Q(s,a) = Q(s,[a_1, a_2]) = q_1(s,a_1) + q_2(s,a_2)$ for all $a_1\in\Acal_1, a_2\in\Acal_2$. Then
\begin{align*}
    \boldsymbol{Q}(s,\Acal) &= 
    \begin{bmatrix}
        Q(s, a_1=0, a_2=0) \\
        Q(s, a_1=0, a_2=1) \\
        Q(s, a_1=1, a_2=0) \\
        Q(s, a_1=1, a_2=1) \\
    \end{bmatrix}
    = \begin{bmatrix}
        q_1(s, 0) + q_2(s, 0) \\
        q_1(s, 0) + q_2(s, 1) \\
        q_1(s, 1) + q_2(s, 0) \\
        q_1(s, 1) + q_2(s, 1) \\
    \end{bmatrix}= \begin{bmatrix}
        1 & 0 & 1 & 0 \\
        1 & 0 & 0 & 1 \\
        0 & 1 & 1 & 0 \\
        0 & 1 & 0 & 1 \\
    \end{bmatrix} \begin{bmatrix}
        q_1(s, 0) \\
        q_1(s, 1) \\
        q_2(s, 0) \\
        q_2(s, 1) \\
    \end{bmatrix} \\[1em]
    & = \begin{bmatrix} \vert \\ \rotvert \; \boldsymbol{\Psi} \; \rotvert \\ \vert \end{bmatrix} \begin{bmatrix} \vert \\ \boldsymbol{w} \\ \vert \end{bmatrix}
    \quad \text{where } \boldsymbol{\Psi} = \left[ \begin{matrix}
        1 & 0 & 1 & 0 \\
        1 & 0 & 0 & 1 \\
        0 & 1 & 1 & 0 \\ 
        \coolunder{\boldsymbol{\Psi}_1}{0 & 1} & \coolunder{\boldsymbol{\Psi}_2}{0 & 1}\\
    \end{matrix} \right], \;\;
    \boldsymbol{w} = 
    \begin{bmatrix}
        q_1(s, 0) \\
        q_1(s, 1) \\
        q_2(s, 0) \\
        q_2(s, 1) \\
    \end{bmatrix}
    \begin{matrix}
    \coolrightbrace{x \\ x}{\boldsymbol{w}_1}\\
    \coolrightbrace{y \\ y}{\boldsymbol{w}_2}
\end{matrix}
\end{align*}

We can also write
\begin{align*}
    \boldsymbol{Q}(s,\Acal) = \tilde{\boldsymbol{\Psi}} \tilde{\boldsymbol{w}}, \quad \text{where } \tilde{\boldsymbol{\Psi}} = \left[ \begin{matrix}
        1 & 0 & 0 \\
        1 & 0 & 1 \\
        1 & 1 & 0 \\
        1 & 1 & 1 \\
    \end{matrix} \right], \;\;
    \tilde{\boldsymbol{w}} = 
    \begin{bmatrix}
        v_0(s) \\
        u_1(s) \\
        u_2(s) \\
    \end{bmatrix} = 
    \begin{bmatrix}
        q_1(s,0) + q_2(s,0) \\
        q_1(s,1) - q_1(s,0) \\
        q_2(s,1) - q_2(s,0) \\
    \end{bmatrix}
\end{align*}

One can verify that $\rank(\boldsymbol{\Psi}) = \rank(\tilde{\boldsymbol{\Psi}}) = 3$ and $\colspace(\boldsymbol{\Psi}) = \colspace(\tilde{\boldsymbol{\Psi}})$, because the columns of $\tilde{\boldsymbol{\Psi}}$ are linearly independent, but the columns of $\boldsymbol{\Psi}$ are not linearly independent:
\begin{align*}
    \begin{bmatrix} 1 \\ 1 \\ 0 \\ 0 \end{bmatrix} + 
    \begin{bmatrix} 0 \\ 0 \\ 1 \\ 1 \end{bmatrix} - 
    \begin{bmatrix} 1 \\ 0 \\ 1 \\ 0 \end{bmatrix} =
    \begin{bmatrix} 0 \\ 1 \\ 0 \\ 1 \end{bmatrix}. 
\end{align*}

Furthermore, 
\begin{align*}
    \boldsymbol{\Psi}^{+} = \begin{bmatrix}
        \phantom{-}3/8 & \phantom{-}3/8 & -1/8           & -1/8 \\
        -1/8           & -1/8           & \phantom{-}3/8 & \phantom{-}3/8 \\
        \phantom{-}3/8 & -1/8           & \phantom{-}3/8 & -1/8 \\
        -1/8           & \phantom{-}3/8 & -1/8           & \phantom{-}3/8 \\
    \end{bmatrix}, \;\;
    \tilde{\boldsymbol{\Psi}}^{+} = \begin{bmatrix}
        \phantom{-}3/4 & \phantom{-}1/4 & \phantom{-}1/4 & -1/4 \\
        -1/2           & -1/2           & \phantom{-}1/2 & \phantom{-}1/2 \\
        -1/2           & \phantom{-}1/2 & -1/2           & \phantom{-}1/2 \\
    \end{bmatrix}
\end{align*}
and 
\begin{align*}
    \boldsymbol{\Psi} \boldsymbol{\Psi}^{+} = \tilde{\boldsymbol{\Psi}} \tilde{\boldsymbol{\Psi}}^{+} = 
    \begin{bmatrix}
        \phantom{-}3/4 & \phantom{-}1/4 & \phantom{-}1/4 & -1/4 \\
        \phantom{-}1/4 & \phantom{-}3/4 & -1/4 & \phantom{-}1/4 \\
        \phantom{-}1/4 & -1/4           & \phantom{-}3/4 & \phantom{-}1/4 \\
        -1/4           & \phantom{-}1/4 & \phantom{-}1/4 & \phantom{-}3/4 \\
    \end{bmatrix}. 
\end{align*}
\end{example}

\begin{proof}[Proof of \cref{thm:projection}]
\quad

First note that $\boldsymbol{\Psi}$ is a tall matrix for non-trivial cases, with more rows than columns, because $|\Acal| = \prod_d |\Acal_d| \geq \sum_d |\Acal_d|$ if $|\Acal_d| \geq 2$ for all $d$ (see \href{https://math.stackexchange.com/questions/2998898/show-that-product-is-larger-than-sum}{proof}). Therefore, the rank of $\boldsymbol{\Psi}$ is the number of linear independent columns of $\boldsymbol{\Psi}$. 

We use the following notation to write matrix $\boldsymbol{\Psi}_{d}$ in terms of its columns:
\begin{align*}
    \boldsymbol{\Psi}_{d} = 
    \begin{bmatrix}
        \vert & & \vert \\
        \boldsymbol{c}_{d,1} & \cdots & \boldsymbol{c}_{d,|\Acal_d|} \\
        \vert & & \vert \\
    \end{bmatrix}. 
\end{align*}

The following statements are true: 

\begin{enumerate}[label={\textbf{{Claim} \arabic*:}}, leftmargin=*]
    \item The columns of $\boldsymbol{\Psi}_d$ are pairwise orthogonal, ${\boldsymbol{c}_{d,j}}^{\intercal}\boldsymbol{c}_{d,j'} = 0, \forall j \neq j'$, and they form an orthogonal basis. This is because each row $\boldsymbol{\psi}_d(\avec^i)^{\intercal}$ is a one-hot vector, containing only one $1$; this implies that out of the two entries in row $i$ of $\boldsymbol{c}_{d,j}$ and $\boldsymbol{c}_{d,j'}$, at least one entry is $0$, and their product must be $0$.
    \item The sum of entries in each row of $\boldsymbol{\Psi}_d$ is $1$, and $\sum_{j=1}^{|\Acal_d|} \boldsymbol{c}_{d,j} = \boldsymbol{1}$ a column vector of 1's with matching size. This is a direct consequence of each row $\boldsymbol{\psi}_d(\avec^i)^{\intercal}$ being a one-hot vector. In other words, $\boldsymbol{1} \in \colspace(\boldsymbol{\Psi}_d)$.
    \item The columns of $\boldsymbol{\Psi}$ are not linearly independent. This is because there is not a unique way to write $\boldsymbol{1}$ as a linear combination of the columns of $\boldsymbol{\Psi}$. For example, $\sum_{j=1}^{|\Acal_d|} \boldsymbol{c}_{d,j} = \sum_{j=1}^{|\Acal_{d'}|} \boldsymbol{c}_{d',j} = \boldsymbol{1}$ for some $d' \neq d$, where we used the columns of $\boldsymbol{\Psi}_d$ and $\boldsymbol{\Psi}_{d'}$. 
    \item $\boldsymbol{1} \notin \colspace(\tilde{\boldsymbol{\Psi}}_1 \cdots \tilde{\boldsymbol{\Psi}}_D)$ because the first entry of every column vector in any $\tilde{\boldsymbol{\Psi}}_d$ is $0$ and no linear combination of them can result in a $1$. Consequently, $\boldsymbol{1} \notin \colspace(\tilde{\boldsymbol{\Psi}}_d)$ for any $d$. 
    \item $\boldsymbol{c}_{d,1} \notin \colspace(\boldsymbol{1}, \tilde{\boldsymbol{\Psi}}_{d'}: d' \neq d)$, where $\boldsymbol{c}_{d,1}$ is the column removed from $\boldsymbol{\Psi}_d$ to construct $\tilde{\boldsymbol{\Psi}}_d$. This can also be seen from the first entry of the column vector: the first entry of $\boldsymbol{c}_{d,1}$ is $1$, and all columns of $\tilde{\boldsymbol{\Psi}}_{d'}: d' \neq d$ have the first entry being $0$. 
    \item $\boldsymbol{c}_{d,j} \notin \colspace(\boldsymbol{1}, \tilde{\boldsymbol{\Psi}}_1 \cdots \tilde{\boldsymbol{\Psi}}_D \setminus \set{\boldsymbol{c}_{d,j}})$ for $j > 1$. By expressing $\boldsymbol{c}_{d,j} = (\boldsymbol{1} - \sum_{j'=2, j' \neq j}^{|\Acal_d|} \boldsymbol{c}_{d,j'})  + (- \boldsymbol{c}_{d,1})$, we observe that the first part of the sum lies in the column space, while the second part does not (from the previous claim, $\boldsymbol{c}_{d,1}$ is not in the column space of $\tilde{\boldsymbol{\Psi}}_{d'}$ where $d' \neq d$; this is because within $\tilde{\boldsymbol{\Psi}}_{d}$, the only way is $\boldsymbol{c}_{d,1} = \boldsymbol{1} - \sum_{j'=2}^{|\Acal_d|} \boldsymbol{c}_{d,j'}$ and we have excluded one of the columns $\boldsymbol{c}_{d,j}$ from the column space). 
\end{enumerate}

Combining these claims implies that each column of $\tilde{\boldsymbol{\Psi}}$ cannot be expressed as a linear combination of all other columns, and thus $\tilde{\boldsymbol{\Psi}}$ has full column rank, $\rank(\tilde{\boldsymbol{\Psi}}) = \ncols(\tilde{\boldsymbol{\Psi}}) = 1 + \sum_{d=1}^{D} (|\Acal_d|-1)$. It follows that $\tilde{\boldsymbol{\Psi}}$ contains the linearly independent subset of columns from $\boldsymbol{\Psi}$, and their column spaces and ranks are equal. 

$\boldsymbol{\Psi} \boldsymbol{\Psi}^{+}$ and $\tilde{\boldsymbol{\Psi}} \tilde{\boldsymbol{\Psi}}^{+}$ are orthogonal projection matrices onto the column space of $\boldsymbol{\Psi}$ and $\tilde{\boldsymbol{\Psi}}$, respectively. Since $\colspace(\boldsymbol{\Phi}) = \colspace(\tilde{\boldsymbol{\Psi}})$, it follows that $\boldsymbol{\Psi} \boldsymbol{\Psi}^{+} = \tilde{\boldsymbol{\Psi}} \tilde{\boldsymbol{\Psi}}^{+}$. 
\end{proof}

\subsection{A Necessary Condition for Unbiasedness} \label{appx:necessary}

Consider the matrix form of the Bellman equation (cf. Sec 2 of \citet{lagoudakis2003least}): 
\begin{equation*}
    \boldsymbol{Q} = \boldsymbol{R} + \gamma \boldsymbol{P}^{\pi} \boldsymbol{Q}
\end{equation*}

where $\boldsymbol{Q} \in \mathbb{R}^{|\Scal||\Acal|}$ is a vector containing the Q-values for all state-action pairs, $\boldsymbol{R}\in \mathbb{R}^{|\Scal||\Acal|}$, and $\boldsymbol{P}^{\pi} \in \mathbb{R}^{|\Scal||\Acal|\times|\Scal||\Acal|}$ is the $(s,a)$-transition matrix induced by the MDP and policy $\pi$. Solving this equation gives us the Q-function in closed form:
\begin{equation}
    \boldsymbol{Q} = (\boldsymbol{I} - \gamma \boldsymbol{P}^{\pi})^{-1} \boldsymbol{R} \label{eqn:bellman-matrix}
\end{equation}
where $\boldsymbol{I} \in \mathbb{R}^{|\Scal||\Acal|\times|\Scal||\Acal|}$. 

To derive a necessary condition, we start by assuming that the Q-function is representable by the linear parameterization, i.e., there exists $\boldsymbol{W} \in \mathbb{R}^{(\sum_{d=1}^{D}|\Acal_d|) \times |\mathcal{S}|}$ such that $\vect^{-1}_{|\Acal|\times|\Scal|}(\boldsymbol{Q}) = \boldsymbol{\Psi} \boldsymbol{W}$. Here, $\vect^{-1}_{|\Acal|\times|\Scal|}$ is the inverse vectorization operator that reshapes the vector of all Q-values into a matrix of size $|\Acal|\times|\Scal|$, and $\boldsymbol{\Psi} \in \{0,1\}^{|\Acal| \times (\sum_{d=1}^{D}|\Acal_d|)}$ is defined in \cref{appx:subspace_proof}. Substituting \cref{eqn:bellman-matrix} into the premise gives us a necessary condition: if there exists $\boldsymbol{W}\in \mathbb{R}^{(\sum_{d=1}^{D}|\Acal_d|) \times |\mathcal{S}|}$ such that 
\[\vect^{-1}_{|\Acal|\times|\Scal|} \big((\boldsymbol{I} - \gamma \boldsymbol{P}^{\pi})^{-1} \boldsymbol{R} \big) = \boldsymbol{\Psi} \boldsymbol{W}\]

Unfortunately, unlike the sufficient conditions in \cref{thm:sufficient-abstract} (and \cref{thm:sufficient-factored}), this necessary condition is not as clean and likely not verifiable in most settings. The matrix inverse and $\vect^{-1}$ reshaping operation make it challenging to further manipulate the expression. This highlights the non-trivial nature of the problem.

\subsection{Variance Reduction in the Bandit Setting} \label{appx:rademacher}

\textbf{\textit{Background on Rademacher complexity.}}
Let $\Fcal$ be a family of functions mapping from $\mathbb{R}^{d}$ to $\mathbb{R}$. The empirical Rademacher complexity of $\Fcal$ for a sample $\mathcal{S} = \{\mathbf{x}_1, \dots, \mathbf{x}_m\}$ is defined by
\[ \widehat{\mathfrak{R}}_{\mathcal{S}}(\Fcal) = \underset{\boldsymbol{\sigma}}{\mathbb{E}}\left[ \sup_{f\in\Fcal} \frac{1}{m}\sum_{i=1}^{m} \sigma_i f(\mathbf{x}_i)\right], \]
where $\boldsymbol{\sigma} = [\sigma_1, \dots, \sigma_m]$ is a vector of i.i.d. Rademacher variables, i.e., independent uniform r.v.s taking values in $\{-1,+1\}$. 

For a matrix $\mathbf{M} \in \mathbb{R}^{m \times D}$, define the $(p,q)$-group norm as the $q$-norm of the $p$-norm of the columns of $\mathbf{M}$, that is $\|\mathbf{M}\|_{p,q} = \|[\|\mathbf{M}_1\|_p, \cdots, \|\mathbf{M}_D\|_p] \|_q$, where $\mathbf{M}_j$ is the $j$-th column of $\mathbf{M}$. 

In \citet{awasthi2020rademacher}, Theorem 2 stated that: let $\mathcal{F} = \{ f = \mathbf{w}^{\intercal} \mathbf{x} : \|\mathbf{w}\|_p \leq A \}$ be a family of linear functions defined over $\mathbb{R}^d$ with bounded weight in $\ell_2$-norm, then the empirical Rademacher complexity of $\Fcal$ for a sample $\mathcal{S} = \{\mathbf{x}_1, \dots, \mathbf{x}_m\}$ satisfies the following lower bound (where $\mathbf{X} = [\mathbf{x}_1 \dots \mathbf{x}_m]^{\intercal}$): 
\[ \widehat{\mathfrak{R}}_{\mathcal{S}}(\Fcal) \geq \frac{A}{\sqrt{2}m}\|\mathbf{X}\|_{2,2}. \]

\begin{proof}[Proof for \cref{thm:variance}]
For the sake of argument, we consider the one-timestep bandit setting; extension to the sequential setting can be similarly derived following \citet{chen2019infotheory,duan2021risk}. Let the true generative model be 
\(\boldsymbol{Q}^* = \boldsymbol{\Psi} \boldsymbol{r} + \boldsymbol{\psi}_{\scriptscriptstyle\textsf{Interact}} r_{\scriptscriptstyle\textsf{Interact}}\) (details in \cref{appx:OVB}). We formally show the reduction in the variance of the estimators, by comparing the lower bound of their respective empirical Rademacher complexities. A smaller Rademacher complexity translates into lower variance estimators. 

Suppose we obtain a sample of $m$ actions and apply the linear approximation. Our approach for factored action space corresponds to the matrix $\boldsymbol{X} \in \set{0,1}^{m \times (\sum_{d}|\Acal_d|)}$, obtained by stacking the corresponding rows of $\boldsymbol{\Psi}$ (recall \cref{def:action-matrix}). The complete, combinatorial action space corresponds to the matrix $\boldsymbol{X}' = [\boldsymbol{X}, \boldsymbol{x}_{\scriptscriptstyle\textsf{Interact}}] \in \set{0,1}^{m \times (1+\sum_{d}|\Acal_d|)}$ by adding the corresponding rows of $\boldsymbol{\psi}_{\scriptscriptstyle\textsf{Interact}}$. By definition, $\|\boldsymbol{X}\|_{p,q} < \|\boldsymbol{X}'\|_{p,q}$, since the former drops a column with non-zero norm that exists in the latter. 

Consider the following two function families, for the factored action space and the complete action space respectively:
\begin{align*}
    \mathcal{F}_{\mathrm{\scriptscriptstyle F}} & = \{ f = \mathbf{w}_{\mathrm{\scriptscriptstyle F}}^\intercal \mathbf{x}: \|\mathbf{w}_{\mathrm{\scriptscriptstyle F}}\|_2 \leq A\} \\
    \mathcal{F}_{\mathrm{\scriptscriptstyle C}} & = \{ f = \mathbf{w}_{\mathrm{\scriptscriptstyle C}}^\intercal \mathbf{x}': \|\mathbf{w}_{\mathrm{\scriptscriptstyle C}}\|_2 \leq A\},
\end{align*}
for some $A >0$.
A straightforward application of Theorem 2 of \citet{awasthi2020rademacher} shows that the lower bound on the Rademacher complexity of the of the factored action space is smaller than that of the complete action space, which completes our argument.
\end{proof}

\subsection{Standardization of Rewards for the Bandit Setting (\cref{thm:argmax-preserve})} \label{appx:bandit-standardize}

\begin{figure}[h]
    \centering
    \includegraphics[valign=c]{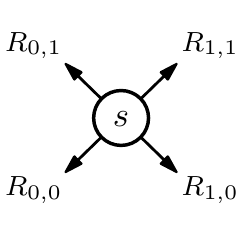} \qquad $\leadsto$ \quad \includegraphics[valign=c]{fig/bandit2d.pdf} 
    \caption{Standardization of rewards.}
    \label{fig:bandit-reward-standard}
\end{figure}

Suppose the rewards of the four arms are $[R_{0,0}, R_{0,1}, R_{1,0}, R_{1,1}]$. We can apply the following transformations to reduce any reward function to the form of $[0, \alpha, 1, 1+\alpha+\beta]$, and these transformations do not affect the least-squares solution: 
\begin{itemize}
    \item If $R_{0,0} = R_{1,0}$ and $R_{0,1} = R_{1,1}$, we can ignore x-axis sub-action as setting it to either $0$ ($\leftarrow$) or $1$ ($\rightarrow$) does not affect the reward. Similarly, if $R_{0,0} = R_{0,1}$ and $R_{1,0} = R_{1,1}$, we can ignore y-axis sub-action. In both cases, this reduces to a one-dimensional action space which we do not discuss further. 
    \item Now at least one of the following is false: $R_{0,0} = R_{1,0}$ or $R_{0,1} = R_{1,1}$. If $R_{0,0} \neq R_{1,0}$, skip this step. Otherwise, it must be that $R_{0,0} = R_{1,0}$ and $R_{0,1} \neq R_{1,1}$. Swap the role of down vs. up such that the new $R_{0,0} \neq R_{1,0}$. 
    \item If $R_{0,0} < R_{1,0}$, skip this step. Otherwise it must be that $R_{0,0} > R_{1,0}$. Swap the role of left vs. right so that $R_{0,0} < R_{1,0}$. 
    \item If $R_{0,0} \neq 0$, subtract $R_{0,0}$ from all rewards so that the new $R_{0,0} = 0$. 
    \item Now $R_{1,0} > R_{0,0} > 0$ must be positive. If $R_{1,0} \neq 1$, divide all rewards by $R_{1,0}$ so that the new $R_{1,0} = 1$. 
    \item Lastly, we should have $R_{0,0} = 0$ and $R_{1,0} = 1$. Set $\alpha = R_{0,1}$ and $\beta = R_{1,1} - R_{1,0} - R_{0,1}$. 
\end{itemize}

\subsection{Omitted-Variable Bias in the Bandit Setting (\cref{thm:argmax-preserve})} \label{appx:OVB}

Suppose the true generative model is
\[ 
Q^*(\avec) = 
\mathbbm{1}_{(a_{\mathrm{x}} = {\scriptscriptstyle\textsf{Left}})} r_{\scriptscriptstyle\textsf{Left}} + 
\mathbbm{1}_{(a_{\mathrm{x}} = {\scriptscriptstyle\textsf{Right}})} r_{\scriptscriptstyle\textsf{Right}} + 
\mathbbm{1}_{(a_{\mathrm{y}} = {\scriptscriptstyle\textsf{Down}})} r_{\scriptscriptstyle\textsf{Down}} + 
\mathbbm{1}_{(a_{\mathrm{y}} = {\scriptscriptstyle\textsf{Up}})} r_{\scriptscriptstyle\textsf{Up}} + 
\mathbbm{1}_{(\avec = {\scriptscriptstyle\textsf{Right,Up}})} r_{\scriptscriptstyle\textsf{Interact}} 
\]

In other words, 

\[
    \begin{bmatrix} Q^*({\scriptstyle\swarrow}) \\ Q^*({\scriptstyle\nwarrow}) \\ Q^*({\scriptstyle\searrow}) \\ Q^*({\scriptstyle\nearrow}) \end{bmatrix} 
    = \begin{bmatrix}
        1 & 0 & 1 & 0 \\
        1 & 0 & 0 & 1 \\
        0 & 1 & 1 & 0 \\
        0 & 1 & 0 & 1 \\
    \end{bmatrix} \begin{bmatrix}
    r_{\scriptscriptstyle\textsf{Left}} \\ r_{\scriptscriptstyle\textsf{Right}} \\ r_{\scriptscriptstyle\textsf{Down}} \\ r_{\scriptscriptstyle\textsf{Up}}
    \end{bmatrix} + \begin{bmatrix} 0 \\ 0 \\ 0 \\ 1 \end{bmatrix} r_{\scriptscriptstyle\textsf{Interact}}
    \quad \leadsto \quad \boldsymbol{Q}^* = \boldsymbol{\Psi} \boldsymbol{r} + \boldsymbol{\psi}_{\scriptscriptstyle\textsf{Interact}} r_{\scriptscriptstyle\textsf{Interact}}
\]

Here, $r_{\scriptscriptstyle\textsf{Left}},
r_{\scriptscriptstyle\textsf{Right}},
r_{\scriptscriptstyle\textsf{Down}},
r_{\scriptscriptstyle\textsf{Up}},
r_{\scriptscriptstyle\textsf{Interact}}$ are parameters of the generative model. Note that the matrix $[\boldsymbol{\Psi}, \boldsymbol{\psi}_{\scriptscriptstyle\textsf{Interact}}]$ has a column space of $\mathbb{R}^{4}$, i.e., this generative model captures every possible reward configuration of the four actions. 

Applying our proposed linear approximation translates to ``dropping'' the interaction parameter, $r_{\scriptscriptstyle\textsf{Interact}}$, and estimate the remaining four parameters. This leads to a form of omitted-variable bias, which can be computed as:
\begin{align*}
\boldsymbol{\Psi}^{+} \boldsymbol{\psi}_{\scriptscriptstyle\textsf{Interact}} r_{\scriptscriptstyle\textsf{Interact}} =&\begin{bmatrix}
    1 & 0 & 1 & 0 \\
    1 & 0 & 0 & 1 \\
    0 & 1 & 1 & 0 \\
    0 & 1 & 0 & 1 \\
\end{bmatrix}^{+} \begin{bmatrix} 0 \\ 0 \\ 0 \\ 1 \end{bmatrix} r_{\scriptscriptstyle\textsf{Interact}} \\
=& \begin{bmatrix}
    \phantom{-}3/8 & \phantom{-}3/8 & -1/8           & -1/8 \\
    -1/8           & -1/8           & \phantom{-}3/8 & \phantom{-}3/8 \\
    \phantom{-}3/8 & -1/8           & \phantom{-}3/8 & -1/8 \\
    -1/8           & \phantom{-}3/8 & -1/8           & \phantom{-}3/8 \\
\end{bmatrix} \begin{bmatrix} 0 \\ 0 \\ 0 \\ 1 \end{bmatrix} r_{\scriptscriptstyle\textsf{Interact}} =  \begin{bmatrix} -1/8 \\ \phantom{-}3/8 \\ -1/8 \\ \phantom{-}3/8 \end{bmatrix} r_{\scriptscriptstyle\textsf{Interact}}
\end{align*}

The biased estimate of the four parameters are:
\begin{align*}
\setlength\arraycolsep{1pt}
\def\arraystretch{1.2}
\hat{\boldsymbol{r}} = \boldsymbol{r} + \boldsymbol{\Psi}^{+} \boldsymbol{\psi}_{\scriptscriptstyle\textsf{Interact}} r_{\scriptscriptstyle\textsf{Interact}} \quad \leadsto \quad \begin{bmatrix}
    \hat{r}_{\scriptscriptstyle\textsf{Left}} \\ \hat{r}_{\scriptscriptstyle\textsf{Right}} \\ \hat{r}_{\scriptscriptstyle\textsf{Down}} \\ \hat{r}_{\scriptscriptstyle\textsf{Up}}
\end{bmatrix}
= 
\begin{bmatrix}
    r_{\scriptscriptstyle\textsf{Left}} &-& \frac{1}{8} r_{\scriptscriptstyle\textsf{Interact}} \\ r_{\scriptscriptstyle\textsf{Right}} &+& \frac{3}{8} r_{\scriptscriptstyle\textsf{Interact}} \\ r_{\scriptscriptstyle\textsf{Down}} &-& \frac{1}{8} r_{\scriptscriptstyle\textsf{Interact}} \\ r_{\scriptscriptstyle\textsf{Up}} &+& \frac{3}{8} r_{\scriptscriptstyle\textsf{Interact}}
\end{bmatrix}
\end{align*}

and the estimated Q-values are:
\begin{align*}
\hat{\boldsymbol{Q}} = \begin{bmatrix} \hat{Q}({\scriptstyle\swarrow}) \\ \hat{Q}({\scriptstyle\nwarrow}) \\ \hat{Q}({\scriptstyle\searrow}) \\ \hat{Q}({\scriptstyle\nearrow}) \end{bmatrix} 
&= \begin{bmatrix}
    1 & 0 & 1 & 0 \\
    1 & 0 & 0 & 1 \\
    0 & 1 & 1 & 0 \\
    0 & 1 & 0 & 1 \\
\end{bmatrix}
\setlength\arraycolsep{1pt}\def\arraystretch{1.2}
\begin{bmatrix}
r_{\scriptscriptstyle\textsf{Left}} &-& \frac{1}{8} r_{\scriptscriptstyle\textsf{Interact}} \\ r_{\scriptscriptstyle\textsf{Right}} &+& \frac{3}{8} r_{\scriptscriptstyle\textsf{Interact}} \\ r_{\scriptscriptstyle\textsf{Down}} &-& \frac{1}{8} r_{\scriptscriptstyle\textsf{Interact}} \\ r_{\scriptscriptstyle\textsf{Up}} &+& \frac{3}{8} r_{\scriptscriptstyle\textsf{Interact}}
\end{bmatrix} 
= \setlength\arraycolsep{1pt}\def\arraystretch{1.2}
\begin{bmatrix} 
r_{\scriptscriptstyle\textsf{Left}} + r_{\scriptscriptstyle\textsf{Down}} & - \frac{1}{4} r_{\scriptscriptstyle\textsf{Interact}} \\ r_{\scriptscriptstyle\textsf{Left}} + r_{\scriptscriptstyle\textsf{Up}} & + \frac{1}{4} r_{\scriptscriptstyle\textsf{Interact}} \\ r_{\scriptscriptstyle\textsf{Right}} + r_{\scriptscriptstyle\textsf{Down}} & +\frac{1}{4} r_{\scriptscriptstyle\textsf{Interact}} \\ r_{\scriptscriptstyle\textsf{Right}} + r_{\scriptscriptstyle\textsf{Up}} & +\frac{3}{4} r_{\scriptscriptstyle\textsf{Interact}} \end{bmatrix}
\end{align*}

For the bandit problem in \cref{fig:2d_bandit}a, substituting $r_{\scriptscriptstyle\textsf{Left}} + r_{\scriptscriptstyle\textsf{Down}} = 0$, $r_{\scriptscriptstyle\textsf{Left}} + r_{\scriptscriptstyle\textsf{Up}} = \alpha$, $r_{\scriptscriptstyle\textsf{Right}} + r_{\scriptscriptstyle\textsf{Down}} = 1$, and $r_{\scriptscriptstyle\textsf{Interact}} = \beta$ gives
\begin{align*}
\setlength\arraycolsep{1pt}
\def\arraystretch{1.2}
    \begin{bmatrix} \hat{Q}({\scriptstyle\swarrow}) \\ \hat{Q}({\scriptstyle\nwarrow}) \\ \hat{Q}({\scriptstyle\searrow}) \\ \hat{Q}({\scriptstyle\nearrow}) \end{bmatrix} = \begin{bmatrix} & - \frac{1}{4} \beta \\ \alpha & + \frac{1}{4} \beta \\ 1 & +\frac{1}{4} \beta \\ 1 + \alpha & +\frac{3}{4} \beta \end{bmatrix}
\end{align*}
which is the solution we presented in \cref{fig:2d_bandit}c.

\subsection{Accounting for Sub-action Interactions} \label{appx:interactions}
When the interaction effect is not negligible and can lead to suboptimal performance, one solution is to explicitly encode the residual interaction terms in the decomposed Q-function by letting $Q(s, \avec) = \sum_{d=1}^{D} q_d(s,a_d) + \mathfrak{R}(\avec)$. The exact parameterization of the residual term $\mathfrak{R}(\avec)$ is problem dependent: one may incorporate \citet{tavakoli2021learning} to systematically consider interactions of certain ``ranks'' (e.g., limiting it to only two-way or three-way interactions), and consider regularizing the magnitude of residual terms so we still benefit from the efficiency gains of the linear decomposition.

\section{More Illustrative Examples} \label{appx:examples}

\begin{figure}[h]
    \centering
    \begin{tabular}{lp{1.25cm}l}
    (a) & & (b) \\[-2ex]
    \includegraphics[scale=1, trim=0 15 10 0, clip, valign=c]{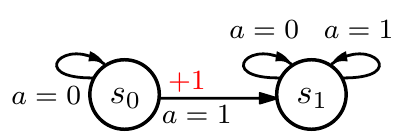} & &
    \includegraphics[scale=0.75, valign=c]{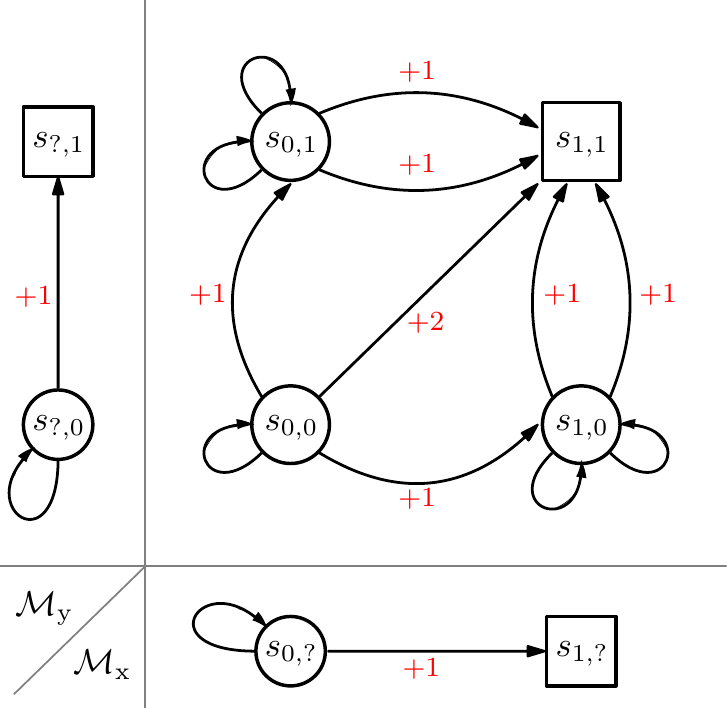}
    \end{tabular}
    \caption{(a) A one-dimensional chain MDP, with an initial state $s_0$ and an absorbing state $s_1$, and two actions $a=0$ (left) and $a=1$ (right). (b) A two-dimensional chain MDP shown together with the component chains $\Mcal_{\mathrm{x}}$ and $\Mcal_{\mathrm{y}}$. Rewards are denoted in \red{red}. Squares $\square$ indicate absorbing states whose outgoing transition arrows are omitted. For readability, in the diagram, the states and actions are laid out following a convention similar to the Cartesian coordinate system so that the bottom left state has index $(0,0)$, and right and up both increase the corresponding coordinate by 1. }
    \label{fig:chains}
\end{figure}

In this appendix, we discuss the building blocks of the examples used in the main paper and provide additional examples to support the theoretical properties presented in \cref{sec:theory}. 

\textbf{One-dimensional Chain.} First, consider the chain problem depicted in \cref{fig:chains}a. The agent always starts in the initial state $s_0$ and can take one of two possible actions: left ($a=0$), which leads the agent to stay at $s_0$, or right ($a=1$), which leads the agent to transition into $s_1$ and receive a reward of $+1$. After reaching the absorbing state $s_1$, both $a=0$ and $a=1$ lead the agent to stay at $s_1$ with zero reward. For $\gamma<1$, a (deterministic) optimal policy is $\pi^{*}(s_0) = 1$, and either action can be taken in $s_1$. Next, we use this MDP to construct a two-dimensional problem. 

\textbf{Two-dimensional Chain.} Following the construction used in \cref{def:factored-MDP}, we consider an MDP $\Mcal = \Mcal_{\mathrm{x}} \times \Mcal_{\mathrm{y}}$ consisting of two chains (the horizontal chain $\Mcal_{\mathrm{x}}$ and the vertical chain $\Mcal_{\mathrm{y}}$) running in parallel, as shown in \cref{fig:chains}b. Their corresponding state spaces are $\Scal_{\mathrm{x}} = \{s_{0,?}, s_{1,?}\}$ and $\Scal_{\mathrm{y}} = \{s_{?,0}, s_{?,1}\}$, which indicate the x- and y-coordinates respectively. There are 4 actions from each state, depicted by diagonal arrows $\{\swarrow, \nwarrow, \searrow, \nearrow\}$; each action $\avec = [a_{\mathrm{x}}, a_{\mathrm{y}}]$ effectively leads the agent to perform $a_{\mathrm{x}}$ in $\Mcal_{\mathrm{x}}$ and $a_{\mathrm{y}}$ in $\Mcal_{\mathrm{y}}$. For example, taking action $\nearrow=[\rightarrow,\uparrow]$ from state $s_{0,0}$ leads the agent to transition into state $s_{1,1}$ and receive a reward of $+2$ (the sum of $+1$ from $\Mcal_{\mathrm{x}}$ and $+1$ from $\Mcal_{\mathrm{y}}$). For $\gamma<1$, an optimal policy for this MDP is to always move up and right, $\pi^{*}(\cdot) = \nearrow = [\rightarrow,\uparrow]$, regardless of which state the agent is in.

\textbf{Satisfying the Sufficient Conditions.} 
Let $\phi_{\mathrm{x}}: \Scal \to \Scal_{\mathrm{x}}$ and $\phi_{\mathrm{y}}: \Scal \to \Scal_{\mathrm{y}}$ be the abstractions. By construction, the transition and reward functions of this MDP satisfy \cref{eqn:abstract-factored-policy,eqn:abstract-factored-reward}. To apply \cref{thm:sufficient-abstract}, the policy must satisfy \cref{eqn:abstract-factored-policy}. In \cref{fig:2d-chain-policies-group1}, we show three such policies (other policies in this category are omitted due to symmetry and transitions that have the same outcome), together with the true Q-functions (with $\gamma=0.9$) and their decompositions in the form of \cref{eqn:Q-decomposition}. 

\textbf{Violating the Sufficient Conditions.}
\begin{itemize}
    \item \textbf{Policy violates \cref{eqn:abstract-factored-policy} - Nonzero bias.} For this setting, we hold the MDP (transitions and rewards) unchanged. In \cref{fig:2d-chain-policies-group2}, we show seven policies that do not satisfy \cref{eqn:abstract-factored-policy}, together with the resultant Q-function and the biased linear approximation with the non-zero approximation error.
    \item \textbf{Transition violates \cref{eqn:abstract-factored-transition} - Nonzero Bias.} \cref{fig:2d-chain-transition} shows an example where one transition has been modified.
    \item \textbf{Reward violates \cref{eqn:abstract-factored-transition} - Nonzero Bias.} \cref{fig:2d-chain-reward} shows an example where one reward has been modified.
    \item \textbf{Transition violates \cref{eqn:abstract-factored-transition}, or policy violates \cref{eqn:abstract-factored-policy} - Zero Bias.} If $\gamma=0$, then the Q-function is simply the immediate reward, and any conditions on the transition or policy can be forgone. 
    \item \textbf{Reward violates \cref{eqn:abstract-factored-reward} - Zero Bias.} It is possible to construct reward functions adversarially such that $r$ itself does not satisfy the condition, and yet $Q$ can be linearly decomposed. See \cref{fig:2d-chain-reward-adversarial} for an example. 
\end{itemize}

\begin{figure}
    \centering
    \setlength{\fboxsep}{1pt}
\centerline{\scalebox{0.8}{
    \begin{tabular}{ccccc}
    \toprule
    Policy \(\pi\) & MDP diagram \phantom{+} &
    \multicolumn{1}{c@{\hspace*{\tabcolsep}\makebox[0pt]{$=$}}}{\(Q^{\pi}\)} & 
    \multicolumn{1}{c@{\hspace*{\tabcolsep}\makebox[0pt]{$+$}}}{\(Q_{\mathrm{x}}\)} & 
    \(Q_{\mathrm{y}}\) \\
    \midrule
    
    Optimal policy $\pi^*$ & & &  \\
    \(
    \begin{matrix} s_{0,0} \\ s_{0,1} \\ s_{1,0} \\ s_{1,1} \end{matrix}
    \begin{bmatrix} \nearrow \\ \nearrow \\ \nearrow \\ \nearrow \end{bmatrix} =
    \begin{bmatrix} \rightarrow,\uparrow \\ \rightarrow,\uparrow \\ \rightarrow,\uparrow \\ \rightarrow,\uparrow \end{bmatrix}
    \) 
    & \includegraphics[scale=0.75,valign=c]{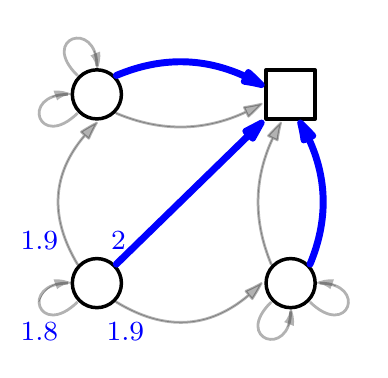} 
    & 
\(\begin{bNiceMatrix}[first-row,first-col]
& \swarrow & \nwarrow & \searrow & \nearrow \\
s_{0,0} & 1.8 & 1.9 & 1.9 & 2 \\
s_{0,1} & 0.9 & 0.9 & 1 & 1 \\
s_{1,0} & 0.9 & 1 & 0.9 & 1 \\
s_{1,1} & 0 & 0 & 0 & 0
\end{bNiceMatrix}\)
    & 
\(\begin{bNiceMatrix}[first-row,first-col]
& \leftarrow & \leftarrow & \rightarrow & \rightarrow \\
s_{0,?} & 0.9 & 0.9 & 1 & 1 \\
s_{0,?} & 0.9 & 0.9 & 1 & 1 \\
s_{1,?} & 0 & 0 & 0 & 0 \\
s_{1,?} & 0 & 0 & 0 & 0
\end{bNiceMatrix}\)
    &
\(\begin{bNiceMatrix}[first-row,first-col]
& \downarrow & \uparrow & \downarrow & \uparrow \\
s_{?,0} & 0.9 & 1 & 0.9 & 1 \\
s_{?,1} & 0 & 0 & 0 & 0 \\
s_{?,0} & 0.9 & 1 & 0.9 & 1 \\
s_{?,1} & 0 & 0 & 0 & 0
\end{bNiceMatrix}\)
    \\
    
    \midrule
    
    A non-optimal policy \\
    \(
    \begin{matrix} s_{0,0} \\ s_{0,1} \\ s_{1,0} \\ s_{1,1} \end{matrix}
    \begin{bmatrix} \nwarrow \\ \nwarrow \\ \nearrow \\ \nearrow \end{bmatrix} =
    \begin{bmatrix} \leftarrow,\uparrow \\ \leftarrow,\uparrow \\ \rightarrow,\uparrow \\ \rightarrow,\uparrow \end{bmatrix}
    \) 
    & \includegraphics[scale=0.75,valign=c]{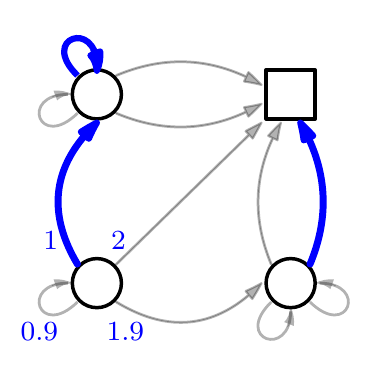} 
    & 
\(\begin{bNiceMatrix}[first-row,first-col]
& \swarrow & \nwarrow & \searrow & \nearrow \\
s_{0,0} & 0.9 & 1 & 1.9 & 2 \\
s_{0,1} & 0 & 0 & 1 & 1 \\
s_{1,0} & 0.9 & 1 & 0.9 & 1 \\
s_{1,1} & 0 & 0 & 0 & 0
\end{bNiceMatrix}\)
    & 
\(\begin{bNiceMatrix}[first-row,first-col]
& \leftarrow & \leftarrow & \rightarrow & \rightarrow \\
s_{0,?} & 0 & 0 & 1 & 1 \\
s_{0,?} & 0 & 0 & 1 & 1 \\
s_{1,?} & 0 & 0 & 0 & 0 \\
s_{1,?} & 0 & 0 & 0 & 0
\end{bNiceMatrix}\)
    &
\(\begin{bNiceMatrix}[first-row,first-col]
& \downarrow & \uparrow & \downarrow & \uparrow \\
s_{?,0} & 0.9 & 1 & 0.9 & 1 \\
s_{?,1} & 0 & 0 & 0 & 0 \\
s_{?,0} & 0.9 & 1 & 0.9 & 1 \\
s_{?,1} & 0 & 0 & 0 & 0
\end{bNiceMatrix}\)
    \\
    
    \midrule
    Another non-optimal policy \\
    \(
    \begin{matrix} s_{0,0} \\ s_{0,1} \\ s_{1,0} \\ s_{1,1} \end{matrix}
    \begin{bmatrix} \swarrow \\ \nwarrow \\ \searrow \\ \nearrow \end{bmatrix} =
    \begin{bmatrix} \leftarrow,\downarrow \\ \leftarrow,\uparrow \\ \rightarrow,\downarrow \\ \rightarrow,\uparrow \end{bmatrix}
    \) 
    & \includegraphics[scale=0.75,valign=c]{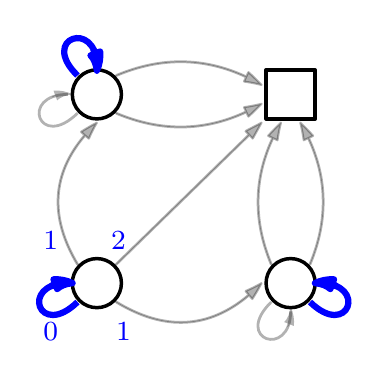} 
    &
\(\begin{bNiceMatrix}[first-row,first-col]
& \swarrow & \nwarrow & \searrow & \nearrow \\
s_{0,0} & 0 & 1 & 1 & 2 \\
s_{0,1} & 0 & 0 & 1 & 1 \\
s_{1,0} & 0 & 1 & 0 & 1 \\
s_{1,1} & 0 & 0 & 0 & 0
\end{bNiceMatrix}\)
    & 
\(\begin{bNiceMatrix}[first-row,first-col]
& \leftarrow & \leftarrow & \rightarrow & \rightarrow \\
s_{0,?} & 0 & 0 & 1 & 1 \\
s_{0,?} & 0 & 0 & 1 & 1 \\
s_{1,?} & 0 & 0 & 0 & 0 \\
s_{1,?} & 0 & 0 & 0 & 0
\end{bNiceMatrix}\)
    &
\(\begin{bNiceMatrix}[first-row,first-col]
& \downarrow & \uparrow & \downarrow & \uparrow \\
s_{?,0} & 0 & 1 & 0 & 1 \\
s_{?,1} & 0 & 0 & 0 & 0 \\
s_{?,0} & 0 & 1 & 0 & 1 \\
s_{?,1} & 0 & 0 & 0 & 0
\end{bNiceMatrix}\)
    \\
    \bottomrule
    \end{tabular}
}}
    \caption{Example MDPs and policies where \cref{thm:sufficient-factored} applies, for the optimal policy and two particular non-optimal policies. $\gamma=0.9$. We show the linear decomposition of the Q-function into $Q_{\mathrm{x}}$ and $Q_{\mathrm{y}}$. $Q_{\mathrm{x}}$ only depends on the x-coordinate of state and the sub-action that moves $\leftarrow$ or $\rightarrow$; $Q_{\mathrm{y}}$ only depends on the y-coordinate of state and the sub-action that moves $\downarrow$ or $\uparrow$. }
    \label{fig:2d-chain-policies-group1}
\end{figure}

\begin{figure}
    \centering
    \setlength{\fboxsep}{1pt}
\centerline{\scalebox{0.725}{
\begin{tabular}{cc}
    \begin{tabular}{cccc}
    \toprule
    \(\pi(\Scal)\) & MDP diagram \phantom{+} & \(Q^{\pi}(s_{0,0}, \Acal)\) & \(\hat{Q}(s_{0,0}, \Acal)\) \\
    
    \midrule
    
    \(
    \begin{matrix} s_{0,0} \\ s_{0,1} \\ s_{1,0} \\ s_{1,1} \end{matrix}
    \begin{bmatrix} \nwarrow \\ \nearrow \\ \nearrow \\ \nearrow \end{bmatrix} =
    \begin{bmatrix} \hspace*{-\fboxsep}{\colorbox{yellow!75}{$\leftarrow$}}\hspace*{-\fboxsep},\uparrow \\ \hspace*{-\fboxsep}{\colorbox{yellow!75}{$\rightarrow$}}\hspace*{-\fboxsep},\uparrow \\ \rightarrow,\uparrow \\ \rightarrow,\uparrow \end{bmatrix}
    \)
    & \includegraphics[scale=0.75,valign=c]{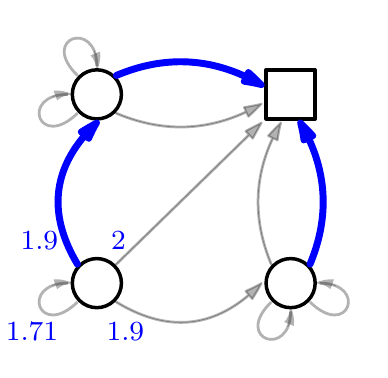} 
    & 
    \(
    \begin{matrix} \swarrow \\ \nwarrow \\ \searrow \\ \nearrow \end{matrix}
    \begin{bmatrix} 1.71 \\ 1.9 \\ 1.9 \\ 2 \end{bmatrix}
    \)
    & 
    \(
    \begin{matrix} \swarrow \\ \nwarrow \\ \searrow \\ \nearrow \end{matrix}
    \begin{bmatrix} 1.7325 \\ 1.8775 \\ 1.8775 \\ 2.0225 \end{bmatrix}
    \) \\
    
    \midrule
    
    \(
    \begin{matrix} s_{0,0} \\ s_{0,1} \\ s_{1,0} \\ s_{1,1} \end{matrix}
    \begin{bmatrix} \swarrow \\ \nearrow \\ \nearrow \\ \nearrow \end{bmatrix} =
    \begin{bmatrix} \hspace*{-\fboxsep}{\colorbox{yellow!75}{$\leftarrow$}}\hspace*{-\fboxsep},\hspace*{-\fboxsep}{\colorbox{pink!75}{$\downarrow$}}\hspace*{-\fboxsep} \\ \hspace*{-\fboxsep}{\colorbox{yellow!75}{$\rightarrow$}}\hspace*{-\fboxsep},\uparrow \\ \rightarrow,\hspace*{-\fboxsep}{\colorbox{pink!75}{$\uparrow$}}\hspace*{-\fboxsep} \\ \rightarrow,\uparrow \end{bmatrix}
    \)
    & \includegraphics[scale=0.75,valign=c]{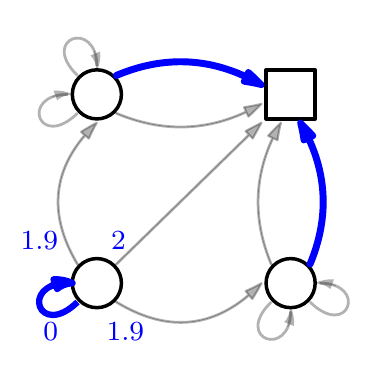}
    & 
    \(
    \begin{matrix} \swarrow \\ \nwarrow \\ \searrow \\ \nearrow \end{matrix}
    \begin{bmatrix} 0 \\ 1.9 \\ 1.9 \\ 2 \end{bmatrix}
    \)
    & 
    \(
    \begin{matrix} \swarrow \\ \nwarrow \\ \searrow \\ \nearrow \end{matrix}
    \begin{bmatrix} 0.45 \\ 1.45 \\ 1.45 \\ 2.45 \end{bmatrix}
    \) \\
    
    \midrule
    
    \(
    \begin{matrix} s_{0,0} \\ s_{0,1} \\ s_{1,0} \\ s_{1,1} \end{matrix}
    \begin{bmatrix} \searrow \\ \swarrow \\ \searrow \\ \searrow \end{bmatrix} =
    \begin{bmatrix} \hspace*{-\fboxsep}{\colorbox{yellow!75}{$\rightarrow$}}\hspace*{-\fboxsep},\downarrow \\ \hspace*{-\fboxsep}{\colorbox{yellow!75}{$\leftarrow$}}\hspace*{-\fboxsep},\downarrow \\ \rightarrow,\downarrow \\ \rightarrow,\downarrow \end{bmatrix}
    \)
    & \includegraphics[scale=0.75,valign=c]{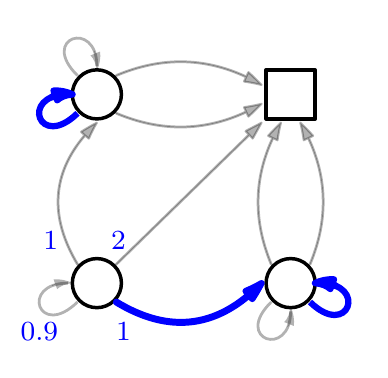}
    & 
    \(
    \begin{matrix} \swarrow \\ \nwarrow \\ \searrow \\ \nearrow \end{matrix}
    \begin{bmatrix} 0.9 \\ 1 \\ 1 \\ 2 \end{bmatrix}
    \)
    & 
    \(
    \begin{matrix} \swarrow \\ \nwarrow \\ \searrow \\ \nearrow \end{matrix}
    \begin{bmatrix} 0.675 \\ 1.225 \\ 1.225 \\ 1.775 \end{bmatrix}
    \) 
    \\
    
    \midrule
    
    \(
    \begin{matrix} s_{0,0} \\ s_{0,1} \\ s_{1,0} \\ s_{1,1} \end{matrix}
    \begin{bmatrix} \searrow \\ \swarrow \\ \searrow \\ \searrow \end{bmatrix} =
    \begin{bmatrix} \hspace*{-\fboxsep}{\colorbox{yellow!75}{$\rightarrow$}}\hspace*{-\fboxsep},\downarrow \\ \hspace*{-\fboxsep}{\colorbox{yellow!75}{$\leftarrow$}}\hspace*{-\fboxsep},\downarrow \\ \rightarrow,\downarrow \\ \rightarrow,\downarrow \end{bmatrix}
    \)
    & \includegraphics[scale=0.75,valign=c]{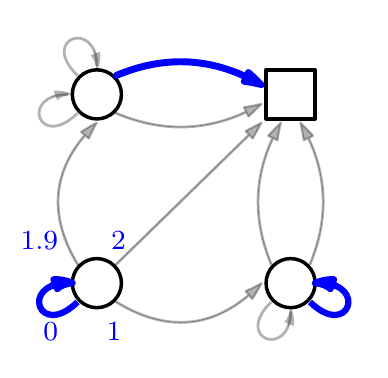}
    & 
    \(
    \begin{matrix} \swarrow \\ \nwarrow \\ \searrow \\ \nearrow \end{matrix}
    \begin{bmatrix} 0 \\ 1.9 \\ 1 \\ 2 \end{bmatrix}
    \)
    & 
    \(
    \begin{matrix} \swarrow \\ \nwarrow \\ \searrow \\ \nearrow \end{matrix}
    \begin{bmatrix} 0.225 \\ 1.675 \\ 0.775 \\ 2.225 \end{bmatrix}
    \) 
    \\
    
    \bottomrule
\end{tabular}
    &
\begin{tabular}{cccc}
    \toprule
    \(\pi(\Scal)\) & MDP diagram \phantom{+} & \(Q^{\pi}(s_{0,0}, \Acal)\) & \(\hat{Q}(s_{0,0}, \Acal)\) \\
    
    \midrule
    
    \(
    \begin{matrix} s_{0,0} \\ s_{0,1} \\ s_{1,0} \\ s_{1,1} \end{matrix}
    \begin{bmatrix} \nearrow \\ \nwarrow \\ \nearrow \\ \nearrow \end{bmatrix} =
    \begin{bmatrix} \hspace*{-\fboxsep}{\colorbox{yellow!75}{$\rightarrow$}}\hspace*{-\fboxsep},\uparrow \\ \hspace*{-\fboxsep}{\colorbox{yellow!75}{$\leftarrow$}}\hspace*{-\fboxsep},\uparrow \\ \rightarrow,\uparrow \\ \rightarrow,\uparrow \end{bmatrix}
    \)
    & \includegraphics[scale=0.75,valign=c]{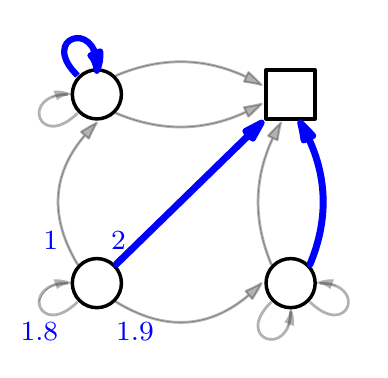} 
    & 
    \(
    \begin{matrix} \swarrow \\ \nwarrow \\ \searrow \\ \nearrow \end{matrix}
    \begin{bmatrix} 1.8 \\ 1 \\ 1.9 \\ 2 \end{bmatrix}
    \)
    & 
    \(
    \begin{matrix} \swarrow \\ \nwarrow \\ \searrow \\ \nearrow \end{matrix}
    \begin{bmatrix} 1.575 \\ 1.225 \\ 2.125 \\ 1.775 \end{bmatrix}
    \) \\
    
    \midrule
    
    \(
    \begin{matrix} s_{0,0} \\ s_{0,1} \\ s_{1,0} \\ s_{1,1} \end{matrix}
    \begin{bmatrix} \searrow \\ \nwarrow \\ \nearrow \\ \nearrow \end{bmatrix} =
    \begin{bmatrix} \hspace*{-\fboxsep}{\colorbox{yellow!75}{$\rightarrow$}}\hspace*{-\fboxsep},\hspace*{-\fboxsep}{\colorbox{pink!75}{$\downarrow$}}\hspace*{-\fboxsep} \\ \hspace*{-\fboxsep}{\colorbox{yellow!75}{$\leftarrow$}}\hspace*{-\fboxsep},\uparrow \\ \rightarrow,\hspace*{-\fboxsep}{\colorbox{pink!75}{$\uparrow$}}\hspace*{-\fboxsep} \\ \rightarrow,\uparrow \end{bmatrix}
    \)
    & \includegraphics[scale=0.75,valign=c]{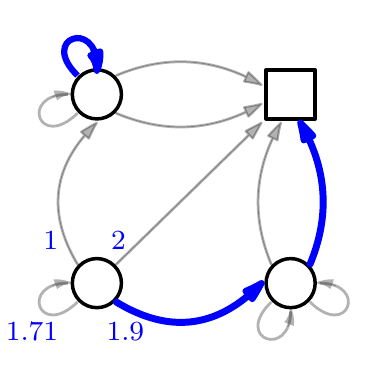}
    & 
    \(
    \begin{matrix} \swarrow \\ \nwarrow \\ \searrow \\ \nearrow \end{matrix}
    \begin{bmatrix} 1.71 \\ 1 \\ 1.9 \\ 2 \end{bmatrix}
    \)
    & 
    \(
    \begin{matrix} \swarrow \\ \nwarrow \\ \searrow \\ \nearrow \end{matrix}
    \begin{bmatrix} 1.5075 \\ 1.2025 \\ 2.1025 \\ 1.7975 \end{bmatrix}
    \) \\
    
    \midrule
    
    \(
    \begin{matrix} s_{0,0} \\ s_{0,1} \\ s_{1,0} \\ s_{1,1} \end{matrix}
    \begin{bmatrix} \nearrow \\ \nwarrow \\ \searrow \\ \searrow \end{bmatrix} =
    \begin{bmatrix} \hspace*{-\fboxsep}{\colorbox{yellow!75}{$\rightarrow$}}\hspace*{-\fboxsep},\hspace*{-\fboxsep}{\colorbox{pink!75}{$\uparrow$}}\hspace*{-\fboxsep} \\ \hspace*{-\fboxsep}{\colorbox{yellow!75}{$\leftarrow$}}\hspace*{-\fboxsep},\downarrow \\ \rightarrow,\hspace*{-\fboxsep}{\colorbox{pink!75}{$\downarrow$}}\hspace*{-\fboxsep} \\ \rightarrow,\downarrow \end{bmatrix}
    \)
    & \includegraphics[scale=0.75,valign=c]{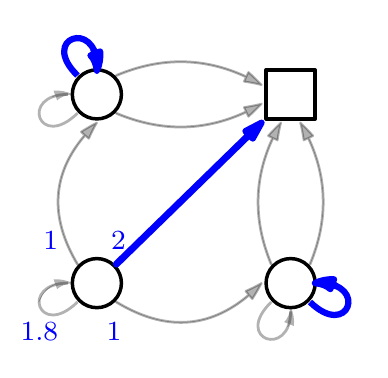}
    & 
    \(
    \begin{matrix} \swarrow \\ \nwarrow \\ \searrow \\ \nearrow \end{matrix}
    \begin{bmatrix} 1.8 \\ 1 \\ 1 \\ 2 \end{bmatrix}
    \)
    & 
    \(
    \begin{matrix} \swarrow \\ \nwarrow \\ \searrow \\ \nearrow \end{matrix}
    \begin{bmatrix} 1.35 \\ 1.45 \\ 1.45 \\ 1.55 \end{bmatrix}
    \) 
    \\
    \midrule
    \vspace{7em} \\
    
    \bottomrule
    \end{tabular}
\end{tabular}
    
}}
    \caption{Example MDPs and policies where \cref{thm:sufficient-factored} does not apply because the policy violates \cref{eqn:abstract-factored-policy} (violations are highlighted). $\gamma=0.9$. For example, in the first case, the policy does not take the same sub-action from $s_{0,0}$ and $s_{0,1}$ with respect to the horizontal chain $\Mcal_{\mathrm{x}}$. Applying the linear approximation produces biased estimates $\hat{Q}$ of the true Q-function, $Q^{\pi}$. }
    \label{fig:2d-chain-policies-group2}
\end{figure}

\begin{figure}[h]
    \centering
    \setlength{\fboxsep}{1pt}
\centerline{\scalebox{0.8}{
    \begin{tabular}{cccc}
    \toprule
    \(\pi(\Scal)\) & MDP diagram \phantom{+} & \(Q^{\pi}(s_{0,0}, \Acal)\) & \(\hat{Q}(s_{0,0}, \Acal)\) \\
    \midrule
    
    \(
    \begin{matrix} s_{0,0} \\ s_{0,1} \\ s_{1,0} \\ s_{1,1} \end{matrix}
    \begin{bmatrix} \nearrow \\ \nearrow \\ \nearrow \\ \nearrow \end{bmatrix} =
    \begin{bmatrix} \rightarrow,\uparrow \\ \rightarrow,\uparrow \\ \rightarrow,\uparrow \\ \rightarrow,\uparrow \end{bmatrix}
    \)  
    & \includegraphics[scale=0.75,valign=c]{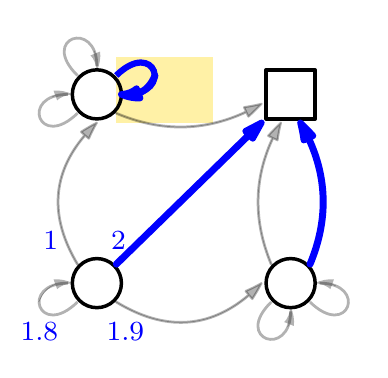} 
    & \(\begin{bmatrix} 1.8 \\ 1.9 \\ 1 \\ 2 \end{bmatrix}\)
    & \(\begin{bmatrix} 1.575 \\ 2.125 \\ 1.225 \\ 1.775 \end{bmatrix}\) \\
    
    \bottomrule
    \end{tabular}
}}
    \caption{Example MDPs and policies where \cref{thm:sufficient-abstract} does not apply because the transition function violates \cref{eqn:abstract-factored-transition}. $\gamma=0.9$. In this example, the highlighted transition corresponding to the action $\nearrow = [\rightarrow, \uparrow]$ from $s_{0,1}$ does not move right ($\rightarrow$ under $\Mcal_{\mathrm{x}}$) to $s_{1,1}$ and instead moves back to state $s_{0,1}$. Applying the linear approximation produces biased estimates $\hat{Q}$ of the true Q-function, $Q^{\pi}$. }
    \label{fig:2d-chain-transition}
\end{figure}

\begin{figure}[h]
    \centering
    \setlength{\fboxsep}{1pt}
\centerline{\scalebox{0.8}{
    \begin{tabular}{cccc}
    \toprule
    Reward function & Q-function & \(Q^{\pi}(s_{0,0}, \Acal)\) & \(\hat{Q}(s_{0,0}, \Acal)\) \\
    \midrule
    
    \includegraphics[scale=0.75,valign=c]{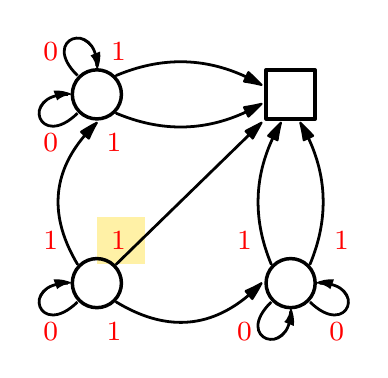} 
    & \includegraphics[scale=0.75,valign=c]{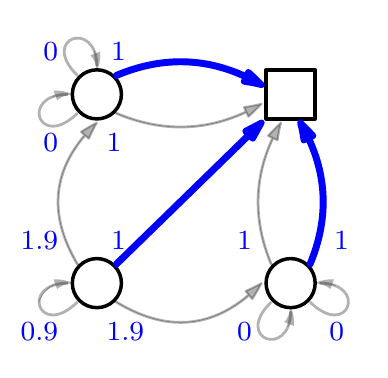}
    & \(\begin{bmatrix} 0.9 \\ 1.9 \\ 1.9 \\ 1 \end{bmatrix}\)
    & \(\begin{bmatrix} 1.375 \\ 1.425 \\ 1.425 \\ 1.475 \end{bmatrix}\) \\
    \bottomrule
    \end{tabular}
}}
    \caption{Example MDPs and policies where \cref{thm:sufficient-abstract} does not apply because the reward function violates \cref{eqn:abstract-factored-reward}. $\gamma=0.9$. In this example, the reward function of the bottom left state $s_{0,0}$ does not satisfy the condition because the reward of $\nearrow$ is $1 \neq 2=1+1$. Applying the linear approximation produces biased estimates $\hat{Q}$ of the true Q-function, $Q^{\pi}$. }
    \label{fig:2d-chain-reward}
\end{figure}

\begin{figure}[h]
    \centering
    \setlength{\fboxsep}{1pt}
\centerline{\scalebox{0.8}{
    \begin{tabular}{ccccc}
    \toprule
    Reward function & Q-function & 
    \multicolumn{1}{c@{\hspace*{\tabcolsep}\makebox[0pt]{$=$}}}{\(Q^{\pi}\)} & 
    \multicolumn{1}{c@{\hspace*{\tabcolsep}\makebox[0pt]{$+$}}}{\(Q_{\mathrm{x}}\)} & 
    \(Q_{\mathrm{y}}\) \\
    \midrule
    
    \includegraphics[scale=0.75,valign=c]{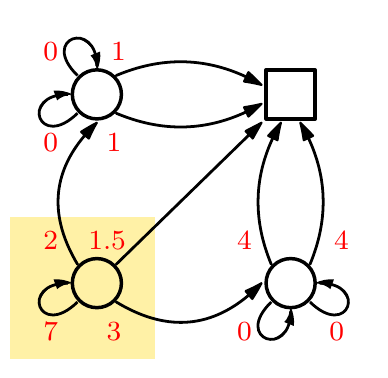} 
    & \includegraphics[scale=0.75,valign=c]{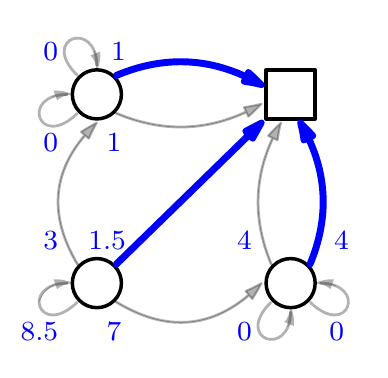}  
    &
\(\begin{bNiceMatrix}[first-row,first-col]
& \swarrow & \nwarrow & \searrow & \nearrow \\
s_{0,0} & 8.5 & 3 & 7 & 1.5 \\
s_{0,1} & 0 & 0 & 1 & 1 \\
s_{1,0} & 0 & 4 & 0 & 4 \\
s_{1,1} & 0 & 0 & 0 & 0
\end{bNiceMatrix}\)
    & 
\(\begin{bNiceMatrix}[first-row,first-col]
& \leftarrow & \leftarrow & \rightarrow & \rightarrow \\
s_{0,?} & 1.5 & 1.5 & 0 & 0 \\
s_{0,?} & 0 & 0 & 1 & 1 \\
s_{1,?} & 0 & 0 & 0 & 0 \\
s_{1,?} & 0 & 0 & 0 & 0
\end{bNiceMatrix}\)
    &
\(\begin{bNiceMatrix}[first-row,first-col]
& \downarrow & \uparrow & \downarrow & \uparrow \\
s_{?,0} & 7 & 1.5 & 7 & 1.5 \\
s_{?,1} & 0 & 0 & 0 & 0 \\
s_{?,0} & 0 & 4 & 0 & 4 \\
s_{?,1} & 0 & 0 & 0 & 0
\end{bNiceMatrix}\)
    \\
    \bottomrule
    \end{tabular}
}}
    \caption{Example MDPs and policies where \cref{thm:sufficient-abstract} does not apply because the reward function violates \cref{eqn:abstract-factored-reward}. $\gamma=1$. In this example, the reward function of the bottom left state $s_{0,0}$ does not satisfy the condition because $7+1.5 \neq 2+3$. However, there exists a linear decomposition of the true Q-function, $Q^{\pi}$, for a particular policy denoted by bold blue arrows. }
    \label{fig:2d-chain-reward-adversarial}
\end{figure}

\clearpage
\section{Experiments} \label{appx:experiments}

\subsection{Sepsis Simulator - Implementation Details} \label{appx:sepsisSim-impl}
When generating the datasets, we follow the default initial state distribution specified in the original implementation. 

By default, we used neural networks consisting of one hidden layer with 1,000 neurons and ReLU activation to allow for function approximators with sufficient expressivity. We trained these networks using the Adam optimizer (default settings) \citep{adam} with a batch size of 64 for a maximum of 100 epochs, applying early stopping on $10\%$ ``validation data'' (specific to each supervised task) with a patience of 10 epochs. We minimized the mean squared error (MSE) for regression tasks (each iteration of FQI). For FQI, we also added value clipping (to be within the range of possible returns $[-1,1]$) when computing bootstrapping targets to ensure a bounded function class and encourage better convergence behavior \citep{mnih2015atari}. 

\subsection{MIMIC Sepsis - Implementation Details} \label{appx:mimic-impl}
The RNN AIS encoder was trained to predict the mean of a unit-variance multivariate Gaussian that outputs the observation at subsequent timesteps, conditioned on the subsequent actions, following the idea in \citet{subramanian2019approximate}. We performed a grid search over the hyperparameters (\cref{tab:mimic_HP}) for training the RNN, selecting the model that achieved the smallest validation loss. Using the best encoder model, we then trained the offline RL policy using BCQ (and factored BCQ), considering validation performance of all checkpoints (saved every 100 iterations, for a maximum of 10,000 iterations) and all combinations of the BCQ hyperparameters (\cref{tab:mimic_HP}). 

\begin{table}[h]
    \centering
    \caption{Hyperparameter values used for training the RNN approximate information state as well as BCQ for offline RL. Discrete BCQ for both the baseline and factored implementation are identical except for the final layer of the Q-networks.}
\scalebox{0.9}{
    \begin{tabular}{lc}
    \toprule
    \textbf{Hyperparameter} & \textbf{Searched Settings} \\
    \midrule
    RNN: \\
    - Embedding dimension, $d_S$ & $\{8, 16, 32, 64, 128\}$ \\
    - Learning rate & $\{$ 1e-5, 5e-4, 1e-4, 5e-3, 1e-3 $\}$ \\
    \midrule
    BCQ (with 5 random restarts): \\
    - Threshold, $\tau$ & $\{0, 0.01, 0.05, 0.1, 0.3, 0.5, 0.75, 0.999\}$ \\
    - Learning rate & 3e-4 \\
    - Weight decay & 1e-3 \\
    - Hidden layer size & 256 \\
    \bottomrule
    \end{tabular}}
    \label{tab:mimic_HP}
\end{table}

\subsection{MIMIC Sepsis results} \label{appx:mimic-results}
\begin{figure}[h]
    \centering
    \begin{tabular}{cc}
        \includegraphics[scale=0.4]{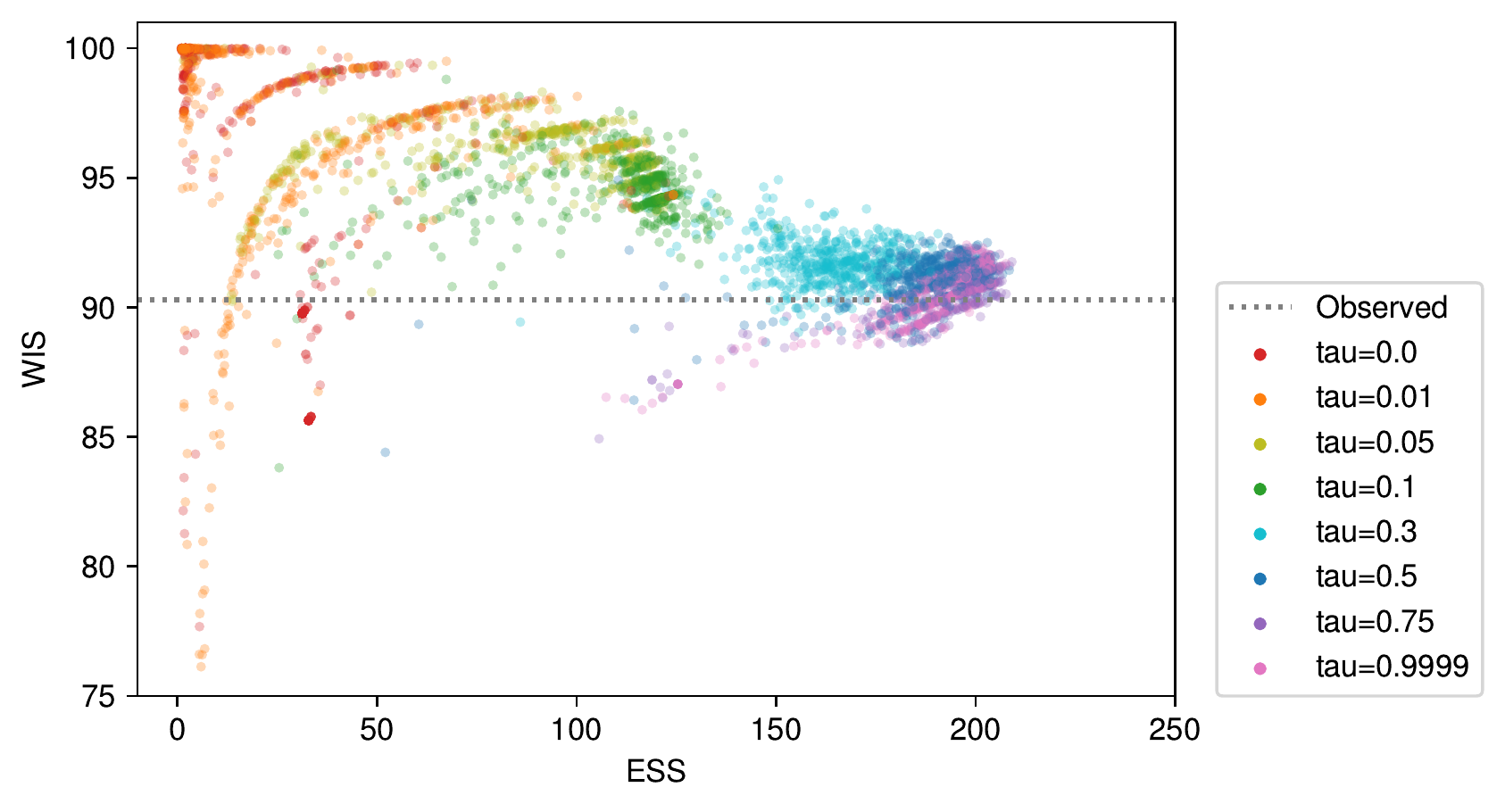} & \includegraphics[scale=0.4]{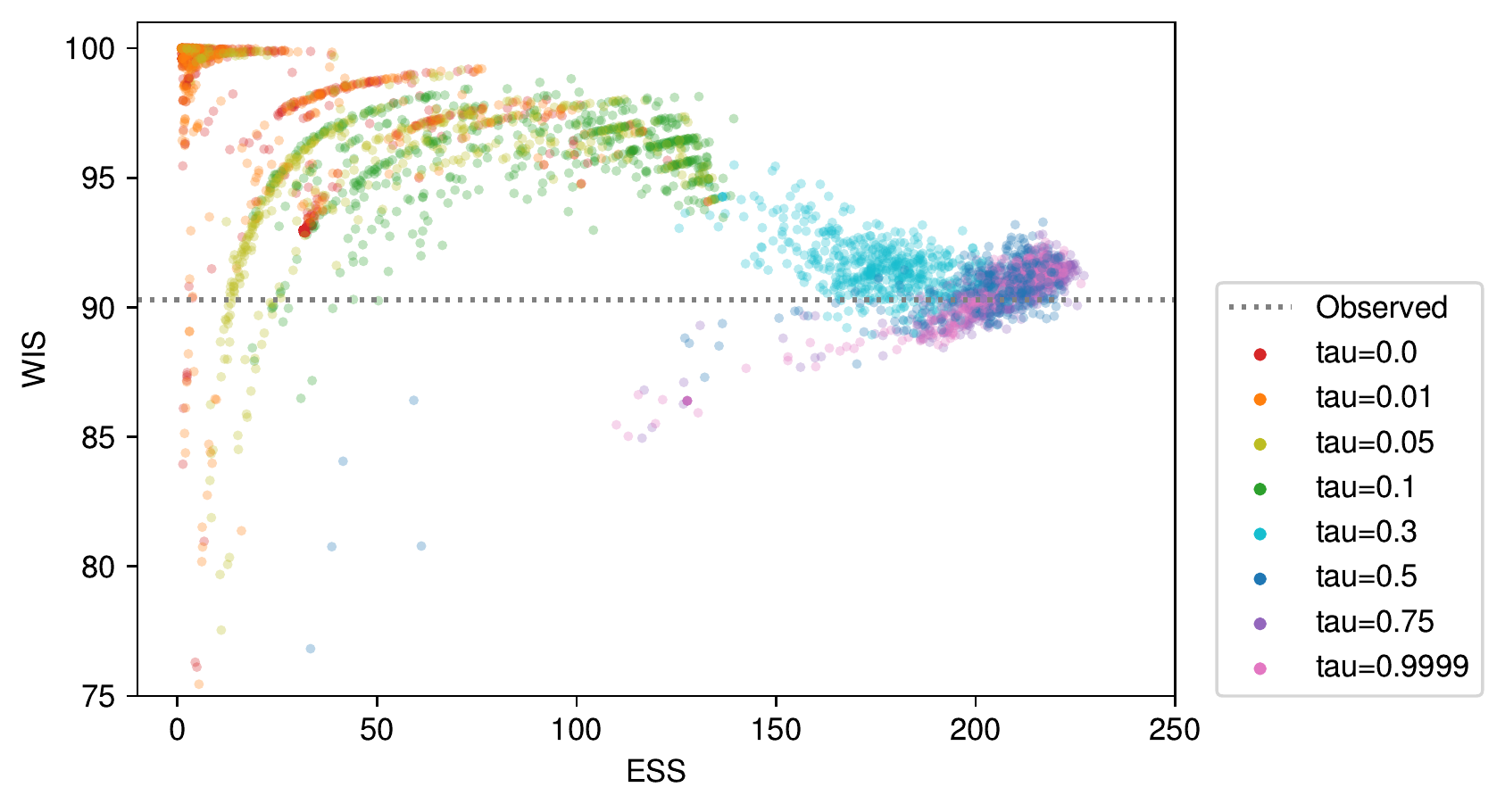}
    \end{tabular}
    \caption{Validation performance (in terms of WIS and ESS) for all hyperparameter settings and all checkpoints considered during model selection. Left - baseline, Right - proposed. }
    \label{fig:mimic_validation_full}
\end{figure}

\begin{figure}
    \centering
    \begin{tabular}{cc}
    \includegraphics[scale=0.75]{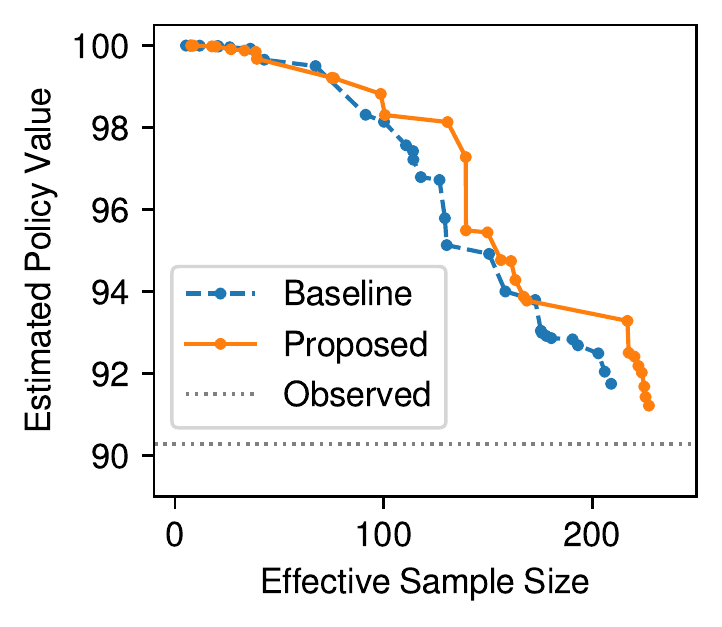} & 
    \includegraphics[scale=0.75]{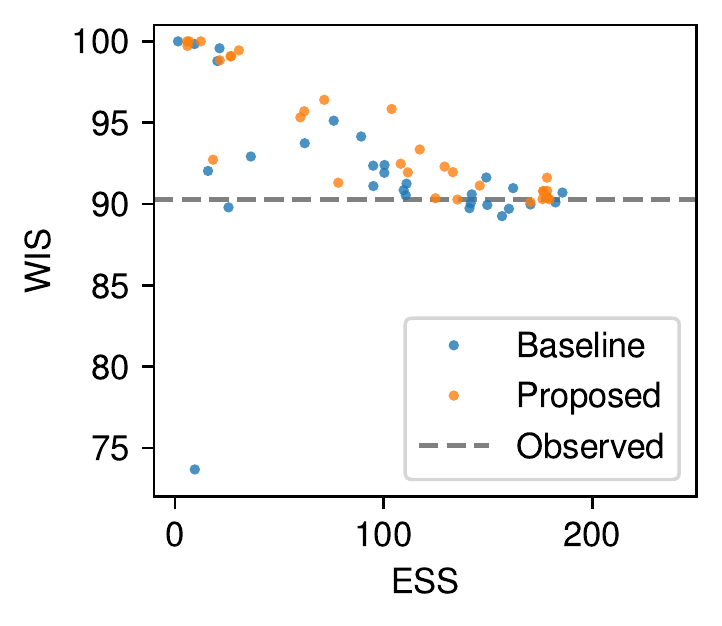}
    \end{tabular}
    \caption{Left - Pareto frontiers of validation performance for the baseline and proposed approaches; Right - test performance of the candidate models that lie on the validation Pareto frontier. The validation performance largely reflects the test performance, and proposed approach outperforms the baseline in terms of test performance albeit with a bit more overlap. }
\end{figure}

\begin{figure}
    \centering
    \includegraphics[scale=0.6]{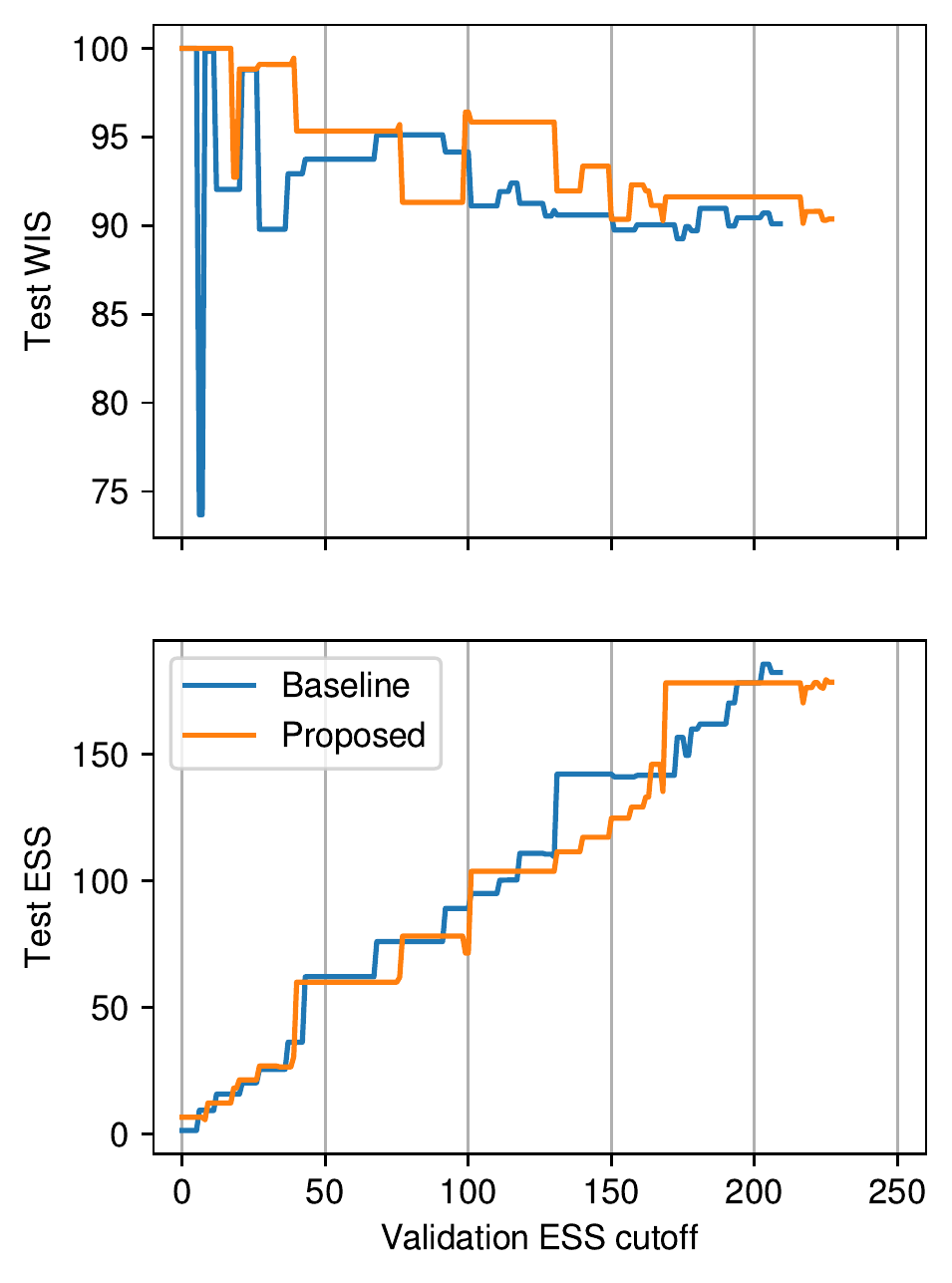}
    \caption{Model selection with different minimum ESS cutoffs. In the main paper we used ESS $\geq 200$; here we sweep this threshold and compare the resultant selected policies for both the baseline and proposed approach (only using candidate models that lie on the validation Pareto frontier). In general, across the ESS cutoffs, the proposed approach outperforms the baseline in terms of test set WIS value, with comparable or slightly lower ESS. }
\end{figure}

\end{document}